\documentclass[10pt]{article}

\usepackage[accepted]{tmlr}
\def\month{11}
\def\year{2026}
\def\openreview{\url{https://openreview.net/forum?id=jcxfzWt5PY}}

\usepackage{amsmath,amsfonts,bm}

\def\1{\bm{1}}

\def\eps{{\epsilon}}

\def\rvc{{\mathbf{c}}}

\def\rve{{\mathbf{e}}}

\def\rvm{{\mathbf{m}}}

\def\rvx{{\mathbf{x}}}

\def\rvz{{\mathbf{z}}}

\DeclareMathAlphabet{\mathsfit}{\encodingdefault}{\sfdefault}{m}{sl}
\SetMathAlphabet{\mathsfit}{bold}{\encodingdefault}{\sfdefault}{bx}{n}

\newcommand{\R}{\mathbb{R}}

\providecommand{\rev}[1]{#1}
\newenvironment{revblock}{}{}

\usepackage{microtype}
\usepackage{graphicx}
\usepackage{subcaption}
\usepackage{booktabs}
\usepackage{wrapfig}
\usepackage{algorithm}

\let\tmlrAND\AND
\let\AND\relax
\usepackage{algorithmic}

\let\AND\tmlrAND

\usepackage{float}
\usepackage{amsmath}
\usepackage{amssymb}
\usepackage{mathtools}
\usepackage{amsthm}
\usepackage{multirow}
\usepackage{xcolor}
\renewcommand{\rev}[1]{#1}
\renewenvironment{revblock}{}{}
\usepackage{etoc}
\usepackage{hyperref}
\usepackage{url}
\usepackage[capitalize,noabbrev,nameinlink]{cleveref}
\hypersetup{hidelinks}

\definecolor{algblue}{RGB}{210,230,255}
\newcommand{\algbluebox}[1]{%
  \colorbox{algblue}{\parbox{\dimexpr\linewidth-2\fboxsep}{#1}}%
}
\newcommand{\alggreenbox}[1]{%
  \colorbox{green!15}{\parbox{\dimexpr\linewidth-2\fboxsep}{#1}}%
}

\theoremstyle{plain}
\newtheorem{theorem}{Theorem}[section]

\theoremstyle{definition}

\theoremstyle{remark}

\title{CRoCoDiL: Continuous and Robust Controlled Diffusion for Language}

\author{%
  \name Roy Uziel \email{rouziel@nvidia.com}\\
  \addr NVIDIA
  \AND
  \name Omer Belhasin \email{obelhasin@nvidia.com}\\
  \addr NVIDIA
  \AND
  \name Matan Rusanovsky \email{mrusanovsky@nvidia.com}\\
  \addr NVIDIA
  \AND
  \name Itay Levy \email{itlevy@nvidia.com}\\
  \addr NVIDIA
  \AND
  \name Ran Zilberstein \email{rzilberstein@nvidia.com}\\
  \addr NVIDIA
  \AND
  \name Akhiad Bercovich \email{abercovich@nvidia.com}\\
  \addr NVIDIA
  \AND
  \name Ran El-Yaniv \email{relyaniv@nvidia.com}\\
  \addr NVIDIA
  \AND
  \name Michael Elad \email{melad@nvidia.com}\\
  \addr NVIDIA
}

\begin{document}

\maketitle
\let\AND\relax
\etocdepthtag.toc{mainpaper}

\begin{abstract}

    Masked Diffusion Models (MDMs) provide an efficient non-causal alternative to autoregressive generation but often struggle with token dependencies and semantic incoherence due to their reliance on discrete marginal distributions. We address these limitations by shifting the diffusion process into a continuous sentence-level semantic space. We propose \emph{CRoCoDiL} -- Continuous and Robust Controlled Diffusion for Language --  a unified fine-tuning approach that jointly trains an encoder–demasker architecture, grounding the MDM demasking in continuous latent representations. This leads to the formation of a novel autoencoder in which the decoding is obtained by an MDM algorithm. %
    Relying on this framework, we proceed by introducing two \emph{unconditional} text synthesis algorithms: Continuous-Then-Discrete (\emph{ConThenDisc}), a hybrid-diffusion approach that first generates latent representations in continuous space and then decodes these to tokens via an MDM, and Continuous-Within-Discrete (\emph{ConWithinDisc}), a multi-diffusion strategy that refines latent representations throughout the discrete sampling process.
Built on MDLM (400M), CRoCoDiL improves generation quality with \rev{up to a $\times 4$ reduction in parameter-weighted sampling cost relative to prior state-of-the-art samplers}. On LLaDA-8B, it achieves \rev{up to $\times 16$ higher measured throughput than} the unguided baseline while preserving generation quality, providing the first study of continuous latent guidance for unconditional text generation at this scale.
    
    \end{abstract}

\section{Introduction}

Diffusion-based alternatives to autoregressive large language models have been drawing much attention recently~\citep{li2022diffusion,yi2024diffusion}.
Such methods encompass an appealing potential to break the causal, one-token-at-a-time, paradigm of autoregressive machines, with the general hope to lead to faster and improved quality text synthesis.
The main challenge in bringing diffusion models to text is the evident gap between the continuous formulation of classical diffusion algorithms and the discrete nature of language~\citep{lou2024discrete}.

Earlier work addressed the discrete-continuum gap in a wide variety of ways; among these, the commonly used ones are based on \emph{masked} diffusion models~\citep{sahoo2024simple,nie2025large,ye2025dream}. This breadth of algorithms relies on a forward degradation process that masks tokens gradually until the whole sequence is masked-out. Text generation is based on a reversed process, in which a demasker iteratively revives tokens, constituting the \emph{Masked Diffusion  Models} (MDMs), such as MDLM~\citep{sahoo2024simple}, LLaDA \citep{nie2025large}, Dream \citep{ye2025dream}, and their many followups, e.g.~\citep{arriola2025block,wu2025fast,liu2025think,liu2025dllm}.

MDMs rely on a demasking model that is trained on partially masked sequences to estimate discrete logits for the missing tokens, representing one-dimensional marginal distributions that lack information on statistical cross-dependencies between tokens. When sampling from these logits, revealing multiple tokens in parallel necessarily produces flawed samples that degrade generation quality~\citep{liu2025discrete}.  Nevertheless, as synthesis speed depends on parallel token sampling, existing algorithms compromise  speed for quality.

Another, related yet different, weakness of MDM algorithms has to do with their core \emph{modus operandi} of constructing the generated text by sampling individual tokens (separately or jointly) sequentially, with each sampled token committed to be part of the final sequence. While appealing due to its resemblance to the autoregressive strategy, the absence of global guidance to drive this synthesis means that MDMs struggle to form coherent sentences.

In this paper we propose a novel extension to MDMs that addresses these limitations. Our approach uses a continuous diffusion model to generate sentence-level semantic representations, while the MDM algorithm serves as a decoder translating these latent vectors into token sequences.
This way, the burden of capturing long-range, cross-token structure is shifted to a lightweight classical diffusion in the latent space. %
This representation is then used to guide the MDM for token decoding, enabling effective multi-token sampling per step, and thereby achieving better efficiency-quality tradeoffs. We name this methodology \emph{CRoCoDiL}: Continuous and Robust Controlled Diffusion for Language (see ~\cref{fig:CRoCoDiL}). 
The benefit of this hybrid design becomes most pronounced at large scale: for an MDM with billions of parameters, the discrete demasker is far larger than the continuous denoiser, and thus a lightweight latent-generation stage can replace many expensive demasker calls.

\begin{figure}[t]
    \centering
    \includegraphics[width=1.0\columnwidth]{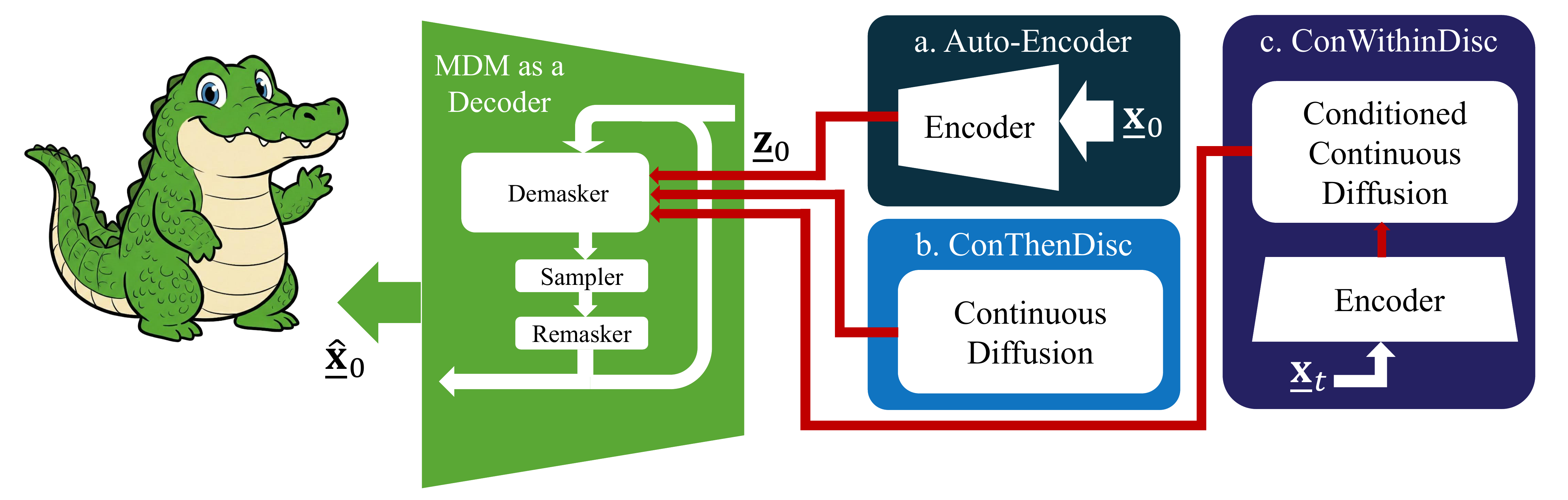} 
    \caption{ The \emph{CRoCoDiL} framework: Building on a learned encoder of text sequences and a demasker guided by this continuous representation, we introduce (a) an autoencoder and (b,c) two text generation algorithms, \emph{ConThenDisc} and \emph{ConWithinDisc}. A regular MDM serves in all cases as a decoder that converts the latent $\rvz_0$ into a sequence ${\hat \rvx}_0$. The text generation algorithms rely on learned diffusion models that operate in the representation domain. }
    \label{fig:CRoCoDiL}
\end{figure}

Building on this framework, we introduce a unified encoder-demasker training scheme that encodes sequences into latent representations for effective token decoding. We then present two text synthesis algorithms: (1) Continuous-Then-Discrete (\emph{ConThenDisc}) that generates embeddings via continuous diffusion and uses MDM to decode the latent vector into tokens; (2) Continuous-Within-Discrete (\emph{ConWithinDisc}), that updates the guidance vector during the demasking steps using a continuous diffusion trained to recover valid latent vectors from partially masked sequences. 
Our work focuses on unconditional text generation, an established evaluation protocol for diffusion-based language models, e.g., SEDD~\citep{lou2024discrete}, MDLM~\citep{sahoo2024simple}, and COSMOS~\citep{meshchaninov2025compressed}, leaving conditional synthesis across benchmarks for future work.

We conduct an extensive experimental study using MDLM~\citep{sahoo2024simple} on OpenWebText~\citep{gokaslan2019openwebtext}, and LLaDA-8B~\citep{nie2025large} on programs from StarCoder Dataset~\citep{li2023starcoder}, both serving as base MDM backbones. Using Qwen-embedding-0.6B~\citep{ren2025qwen3} as an initial encoder, each MDM is jointly retrained with our encoder-demasker framework. We first validate the effectiveness of the continuous guidance via autoencoding, demonstrating faithful reconstruction. We then evaluate our two proposed algorithms for unconditional text synthesis, where they achieve substantially faster sampling at matched quality: our guided MDLM yields a $\times 4$ speedup over the current state-of-the-art sampler~\citep{meshchaninov2025compressed}, and the guided LLaDA-8B yields \rev{up to $\times 16$ higher measured throughput than} its base version.

To summarize, the following are the main contributions of this work: 
\begin{itemize}
    \item We propose \emph{CRoCoDiL}, a novel framework that guides discrete MDMs using a continuous, sentence-level semantic representation, bridging the gap between global coherence and local token dependencies, thus enabling faithful parallel token sampling.
    CRoCoDiL leads to a general purpose autoencoder that maps sequences to the continuum and back. Building on this, we derive two unconditional text synthesis algorithms: \emph{ConThenDisc} and \emph{ConWithinDisc}, both using a continuous latent sketch that guides an MDM decoder. %

    \item %
    We empirically validate CRoCoDiL across two MDM scales and two corresponding datasets: on MDLM (400M, OpenWebText) it surpasses the state-of-the-art sampler, and on LLaDA-8B (StarCoder Python) it surpasses the unguided baseline, in both cases improving generation quality and sampling speed.
    
    \item We show that CRoCoDiL is particularly well-suited to the large-scale regime, since the continuous denoiser remains lightweight while the demasker scales with the model: the speedup grows from $\times 4$ over the state-of-the-art sampler on MDLM (400M) \rev{to $\times 16$ higher measured end-to-end throughput over the unguided baseline on LLaDA-8B}.

\end{itemize}

\section{Related work}
\label{sec:related}

\cref{app:background}  
provides a broad overview on the field of diffusion models for text generation. In this section we dive into specific recent work that has a direct relevance to this paper's  contributions. 

Unconditional text generation is a fundamental setting in language modeling, and it has become a central evaluation protocol for diffusion-based language models, e.g., SEDD~\citep{lou2024discrete}, MDLM~\citep{sahoo2024simple}, GIDD~\citep{von2025generalized}, TEncDM~\citep{shabalin2025tencdm},  COSMOS~\citep{meshchaninov2025compressed}, and others~\citep{liu2025discrete,xu2025energy,xie2025variational,azangulov2025parallel,luxembourg2025plan}. We adopt this protocol to evaluate \emph{CRoCoDiL} across scale: first on MDLM (400M), matching the small-model regime of prior studies, and then on LLaDA-8B, demonstrating that the proposed continuous guidance remains effective for a modern billion-parameter MDM backbone.

The work reported in \citep{meshchaninov2025compressed} presents COSMOS, a language generation algorithm that relies on a continuous latent space diffusion, improving over TEncDM~\citep{shabalin2025tencdm}. While similar to the main theme of our work, COSMOS differs  from it substantially. %
In particular, the decoder that converts the embedding to tokens in COSMOS has no generative capabilities, which implies that the latent representation must be fully informative in order to enable proper text synthesis. In contrast, our latent representation serves as a sketch guide that conditions an iterative MDM-based decoding process, and thus even partially informative representations can lead to valid and high quality generated text, as MDM complements and refines the synthesis process. In this work we further examine this distinction by offering comparisons to COSMOS, showing a marked advantage of the CRoCoDiL framework in both scales studied (400M and 8B).

\rev{Several recent methods also combine continuous and discrete diffusion. CCDD~\citep{zhou2026ccdd} and CADD~\citep{zheng2025cadd} maintain per-token continuous states and update them together with the tokens at each step; LADD~\citep{shariatian2025ladd} denoises a pooled latent either jointly with the tokens or before them; and DiLaDiff~\citep{lemercier2026diladiff}, concurrent with this work, generates a long latent before masked decoding and keeps it fixed. CRoCoDiL instead learns a compact, sequence-level representation jointly with the demasker, can refresh it during decoding (\emph{ConWithinDisc}), and is demonstrated with an 8B demasker.}

In the spirit of the contrast between COSMOS and our paradigm, the work of \citep{morris2023text} argues that when using embedding representations, decoding must be performed iteratively rather than in a single step, which supports our proposed fusion of continuum and MDM. That said, \citep{morris2023text} is distinct from our work as it focuses on text correction tasks rather than their generation.

Another relevant work is GIDD~\citep{von2025generalized}, which generalizes masked diffusion by interpolating between masking and uniform noising, enabling the model to revise tokens it has already committed to, thereby overcoming the MDLM's pitfall with cross-token dependencies.  GIDD~\citep{von2025generalized}, together with TEncDM~\citep{shabalin2025tencdm} and COSMOS~\citep{meshchaninov2025compressed}, provide the most direct and competitive points of comparison for our MDLM experiments in~\cref{sec:uncond}.

As a final note, we mention the work reported in \citep{arriola2025encoder}, which presents an autoencoding framework, referred to as E2D2. Under a conditional synthesis setup in which the model responds to a prompt, E2D2 encodes the prompt to a continuous vector and uses it to guide a fully discrete MDM decoder that constructs the response. As the synthesis of the answer relies on a plain MDM, the statistical cross-token dependencies are not taken into account -- a problem that we tackle in this work.

\section{Problem formulation and background}
\label{sec:background}

Let $\rvx = (x^1, x^2, \dots, x^n)$ be a discrete sequence of $n$ tokens drawn 
from a data distribution $q_{\text{data}}$, where each $x^i$ takes 
values in a vocabulary $\mathcal{V}$. Following recent work on discrete 
diffusion methods, we adopt the masked diffusion modeling (MDM) 
framework~\citep{sahoo2024simple}. The vocabulary is augmented with a special 
mask token \texttt{[M]}, and we denote by $\rvm$ the fully masked sequence.

\paragraph{Forward corruption.}
The forward process masks tokens independently across positions,
\begin{equation}
\label{eq:forward}
    q(\rvx_t \mid \rvx_0) = \prod_{i=1}^n q(x_t^i \mid x_0^i), \qquad
    q(x_t^i \mid x_0^i) := \alpha_t \mathbf{e}_{x_0^i} + (1 - \alpha_t) \mathbf{e}_{\texttt{[M]}},
\end{equation}
with $\alpha_t \in [0,1]$ being a strictly decreasing schedule satisfying 
$\alpha_0 \approx 1$ and $\alpha_1 \approx 0$, and $\mathbf{e}_j$ is the one-hot encoding of token 
$j \in \mathcal{V} \cup \{\texttt{[M]}\}$.

\paragraph{Reverse process and demasking.}
Generation is obtained by reversing~\cref{eq:forward}. However, as the marginal posterior 
$q(\rvx_s \mid \rvx_t)$ for $0 \le s < t \le 1$ is intractable, a conditioning on $\rvx_0$ is added as in~\citep{ho2020denoising,sahoo2024simple}. This leads to a transition that factorizes over tokens and admits a closed-form 
per-coordinate expression $q(x_s^i \mid x_t^i, x_0^i)$ (the full form is given 
in~\cref{app:mdm_details}). At inference $\rvx_0$ is unknown, and 
modeling the joint $q(\rvx_0 \mid \rvx_t)$ directly is intractable due to the 
$\mathcal{O}(|\mathcal{V}|^n)$ combinatorial explosion. MDM therefore 
introduces a demasking network $f_\theta$ that estimates the marginals 
$p_\theta(x_0^i \mid \rvx_t)$ at masked positions. Tokens are sampled 
independently, $\hat{x}_0^i \sim f_\theta^i(\rvx_t, t)$, and substituted into 
the per-coordinate reverse transition, producing the effective posterior
\begin{equation}
\label{eq:effective_reverse}
    p_\theta(\rvx_s \mid \rvx_t) := \prod_{i=1}^n q\!\left(x_s^i \,\middle|\, x_t^i, \hat{x}_0^i\right).
\end{equation}

\paragraph{Fundamental limitation.}
The factorized approximation in~\cref{eq:effective_reverse} suffers from 
a well-known limitation~\citep{liu2025discrete}: independent sampling from the 
estimated marginals fails to capture cross-token dependencies and semantic 
correlations across masked positions. In~\cref{app: theoretical analysis} 
we characterize this gap formally, showing that it persists even under an 
optimal demasking model. This limitation directly motivates the continuous 
guidance mechanism developed next.

\section{Continuously guided MDM}
\label{sec:method1}

We now turn to introduce the \emph{CRoCoDiL} framework that aims to bridge the modeling gap between the desired reverse and the practical approximation posed in~\cref{eq:effective_reverse}.
Another benefit of our strategy is the introduction of a sketched guidance to the discrete generation process, which further boosts synthesis quality. 

Our solution relies on a guidance mechanism for the demasking model, derived from a continuous latent representation of the clean sample $\rvx_0$. This approach enables the model to incorporate global context and cross-token statistical dependencies, bypassing the dimensionality barrier of modeling the joint conditional $q(\rvx_0|\rvx_t)$ explicitly.
We start by introducing a training framework that constructs an embedding model for turning a token sequence $\rvx_0$ into a corresponding continuum latent vector, and allowing the MDM's demasker to be effectively guided by it. Equipped with this machinery, we present three guided and fast MDM algorithms:

(i) A general-purpose autoencoder scheme that converts sequences of tokens to the continuum and back, where the encoder is the one described above and the decoder is an MDM algorithm; 

(ii) A novel text synthesis algorithm, \emph{ConThenDisc} (Continuous-Then-Discrete), that generates a valid embedding vector via a continuous diffusion and decodes it by a fast MDM, as in the above autoencoder scheme; and 

(iii) A refined \emph{ConThenDisc} in which the guidance vector is updated within the MDM steps by a conditional diffusion that operates in the embedding domain. We term this method \emph{ConWithinDisc} (Continuous-Within-Discrete).

Common to all three methods is the fact that the MDM may operate in a multi-token sampling regime per step, enabled due to the continuum guidance, and thus becoming much faster than the vanilla MDM alternative. %

\subsection{Continuously guided demasking}
\label{sec:guidance-learning}

Consider an encoder model that maps a discrete sequence $\rvx_0$ into a continuous latent representation $\rvz_0\in \R^d$. Assume further that this continuous latent vector is learned so as to serve as an informative guidance for the demasking process in MDM, enabling the recovery of cross-token  dependencies among multiple masked positions. Herein we offer a training framework for this constellation.

{\bf Encoder.} 
We define an encoder $h_\phi: \mathcal{V}^n \rightarrow \mathbb{R}^d$, parameterized by $\phi$, that maps a clean sequence $\rvx_0 \in \mathcal{V}^n$ to a continuous representation $\rvz_0 \in \mathbb{R}^d$, denoted as $\rvz_0 = h_\phi(\rvx_0)$.
This encoder is trained to capture the essential information of the input sequence in continuous form, so as to enable (even a partial\footnote{In our work $\rvz_0$ does not have to be a one-to-one representation of $\rvx_0$, as in COSMOS~\citep{meshchaninov2025compressed}.}) reconstruction back to its discrete form.
In practice, $\rvz_0$ contains a set of $K$ register vectors, where nested dropout encourages a progressive order across the register index~\citep{flextok}. \rev{The representation combines the hidden state of the final valid text token with $K-1$ learned suffix registers; implementation details are given in~\cref{app:Encoder-Demasker-Training}.} For clarity, the equations below are written for one register.

{\bf Guided Demasker.} 
We propose a conditional demasking model $f_\theta: (\mathcal{V} \cup \{\texttt{[M]}\})^n \times \mathbb{R}\times \mathbb{R}^d \rightarrow \mathcal{R}^{n\times |\cal{V}|}$, parametrized by $\theta$, that predicts a clean data sample $\rvx_0$ from a partially masked sequence $\rvx_t$, conditioned on a latent representation $\rvz_0 \in\R^d$ of the very same clean sequence $\rvx_0$. 
\rev{From here on, we denote the unguided demasker of~\cref{eq:effective_reverse} by $f_{\mathrm{base}}(\rvx_t,t)$ and reserve $f_\theta(\rvx_t,t,\rvz_0)$ for the guided one. The guided demasker} outputs the distributions $f_\theta^i(\rvx_t, t, \rvz_0) $ 
that approximate the true marginals $q(x_0^i | \rvx_t, \rvz_0)$ when $x_t^i = \texttt{[M]}$, and $f_\theta^i(\rvx_t, t, \rvz_0) := \mathbf{e}_{x_t^i}$ otherwise. 

{\bf Training Objective.} 
The encoder and demasker are jointly trained to minimize the following loss:
\begin{equation}
    \mathcal{L}(\theta,\phi) = \mathbb{E}_{t, \rvx_0, \rvx_t} \left[ -w_{\rvx_t} \cdot \frac{1}{n} \sum_{i=1}^n  \log \left[ f_\theta^i(\rvx_t ,t, h_\phi(\rvx_0))\right]_{x^i_0} \right],
    \label{eq:loss-trainingdemasker}
\end{equation}
where $w_{\rvx_t}$ denotes the weight that prioritizes clean sequences over corrupted ones. For example, in LLaDA \citep{nie2025large}, the weights are defined as $w_{\rvx_t}^i := 1/\alpha_t$. %

The proposed loss in~\cref{eq:loss-trainingdemasker} is constructed as  follows: We start by sampling a clean sequence $\rvx_0$ from the training set, choose $t$ at random from the uniform $[0,1]$ distribution, and generate $\rvx_t$ by randomly masking appropriate portion from $\rvx_0$. The demasker operates on $\rvx_t$, $t$, and $h_\phi (\rvx_0)$ (the embedding of the original sequence $\rvx_0$). 
At each location $i$, the demasker outputs logits that should place high probability on the true token $x_0^i$. We optimize both the encoder and the demasker to recover $\rvx_0$ from $\rvx_t$. The same loss applies to a regular MDM, except that it does not condition on $h_\phi(\rvx_0)$; this additional guidance enables the demasker to better account for cross-token dependencies.

\begin{figure}[t]
    \centering
    \includegraphics[width=0.72\linewidth]{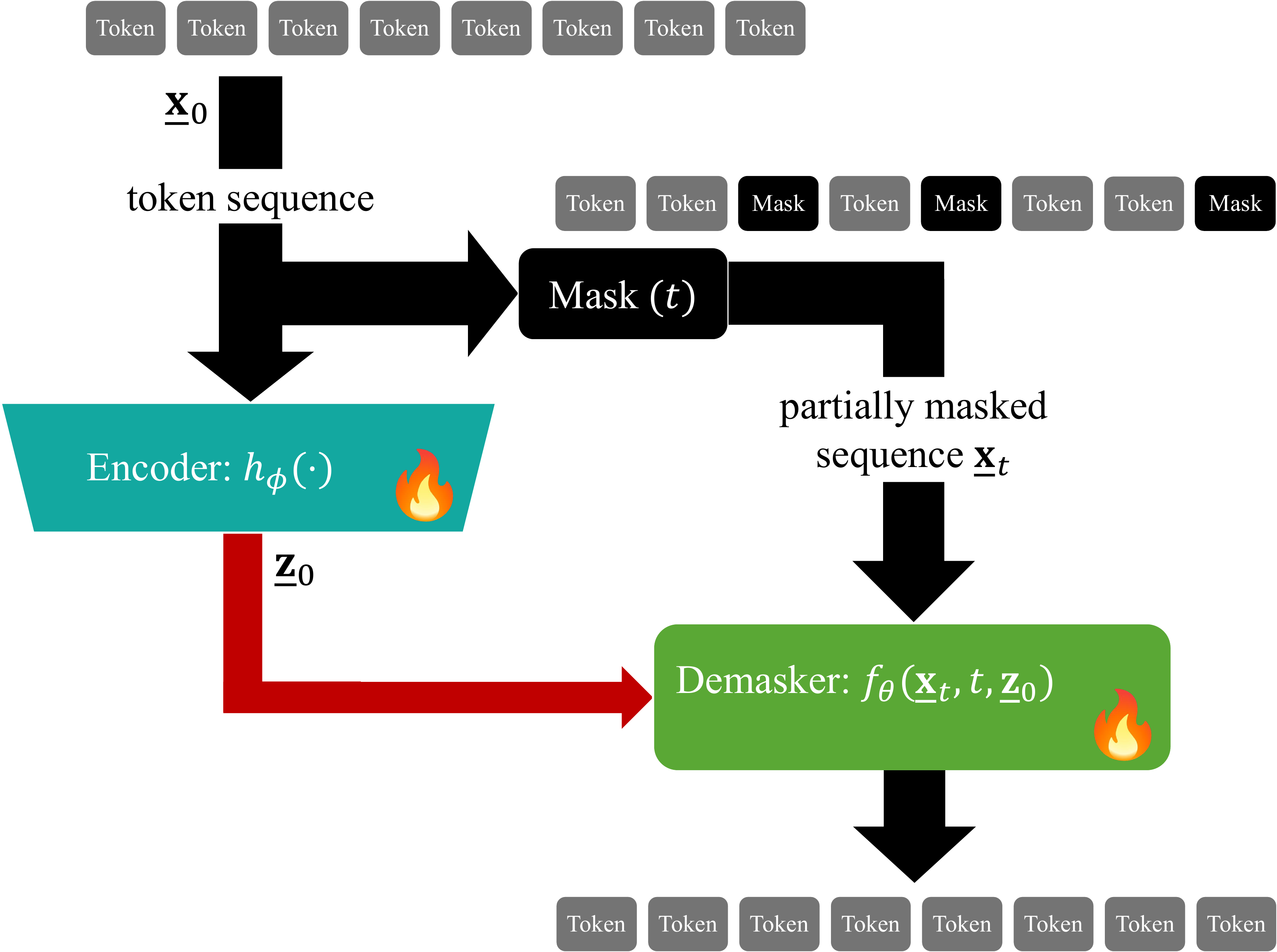}
    \caption{The training framework for the embedding network $\rvz_0 = h_\phi(\rvx_0)$ and the guided demasker $f_\theta(\rvx_t,t,\rvz_0)$. Flame indicates a trained network.}
    \label{fig:TrainingMethod}
\end{figure}

\cref{fig:TrainingMethod} illustrates the training framework for the embedding and the guided demasker. \cref{alg:training} summarizes one training step. In lines 5-6, Gaussian noise is added to $\rvz_0$ with a shared variance $\sigma$, calibrated such that $\mathbb{E}_{\rvz_0,\rve}[\cos(\rvz_0,\rvz_0+\rve)]\approx0.8$.
As shown in~\cref{app:RobustnessAblation}, this step is critical to the success of the training results, as it introduces a robustification of the embedding obtained and its influence on the demasker predictions. %
This encoder-demasker framework %
recovers the essential cross-token dependencies during unmasking, enabling the effective factorized reverse transition in~\cref{eq:effective_reverse}.~\cref{theo:Factorizability} in~\cref{app: theoretical analysis} formally justifies this claim. 

A key implementation detail enabling the hybrid framework is a two-axis MRoPE-style positional encoding~\citep{wang2024qwen2} that distinguishes the continuous-register prefix from the text stream. An ablation in the MDLM setting (\cref{app:Encoder-Demasker-Training}) shows that without this design choice the guided demasker degrades substantially, especially at high masking ratios.

\subsection{MDM-based autoencoder}
\label{sec:auto-encoder-gen}

\begin{figure}[t]
\centering
\begin{minipage}{0.58\linewidth}
\begin{algorithm}[H]
  \caption{Encoder-Demasker Training}
  \label{alg:training}
  \begin{algorithmic}[1]
    \STATE {\bfseries Input:} $\rvx_0 \sim p_{\text{data}}$
    \STATE $t \sim \mathcal{U}([0,1])$
    \STATE $\rvx_t \leftarrow \operatorname{Mask}(\rvx_0,t)$
    \STATE $\rvz_0 \leftarrow h_\phi(\rvx_0)$
    \STATE $\rve \sim \mathcal{N}(0,\sigma^2 I)$
    \STATE $\tilde{\rvz}_0 \leftarrow \rvz_0+\rve$
    \STATE $\hat{\rvx}_0 \leftarrow f_\theta(\rvx_t,t,\tilde{\rvz}_0)$
    \STATE $\hat{\mathcal{L}}_{\rm CE} \leftarrow -\frac{w_{\rvx_t}}{n}\sum_i \log [\hat{\rvx}_0^i]_{x_0^i}$
    \STATE Update $\theta,\phi$ using $\nabla_{\theta,\phi}\hat{\mathcal{L}}_{\rm CE}$
  \end{algorithmic}
\end{algorithm}
\end{minipage}
\end{figure}
We now utilize the trained encoder-demasker framework
to bridge between discrete and continuous representations.
Specifically, we aim to transform a clean discrete sequence $\rvx_0$ into a continuous representation and reconstruct it back as accurately as possible to its original discrete form.

Given a clean input sequence $\rvx_0 \sim p_{\text{data}}$, we first encode it into the continuous latent via $\rvz_0 = h_\phi(\rvx_0)$.
To reconstruct it, we deploy a full yet fast (using fewer demasker \rev{forward passes}) MDM process, in which the demasker $f_\theta$ iteratively refines the sequence, initialized by the fully masked vector $\rvm$, and applying $T$ discrete diffusion steps, all conditioned on the representation $\rvz_0$. As in regular MDM, 
at each timestep $t$, the demasker predicts the clean tokens independently, $\hat{\rvx}_0^i \sim f_\theta^i(\rvx_t,t,\rvz_0)$, which are then masked according to the forward diffusion process for the next iteration.
This general-purpose autoencoder enables flexible processing of sequences in the continuous domain for generation, interpolation, or other downstream tasks, while maintaining the capability to recover faithful discrete outputs. In the context of the discussion in this paper, this autoencoder serves as a stepping stone towards the following hybrid text synthesis algorithms, which build on it.

\begin{figure}[h]
\begin{minipage}[t]{0.48\linewidth}
\begin{algorithm}[H]
  \caption{\emph{ConThenDisc} Text Synthesis}
  \label{alg:ConThenDisc-Inference}
  \begin{algorithmic}[1]
    \STATE {\bfseries Input:} \\
    \algbluebox{%
        \STATE $\epsilon \sim \mathcal{N}(0,I)$
        \STATE $\rvz_0 \leftarrow G_\psi(\epsilon)$%
      }
    \alggreenbox{
    \STATE $t = 1$
    \STATE $\rvx_t := \rvm$
    \WHILE{$t>0$}
      \STATE $\hat{\rvx}_0 \sim f_\theta(\rvx_t,t,\rvz_0)$
      \STATE $t \leftarrow t - 1/T$
      \STATE $\rvx_{t} = \rev{\operatorname{Mask}}(\hat{\rvx}_0, t)$ 
    \ENDWHILE
    \STATE \textbf{return} $\hat{\rvx}_0$
    }
  \end{algorithmic}
\end{algorithm}
\end{minipage}
\hfill
\begin{minipage}[t]{0.48\linewidth}
\begin{algorithm}[H]
  \caption{\emph{ConWithinDisc} Text Synthesis}
  \label{alg:Continuous-Within-Discrete}
  \begin{algorithmic}[1]
    \STATE {\bfseries Input:}  \\
     \alggreenbox{
    \STATE $t = 1$
    \STATE $\rvx_t := \rvm$
    \WHILE{$t>0$}
        \algbluebox{
            \STATE $\epsilon \sim \mathcal{N}(0,I)$
            \STATE $\rvz_0 \leftarrow G_\psi(\epsilon,h_\phi(\rvx_t))$
        }
        \STATE $\hat{\rvx}_0 \sim f_\theta(\rvx_t,t,\rvz_0)$
        \STATE $t \leftarrow t - 1/T$
        \STATE $\rvx_{t} = \rev{\operatorname{Mask}}(\hat{\rvx}_0, t)$ 
    \ENDWHILE
    \STATE {\bfseries return} $\hat{\rvx}_0$}
    \vspace{1.0pt}
  \end{algorithmic}
\end{algorithm}
\end{minipage}
\end{figure}

\subsection{Hybrid text generation strategies}
\label{sec:generation}

We now turn to introduce two hybrid (fusion of continuous and discrete) text generative strategies that take advantage of core MDM while also leaning on the availability of the newly formed continuously guided demasker. We start with \emph{ConThenDisc}, and proceed with its improvement, the \emph{ConWithinDisc} method, both targeting unconditional text synthesis.

\subsubsection{Continuous-Then-Discrete}
\label{sec:continuous-then-discrete}

In the proposed autoencoder setting, a given sentence $\rvx_0$ is converted to a continuous latent representation $\rvz_0$, followed by a decoding stage that relies on a full MDM, aiming to recover the original $\rvx_0$. 
Building on this very configuration, the \emph{ConThenDisc} algorithm suggests producing $\rvz_0$ differently; rather than leaning on a given sentence $\rvx_0$, we suggest to generate it by randomly drawing from its corresponding distribution $\rvz_0 \sim P(\rvz)$. In practice, $\rvz_0$ is synthesized via a pre-trained continuous diffusion generator. %
\cref{alg:ConThenDisc-Inference} describes the \emph{ConThenDisc} text synthesis method. \rev{We write $g_\psi(\rvz_t,t)$ for a single clean-register prediction of the continuous denoiser and $G_\psi(\eps)$ for the complete DDIM sampler. Starting from Gaussian noise $\eps$, $G_\psi$ repeatedly applies $g_\psi$; in \emph{ConWithinDisc}, it is additionally conditioned on $h_\phi(\rvx_t)$. Running $G_\psi(\eps)$ yields a latent sample $\rvz_0 \sim P(\rvz)$.}\footnote{%
\cref{app: Continuous Diffusion} describes $G_\psi (\eps)$ in more detail. \rev{We reserve $\operatorname{Mask}$ for token corruption and use $q_s(\rvz_s\mid\rvz_0)$ for Gaussian corruption in the continuous diffusion.}}
We proceed by decoding it to a sequence of tokens by applying a complete MDM algorithm -- represented by the green lines in~\cref{alg:ConThenDisc-Inference}. Note that in generating the latent $\rvz_0$, a knowledge of the desired sequence length can be injected. This can be done by a slightly different design of the continuous diffusion that includes a conditioning on this length. In this work we do not implement this option, and allow any generation length, dictated by the MDM. %

\subsubsection{Continuous-Within-Discrete}
\label{sec:continuous-within-discrete}

\begin{figure}[t]
\centering
\begin{minipage}{0.58\linewidth}
\begin{algorithm}[H]
  \caption{\emph{ConWithinDisc} Training}
  \label{alg:conwithindisc-training}
  \begin{algorithmic}[1]
    \STATE {\bfseries Input:} $\rvx_0 \sim p_{\text{data}}$
    \STATE $\rvz_0 \leftarrow h_\phi(\rvx_0)$
    \STATE $s,t \sim \mathcal{U}([0,1])$
    \STATE $\rev{\rvz_s \sim q_s(\rvz_s\mid\rvz_0)}$
    \STATE $\rvx_t \leftarrow \operatorname{Mask}(\rvx_0,t)$
    \STATE $\rvc_t \leftarrow h_\phi(\rvx_t)$
    \STATE $\hat{\rvz}_0 \leftarrow g_\psi(\rvz_s,s,\rvc_t)$
    \STATE $\hat{\mathcal{L}}_{\rm MSE} \leftarrow \|\hat{\rvz}_0-\rvz_0\|_2^2$
    \STATE Update $\psi$ using $\nabla_\psi\hat{\mathcal{L}}_{\rm MSE}$
  \end{algorithmic}
\end{algorithm}
\end{minipage}
\end{figure}
A delicate weakness (and thus an unexploited opportunity) in~\cref{alg:ConThenDisc-Inference} is the fact that the guidance is kept fixed throughout the $T$ iterations, even though the sequence $\rvx_t$ is available, giving additional yet partial information about the text to be created. The \emph{ConWithinDisc} algorithm (see~\cref{alg:Continuous-Within-Discrete}) aims to leverage this opportunity, by updating the guidance vector within the MDM steps. More specifically, within the demasking steps, the guidance vector can be updated by drawing from the conditional distribution $\rvz_0\sim P(\rvz| h_\phi (\rvx_t))$. In words, the guidance vector is sharpened to take into account the currently held temporal sequence $\rvx_t$. In this step, $h_\phi$ is reused with masked positions excluded from attention, so the refreshed latent depends only on the visible tokens. %

A Few comments are in order: (i) The update of $\rvz_0$ can be done in a pre-selected subset of the overall $T$ steps, in order to benefit from the improved guidance while reducing the overall complexity of the generative algorithm; and (ii) In drawing the guidance vector, the conditioning we present leans on the \emph{embedding} of the partially masked sequence $\rvx_t$, i.e. $\rvz_0\sim P(\rvz| h_\phi (\rvx_t))$. Alternatively, we could have  conditioned the distribution directly on $\rvx_t$. %
(iii) In training the conditional diffusion (\cref{alg:conwithindisc-training}), we use $h_\phi(\rvx_t)$ for embedding the partially masked sentence. However, this encoder was not trained for such masked content. An improved strategy would be to define a second encoder $h_\mu(\rvx_t)$ to be used in Line 7, defining the loss as $\hat{\mathcal{L}}(\mu,\psi) = \|g_\psi(\rvz_s,h_\mu(\rvx_t))- \rvz_0\|^2_2$, and optimizing it w.r.t. both $\mu$ and $\psi$; and (iv)  Drawing samples from the conditional distribution $\rvz_0\sim P(\rvz| h_\phi (\rvx_t))$ can be interpreted as a solution of an inverse problem. Given a prior  $P(\rvz)$, and given measurements $h_\phi (\rvx_t)$, our goal is to produce posterior samples that recover the $\rvz_0$ that led to the given measurements. %

\section{Experimental results}

\subsection{Experimental protocol}\label{sec:setup}

We evaluate CRoCoDiL on two MDM backbones. The first is LLaDA-8B~\citep{nie2025large}, adapted to Python using $12$M programs from StarCoder~\citep{li2023starcoder}. The second is an MDLM ($\sim$400M parameters) trained on OpenWebText, used to test whether the observed trends persist across backbone scale and data domain. In both cases, the encoder $h_\phi$ is initialized from Qwen-embedding-0.6B~\citep{ren2025qwen3}, and the continuous denoiser $g_\psi$ is a $\sim$400M PixArt-style DiT~\citep{chen2023pixartalpha}. 

Throughout our experiments,  baseline and CRoCoDiL demasker variants use the same backbone, the same training data, and the same training steps. %
For unconditional generation we report Gen-PPL (via GPT-2-Large~\citep{gpt2L}) and MAUVE ~\citep{pillutla2021mauve}, as in~\citep{lou2024discrete,meshchaninov2025compressed}. \rev{For Python, we additionally report compilation and sandbox-loading metrics, defined in~\cref{app:code-evaluation}; sampling configurations are listed in~\cref{app:additional-evaluation}.} Autoencoding is additionally evaluated with Character Error Rate (CER) and a semantic similarity metric: CodeBERTScore~\citep{zhou2023codebertscore} for the LLaDA setting and BERTScore~\citep{zhang2019bertscore} for the MDLM setting.

\subsection{Guided demasking}

We first evaluate the core ingredient of CRoCoDiL: the ability of a learned latent $\rvz_0=h_\phi(\rvx_0)$ to improve demasking. Starting from the MDLM or LLaDA demasker, we train the guided variant $f_\theta(\rvx_t,t,\rvz_0)$ as described in~\cref{sec:guidance-learning}. The encoder outputs $K$ register vectors of dimension $1024$; additional training details are provided in~\cref{app:Encoder-Demasker-Training}.
\cref{fig:demasker-performance} reports top-1 token prediction accuracy as a function of the masking probability and the number of registers $k \in [1,K]$. We compare the unguided demasker \rev{$f_{\mathrm{base}}(\rvx_t,t)$} with guided variants $f_\theta(\rvx_t,t,\rvz_0)$ that %
differ only in the number of conditioning registers $k$ exposed at inference. Conditioning consistently improves the demasker, and the gap grows as more registers are provided. The same trends hold for cross-entropy, both for LLaDA and MDLM, as reported in~\cref{app:Encoder-Demasker-Training}. \rev{Results with generated registers are reported in~\cref{tab:code-levels,tab:owt8b} (Appendix~H).}

\begin{figure}[t]
    \centering
    \begin{subfigure}[t]{0.49\linewidth}
        \centering
        \includegraphics[width=\linewidth]{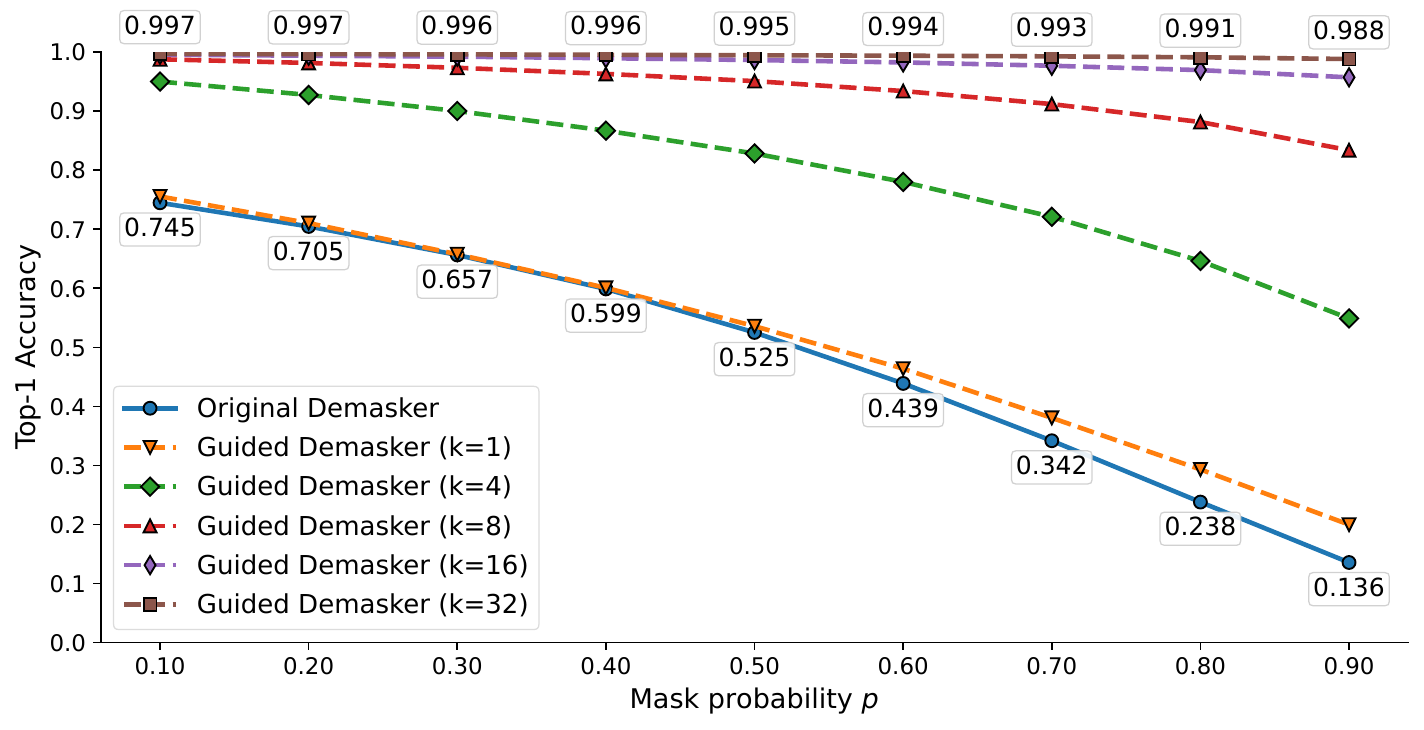}
        \caption{\textbf{MDLM:} Top-1 token prediction accuracy.}
    \end{subfigure}
    \hfill
    \begin{subfigure}[t]{0.49\linewidth}
        \centering
        \includegraphics[width=\linewidth]{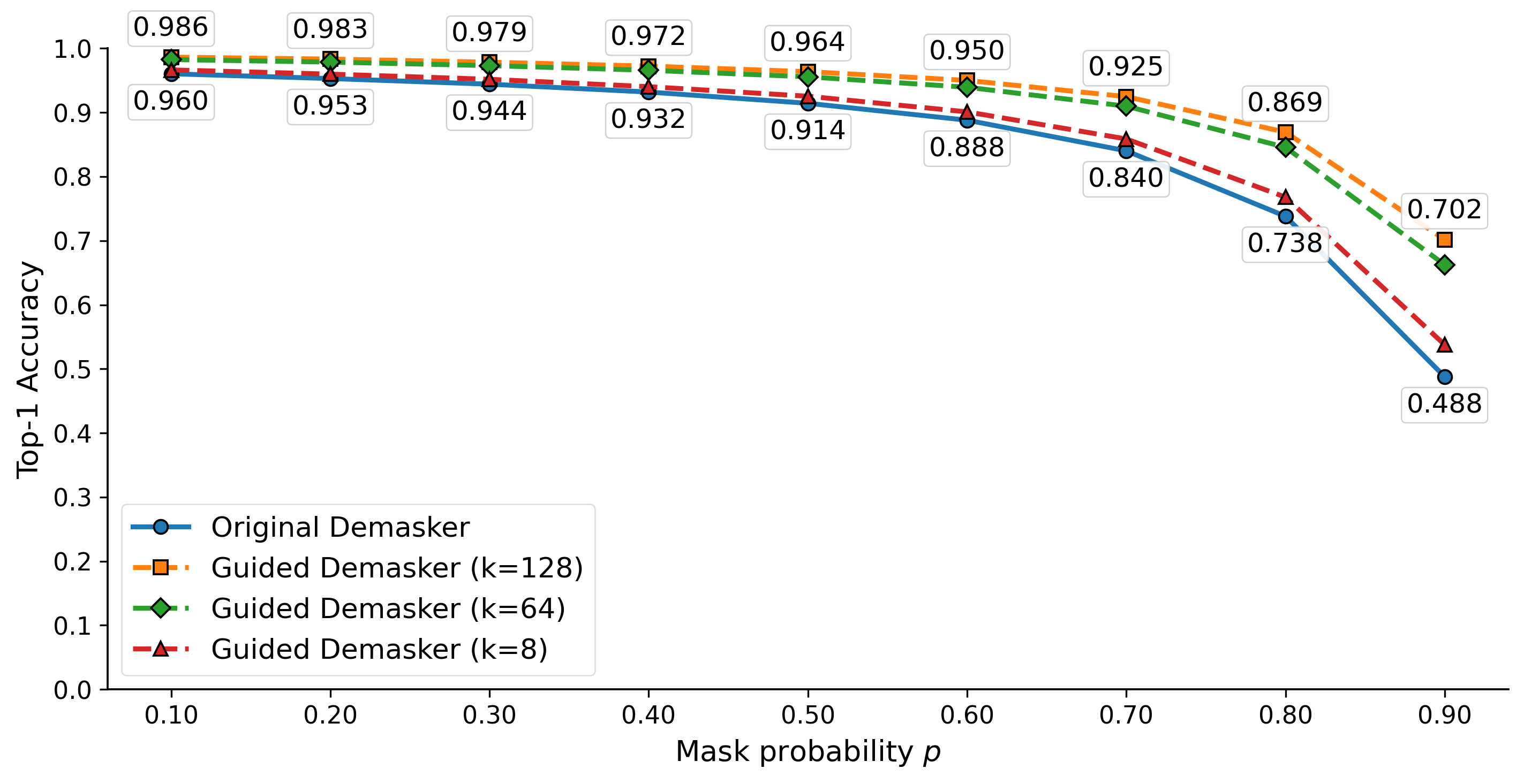}
        \caption{\textbf{LLaDA:} Top-1 token prediction accuracy.}
    \end{subfigure}
    \caption{\textbf{Guided-Demasker:} Top-1 token prediction accuracy in the MDLM (left) and LLaDA (right) settings versus the fraction of masked tokens, measured on a validation set of $1000$ sequences, and while varying $k$, the number of latent registers in $\rvz_0$.}
    \label{fig:demasker-performance}
\end{figure}

\subsection{MDM-Based autoencoding}

Equipped with a trained conditioned demasker $f_\theta(\rvx_t,t,\rvz_0)$ and
an embedding model $\rvz_0=h_\phi(\rvx_0)$, we turn to evaluate the
autoencoder scheme, in which a clean sequence $\rvx_0$ is encoded to a
continuum latent and an MDM algorithm decodes it back to tokens. The main
decoding hyper-parameters are the sequence length, the \rev{block size}, and the
NFE, i.e., the number of demasker \rev{forward passes. Decoding is semi-autoregressive: the sequence is split into consecutive blocks of the given size, and each block is fully unmasked before the next one starts.}
In these experiments, LLaDA is tested for Python programs of lengths $256$ and $512$, and MDLM handles free text of length $512$ tokens from OpenWebText.
In both 400M and 8B scales, the autoencoder reconstructs the input faithfully, even with relatively few demasking steps. Increasing NFE steadily improves reconstruction: in the LLaDA setting CER converges to around $0.10$ and CodeBERTScore F1 to $0.97$, while in the MDLM setting CER converges to around $0.02$ and BERTScore to $0.99$. The remaining discrepancies are predominantly local: formatting, whitespace, and identifier naming for code, or surface-level wording variations for text, rather than major semantic changes. \rev{\cref{tab:gen_metrics256,tab:gen_metrics512} report the LLaDA sweeps for lengths $256$ and $512$, and \cref{tab:gen_metrics_mdlm_tbd} the MDLM sweep for length $512$; more details are provided in~\cref{app: autoencoder}.}

\subsection{Unconditional text synthesis}\label{sec:uncond}

We now evaluate CRoCoDiL for unconditional synthesis. \emph{ConThenDisc} samples continuous registers and decodes them with the guided MDM; \emph{ConWithinDisc} adds a single mid-trajectory register refresh. Continuous-diffusion training is detailed in~\cref{app: Continuous Diffusion}.

\begin{figure}[t]
    \centering
    \begin{subfigure}[t]{0.49\linewidth}
        \centering
        \includegraphics[width=\linewidth]{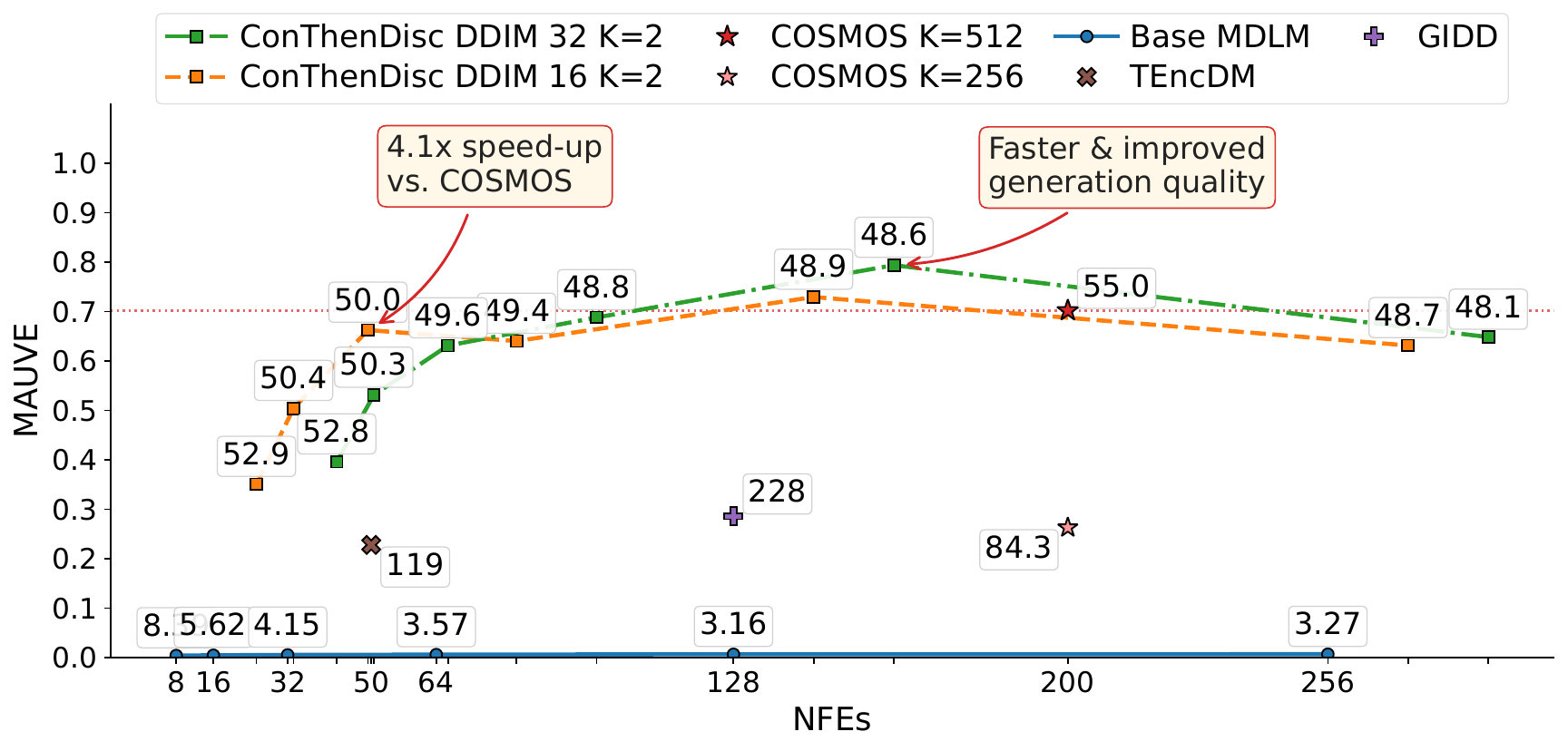}
        \caption{\textbf{MDLM ($n=512$):} \emph{ConThenDisc} with $K=2$ at $16$ or $32$ DDIM steps, compared against baseline MDLM, COSMOS ($K=256$, $K=512$), TEncDM, and GIDD. Red dashed line: COSMOS-$K=512$.}
    \end{subfigure}
    \hfill
    \begin{subfigure}[t]{0.49\linewidth}
        \centering
        \includegraphics[width=\linewidth]{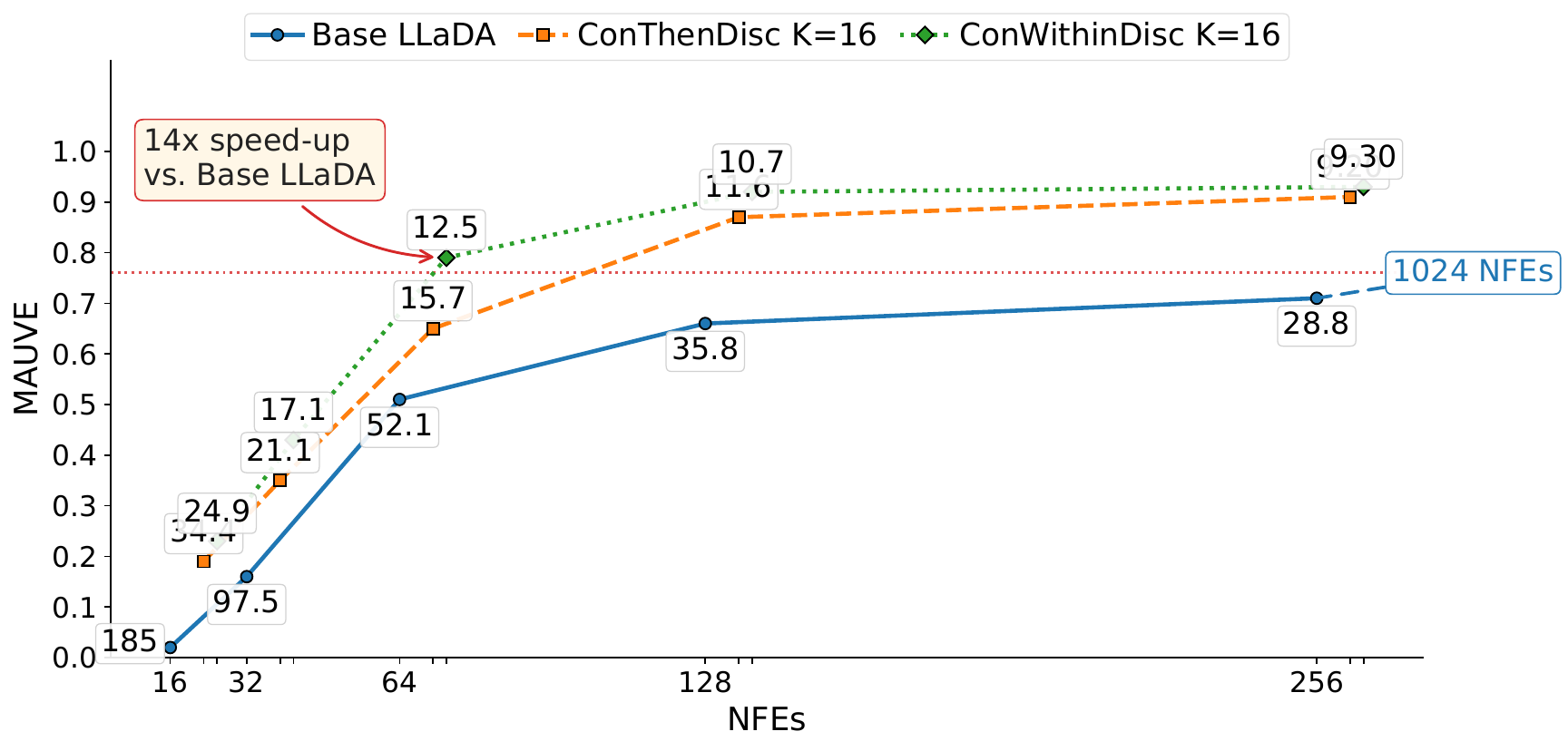}
        \caption{\textbf{LLaDA-8B ($n=1024$):} \emph{ConThenDisc} and \emph{ConWithinDisc} with $K=16$ at $128$ DDIM steps. Red dashed line: unguided Base LLaDA-8B using $1024$ demasking steps.}
    \end{subfigure}
    \caption{\textbf{Text Generation:} MAUVE and Generative Perplexity versus NFE, reported in parameter-weighted demasker-equivalent units (each continuous-diffusion call is weighted by the denoiser-to-demasker parameter ratio). \rev{CRoCoDiL decodes each sequence as a single block, whereas Base LLaDA uses $32$-token blocks.}}
    \label{fig:ConXDisc}
\end{figure}

\cref{fig:ConXDisc} (left) refers to the OpenWebText generation with MDLM (400M). This graph compares \emph{ConThenDisc} with the unguided demasker, TEncDM~\citep{shabalin2025tencdm}, GIDD~\citep{von2025generalized}, and two COSMOS~\citep{meshchaninov2025compressed} configurations. Despite using only $K=2$ registers, CRoCoDiL is competitive with COSMOS-$K=512$ and then surpasses it as the continuous diffusion (DDIM~\citep{song2021denoising} steps) budget increases.

Throughout these experiments we report NFE in demasker-equivalent units: each continuous-diffusion forward pass is counted as the denoiser-to-demasker parameter ratio in demasker calls. For MDLM, this ratio is $\approx 1.08$ ($458$M PixArt denoiser vs.\ $424$M MDLM demasker), so each DDIM step contributes $\approx 1.08$ NFE. COSMOS-$K=512$ attains MAUVE $0.702$ and Gen-PPL $55.0$ at NFE $\approx 200$. \emph{ConThenDisc} with DDIM~$16$ reaches MAUVE $\approx 0.66$ and Gen-PPL $\approx 50.0$ at NFE $\approx 49$, recovering essentially the same MAUVE as COSMOS-$K=512$ at a $\times 4.1$ \rev{reduction in demasker-equivalent NFE; a parameter-weighted cost breakdown is given in~\cref{app:cosmos-accounting}.} With DDIM~$32$ (NFE $\approx 163$, still $K=2$), it further improves to MAUVE $\approx 0.79$ and Gen-PPL $\approx 48.6$, surpassing COSMOS-$K=512$ on both metrics. Finally, note that the baseline MDLM is not competitive under these metrics, leading to poor MAUVE performance, and low Gen-PPL that reflects repetitive and poor text synthesis.

In our MDLM experiments, \emph{ConWithinDisc} did not yield a clear benefit over \emph{ConThenDisc} at matched compute, and therefore it has been omitted from the evaluation. A plausible explanation is that with a parameter ratio close to one, a mid-trajectory refresh costs nearly as much as allocating the same budget to additional DDIM steps within \emph{ConThenDisc}, leaving little room for the refresh to pay off. 

\cref{fig:ConXDisc} (right) reports the same tradeoff for LLaDA-8B. Here the continuous denoiser is much smaller than the $8$B demasker, so the parameter ratio is small and each DDIM step contributes only a fraction of a demasker NFE. This makes \emph{ConWithinDisc} attractive: one midpoint refresh adds less than $2$ NFE. The length-$512$ refresh and register-count ablations are reported in~\cref{app: More on ConThenDisc and ConWithinDisc}. %
At length $512$, \emph{ConWithinDisc} matches Base LLaDA using NFE$=40$ instead of NFE$=512$, \rev{a $\times 13$ reduction in demasker-equivalent NFE}; at length $1024$ it improves over Base LLaDA with NFE$=72$ instead of NFE$=1024$, \rev{a $\times 14$ reduction}. %

An interesting pattern emerges across both scales: the most efficient CRoCoDiL configurations that match the strong baselines operate at $16$ tokens per demasking step; the corresponding baselines are COSMOS for MDLM and Base LLaDA-8B for the large-scale setting.

\rev{To complement parameter-weighted NFE, we measure end-to-end throughput on an H100-80GB GPU, including latent generation, conditioning, and discrete decoding. At the selected operating point ($128$ continuous and $64$ discrete steps), CRoCoDiL achieves $\times 3.87$, $\times 8.30$, and $\times 16.39$ higher throughput than full-budget Base LLaDA at lengths $256$, $512$, and $1024$. On the complete-function and boundary-aware code metrics, CRoCoDiL outperforms Base LLaDA at every matched decoder budget. Full protocols and additional results are provided in~\cref{app:additional-evaluation}.}

\cref{app:cosmos-oracle} provides a comparison to COSMOS in the 8B scale. We consider an oracle setting, in which a given Python program of length $512$ tokens is encoded via $h_\phi(\rvx_0)$ to a rich latent representation of dimension $256 \times 1024$, followed by a  
one-step 8B decoder that aims to recover the original sequence. Both encoder and decoder are trained specifically for this setting, leading to Gen-PPL $\approx 52$ and MAUVE $\approx 0.07$, showing that iterative MDM decoding is essential at this scale.

\section{Discussion}

While Masked Diffusion Models (MDMs) offer a compelling approach to text generation, they face inherent performance hurdles. We introduce \emph{CRoCoDiL}, a framework that overcomes these bottlenecks to provide a substantial boost in synthesis speed and text quality. By first generating a 
``sketched'' latent representation in a continuous space and then converting it into tokens via a guided MDM, \emph{CRoCoDiL} achieves superior results, as demonstrated in our experiments with both MDLM~\citep{sahoo2024simple} and LLaDA~\citep{nie2025large}. Our future research will focus on extending the proposed algorithms to handle conditional (prompt-based) text synthesis, optimizing the continuous diffusion processes via distillation, %
and exploring more efficient latent designs. %

\paragraph{Limitations and future work.}
The proposed CRoCoDiL framework relies on several hyper-parameters (register count $K$, 
register-noise level for robustness, latent refresh timing in \emph{ConWithinDisc}) that are regime-dependent. 
Natural next steps of this work include (i) better design and training of the encoder $h_\phi(\cdot)$ in order to form a better-shaped latent domain; (ii) conditional generation 
and (iii) distillation of the 
continuous diffusion into few-step samplers.

\section*{Broader impact statement}
This work advances the efficiency and quality of masked diffusion language
models, which may benefit research on more capable and computationally
efficient generative systems. The same capabilities could also be misused,
for example to produce misleading or harmful text and code at scale. The
method itself does not provide safeguards against such misuse. Responsible
deployment should therefore include appropriate content and access controls,
evaluation for misuse and reliability, and documentation of the limitations
of the underlying language models and generated outputs.

\bibliographystyle{tmlr}
\bibliography{main}

\begin{thebibliography}{59}
\providecommand{\natexlab}[1]{#1}
\providecommand{\url}[1]{\texttt{#1}}
\expandafter\ifx\csname urlstyle\endcsname\relax
  \providecommand{\doi}[1]{doi: #1}\else
  \providecommand{\doi}{doi: \begingroup \urlstyle{rm}\Url}\fi

\bibitem[Arriola et~al.(2025{\natexlab{a}})Arriola, Gokaslan, Chiu, Yang, Qi,
  Han, Sahoo, and Kuleshov]{arriola2025block}
Marianne Arriola, Aaron Gokaslan, Justin~T Chiu, Zhihan Yang, Zhixuan Qi, Jiaqi
  Han, Subham~Sekhar Sahoo, and Volodymyr Kuleshov.
\newblock Block diffusion: Interpolating between autoregressive and diffusion
  language models.
\newblock In \emph{The Thirteenth International Conference on Learning
  Representations}, 2025{\natexlab{a}}.

\bibitem[Arriola et~al.(2025{\natexlab{b}})Arriola, Schiff, Phung, Gokaslan,
  and Kuleshov]{arriola2025encoder}
Marianne Arriola, Yair Schiff, Hao Phung, Aaron Gokaslan, and Volodymyr
  Kuleshov.
\newblock Encoder-decoder diffusion language models for efficient training and
  inference.
\newblock \emph{arXiv preprint arXiv:2510.22852}, 2025{\natexlab{b}}.

\bibitem[Austin et~al.(2021)Austin, Johnson, Ho, Tarlow, and van~den
  Berg]{austin2021structured}
Jacob Austin, Daniel~D Johnson, Jonathan Ho, Marc Tarlow, and Rianne van~den
  Berg.
\newblock Structured denoising diffusion models in discrete state-spaces.
\newblock In \emph{NeurIPS}, 2021.

\bibitem[Azangulov et~al.(2025)Azangulov, Pandeva, Prasad, Zazo, and
  Karmalkar]{azangulov2025parallel}
Iskander Azangulov, Teodora Pandeva, Niranjani Prasad, Javier Zazo, and Sushrut
  Karmalkar.
\newblock Parallel sampling from masked diffusion models via conditional
  independence testing.
\newblock \emph{arXiv preprint arXiv:2510.21961}, 2025.

\bibitem[Bachmann et~al.(2025)Bachmann, Allardice, Mizrahi, Fini, Kar, Amirloo,
  El-Nouby, Zamir, and Dehghan]{flextok}
Roman Bachmann, Jesse Allardice, David Mizrahi, Enrico Fini,
  O{\u{g}}uzhan~Fatih Kar, Elmira Amirloo, Alaaeldin El-Nouby, Amir Zamir, and
  Afshin Dehghan.
\newblock {FlexTok}: Resampling images into 1d token sequences of flexible
  length.
\newblock \emph{arXiv 2025}, 2025.

\bibitem[Ben-Hamu et~al.(2025)Ben-Hamu, Gat, Severo, Nolte, and
  Karrer]{benhamu2025eb}
Heli Ben-Hamu, Itai Gat, Daniel Severo, Niklas Nolte, and Brian Karrer.
\newblock Accelerated sampling from masked diffusion models via entropy bounded
  unmasking.
\newblock \emph{arXiv preprint arXiv:2505.24857}, 2025.

\bibitem[Chen et~al.(2023)Chen, Yu, Ge, Yao, Xie, Wu, Wang, Kwok, Luo, Lu, and
  Li]{chen2023pixartalpha}
Junsong Chen, Jincheng Yu, Chongjian Ge, Lewei Yao, Enze Xie, Yue Wu, Zhongdao
  Wang, James Kwok, Ping Luo, Huchuan Lu, and Zhenguo Li.
\newblock Pixart-$\alpha$: Fast training of diffusion transformer for
  photorealistic text-to-image synthesis, 2023.

\bibitem[Chen et~al.(2021)]{chen2021codex}
Mark Chen et~al.
\newblock Evaluating large language models trained on code.
\newblock \emph{arXiv preprint arXiv:2107.03374}, 2021.

\bibitem[Croitoru et~al.(2023)Croitoru, Hondru, Ionescu, and
  Shah]{croitoru2023diffusion}
Florinel-Alin Croitoru, Vlad Hondru, Radu~Tudor Ionescu, and Mubarak Shah.
\newblock Diffusion models in vision: A survey.
\newblock \emph{IEEE Transactions on Pattern Analysis and Machine
  Intelligence}, 45\penalty0 (9):\penalty0 10850--10869, 2023.

\bibitem[Dhariwal \& Nichol(2021)Dhariwal and Nichol]{dhariwal2021diffusion}
Prafulla Dhariwal and Alexander Nichol.
\newblock Diffusion models beat gans on image synthesis.
\newblock In \emph{Advances in Neural Information Processing Systems},
  volume~34, pp.\  8780--8794, 2021.

\bibitem[Dieleman et~al.(2022)Dieleman, Sartran, Roshannai, Savinov, Ganin,
  et~al.]{dieleman2022continuous}
Sander Dieleman, Laurent Sartran, Arman Roshannai, Nikolay Savinov, Yaroslav
  Ganin, et~al.
\newblock Continuous diffusion for categorical data.
\newblock \emph{arXiv preprint arXiv:2211.15089}, 2022.

\bibitem[Gao et~al.(2024)]{gao2024empowering}
Yifan Gao et~al.
\newblock Empowering diffusion models on the embedding space for text
  generation.
\newblock In \emph{ICLR}, 2024.

\bibitem[Gokaslan \& Cohen(2019)Gokaslan and Cohen]{gokaslan2019openwebtext}
Aaron Gokaslan and Vanya Cohen.
\newblock {OpenWebText} corpus.
\newblock \url{http://Skylion007.github.io/OpenWebTextCorpus}, 2019.

\bibitem[Gulrajani \& Hashimoto(2024)Gulrajani and
  Hashimoto]{gulrajani2024plaid}
Ishaan Gulrajani and Tatsunori~B Hashimoto.
\newblock Plaid: Likelihood-based diffusion language models.
\newblock In \emph{NeurIPS}, 2024.

\bibitem[Han et~al.(2023)Han, Kumar, and Tsvetkov]{han2023ssd}
Xiaochuang Han, Sachin Kumar, and Yulia Tsvetkov.
\newblock {SSD-LM}: Semi-autoregressive simplex-based diffusion language model.
\newblock In \emph{ACL}, 2023.

\bibitem[Ho et~al.(2020)Ho, Jain, and Abbeel]{ho2020denoising}
Jonathan Ho, Ajay Jain, and Pieter Abbeel.
\newblock Denoising diffusion probabilistic models.
\newblock \emph{Advances in neural information processing systems},
  33:\penalty0 6840--6851, 2020.

\bibitem[Ho et~al.(2022)Ho, Salimans, Gritsenko, Chan, Norouzi, and
  Fleet]{ho2022video}
Jonathan Ho, Tim Salimans, Alexey Gritsenko, William Chan, Mohammad Norouzi,
  and David~J Fleet.
\newblock Video diffusion models.
\newblock In \emph{Advances in Neural Information Processing Systems}, 2022.

\bibitem[Hoogeboom et~al.(2021)Hoogeboom, Nielsen, Jaini, Alaa, and
  Welling]{hoogeboom2021argmax}
Emiel Hoogeboom, Didrik Nielsen, Priyank Jaini, Ahmed Alaa, and Max Welling.
\newblock Argmax flows and multinomial diffusion: Towards non-autoregressive
  language models.
\newblock In \emph{NeurIPS}, 2021.

\bibitem[Karimi~Mahabadi et~al.(2024)]{mahabadi2024tess}
Rabeeh Karimi~Mahabadi et~al.
\newblock {TESS}: Text-to-text self-conditioned simplex diffusion.
\newblock In \emph{EACL}, 2024.

\bibitem[Kong et~al.(2021)Kong, Ping, Huang, Zhao, and
  Catanzaro]{kong2021diffwave}
Zhifeng Kong, Wei Ping, Jiaji Huang, Kexin Zhao, and Bryan Catanzaro.
\newblock Diffwave: A versatile diffusion probabilistic model for audio
  synthesis.
\newblock In \emph{International Conference on Learning Representations}, 2021.

\bibitem[Lemercier et~al.(2026)Lemercier, Geffner, Kreis, Mardani, Vahdat, and
  Juki{\'c}]{lemercier2026diladiff}
Jean-Marie Lemercier, Tomas Geffner, Karsten Kreis, Morteza Mardani, Arash
  Vahdat, and Ante Juki{\'c}.
\newblock {DiLaDiff}: Distilled latent-augmented diffusion for language
  modeling.
\newblock \emph{arXiv preprint arXiv:2605.23605}, 2026.

\bibitem[Li et~al.(2023)Li, Allal, Zi, Muennighoff, Kocetkov, Mou, Marone,
  Akiki, Li, Chim, et~al.]{li2023starcoder}
Raymond Li, Loubna~Ben Allal, Yangtian Zi, Niklas Muennighoff, Denis Kocetkov,
  Chenghao Mou, Marc Marone, Christopher Akiki, Jia Li, Jenny Chim, et~al.
\newblock Starcoder: may the source be with you!
\newblock \emph{arXiv preprint arXiv:2305.06161}, 2023.

\bibitem[Li et~al.(2022)Li, Thickstun, Gulrajani, Liang, and
  Hashimoto]{li2022diffusion}
Xiang~Lisa Li, John Thickstun, Ishaan Gulrajani, Percy Liang, and Tatsunori~B
  Hashimoto.
\newblock Diffusion-lm: Improves controllable text generation.
\newblock In \emph{NeurIPS}, 2022.

\bibitem[Liu et~al.(2025{\natexlab{a}})Liu, Broadrick, Niepert, and Van~den
  Broeck]{liu2025discrete}
Anji Liu, Oliver Broadrick, Mathias Niepert, and Guy Van~den Broeck.
\newblock Discrete copula diffusion.
\newblock In \emph{The Thirteenth International Conference on Learning
  Representations}, 2025{\natexlab{a}}.

\bibitem[Liu et~al.(2025{\natexlab{b}})Liu, Nam, Campbell, Stark, Xu, Jaakkola,
  and Gomez-Bombarelli]{liu2025think}
Sulin Liu, Juno Nam, Andrew Campbell, Hannes Stark, Yilun Xu, Tommi Jaakkola,
  and Rafael Gomez-Bombarelli.
\newblock Think while you generate: Discrete diffusion with planned denoising.
\newblock In \emph{The Thirteenth International Conference on Learning
  Representations}, 2025{\natexlab{b}}.

\bibitem[Liu et~al.(2025{\natexlab{c}})Liu, Yang, Zhang, Chen, Zou, Wei, Wang,
  and Zhang]{liu2025dllm}
Zhiyuan Liu, Yicun Yang, Yaojie Zhang, Junjie Chen, Chang Zou, Qingyuan Wei,
  Shaobo Wang, and Linfeng Zhang.
\newblock dllm-cache: Accelerating diffusion large language models with
  adaptive caching.
\newblock \emph{arXiv preprint arXiv:2506.06295}, 2025{\natexlab{c}}.

\bibitem[Lou et~al.(2024)Lou, Meng, and Ermon]{lou2024discrete}
Aaron Lou, Chenlin Meng, and Stefano Ermon.
\newblock Discrete diffusion modeling by estimating the ratios of the data
  distribution.
\newblock In \emph{Proceedings of the 41st International Conference on Machine
  Learning (ICML)}, 2024.

\bibitem[Luxembourg et~al.(2025)Luxembourg, Permuter, and
  Nachmani]{luxembourg2025plan}
Omer Luxembourg, Haim Permuter, and Eliya Nachmani.
\newblock Plan for speed--dilated scheduling for masked diffusion language
  models.
\newblock \emph{arXiv preprint arXiv:2506.19037}, 2025.

\bibitem[Meshchaninov et~al.(2025)Meshchaninov, Chimbulatov, Shabalin, Abramov,
  and Vetrov]{meshchaninov2025compressed}
Viacheslav Meshchaninov, Egor Chimbulatov, Alexander Shabalin, Aleksandr
  Abramov, and Dmitry Vetrov.
\newblock Compressed and smooth latent space for text diffusion modeling.
\newblock \emph{arXiv preprint arXiv:2506.21170}, 2025.

\bibitem[Morris et~al.(2023)Morris, Kuleshov, Shmatikov, and
  Rush]{morris2023text}
John Morris, Volodymyr Kuleshov, Vitaly Shmatikov, and Alexander~M Rush.
\newblock Text embeddings reveal (almost) as much as text.
\newblock In \emph{Proceedings of the 2023 Conference on Empirical Methods in
  Natural Language Processing}, pp.\  12448--12460, 2023.

\bibitem[Naeem et~al.(2020)Naeem, Oh, Uh, Choi, and Yoo]{ferjad2020icml}
Muhammad~Ferjad Naeem, Seong~Joon Oh, Youngjung Uh, Yunjey Choi, and Jaejun
  Yoo.
\newblock Reliable fidelity and diversity metrics for generative models.
\newblock In \emph{International Conference on Machine Learning}, 2020.

\bibitem[Nie et~al.(2025)Nie, Zhu, You, Zhang, Ou, Hu, Zhou, Lin, Wen, and
  Li]{nie2025large}
Shen Nie, Fengqi Zhu, Zebin You, Xiaolu Zhang, Jingyang Ou, Jun Hu, Jun Zhou,
  Yankai Lin, Ji-Rong Wen, and Chongxuan Li.
\newblock Large language diffusion models.
\newblock \emph{arXiv preprint arXiv:2502.09992}, 2025.

\bibitem[Patel et~al.(2025)Patel, Naseem, Pandey, Sultan, McCallum, and
  Fernandez~Astudillo]{patel2025pcg}
Dhruvesh Patel, Tahira Naseem, Gaurav Pandey, Md~Arafat Sultan, Andrew
  McCallum, and Ram{\'o}n Fernandez~Astudillo.
\newblock Improved sampling from masked diffusion models with position
  contrastive guidance.
\newblock In \emph{Structured Probabilistic Inference and Generative Modeling
  Workshop at NeurIPS}, 2025.
\newblock URL \url{https://openreview.net/forum?id=e0WmOrWbtc}.

\bibitem[Pillutla et~al.(2021)Pillutla, Swayamdipta, Zellers, Thickstun,
  Welleck, Choi, and Harchaoui]{pillutla2021mauve}
Krishna Pillutla, Swabha Swayamdipta, Rowan Zellers, John Thickstun, Sean
  Welleck, Yejin Choi, and Zaid Harchaoui.
\newblock Mauve: Measuring the gap between neural text and human text using
  divergence frontiers.
\newblock \emph{Advances in Neural Information Processing Systems},
  34:\penalty0 4816--4828, 2021.

\bibitem[Radford et~al.(2019)Radford, Wu, Child, Luan, Amodei, Sutskever,
  et~al.]{gpt2L}
Alec Radford, Jeffrey Wu, Rewon Child, David Luan, Dario Amodei, Ilya
  Sutskever, et~al.
\newblock Language models are unsupervised multitask learners.
\newblock \emph{OpenAI blog}, 1\penalty0 (8):\penalty0 9, 2019.

\bibitem[Ren et~al.(2025)]{ren2025qwen3}
Xuancheng Ren et~al.
\newblock Qwen3 embedding: Advancing text embedding and reranking through
  foundation models.
\newblock \emph{arXiv preprint arXiv:2506.05176}, 2025.

\bibitem[Rombach et~al.(2022)Rombach, Blattmann, Lorenz, Esser, and
  Ommer]{rombach2022high}
Robin Rombach, Andreas Blattmann, Dominik Lorenz, Patrick Esser, and Bj{\"o}rn
  Ommer.
\newblock High-resolution image synthesis with latent diffusion models.
\newblock In \emph{Proceedings of the IEEE/CVF conference on computer vision
  and pattern recognition}, pp.\  10684--10695, 2022.

\bibitem[Rudin(1987)]{rudin1987real}
Walter Rudin.
\newblock \emph{Real and Complex Analysis}.
\newblock McGraw-Hill, New York, 3rd edition, 1987.

\bibitem[Sahoo et~al.(2024)Sahoo, Arriola, Schiff, Gokaslan, Marroquin, Chiu,
  Rush, and Kuleshov]{sahoo2024simple}
Subham Sahoo, Marianne Arriola, Yair Schiff, Aaron Gokaslan, Edgar Marroquin,
  Justin Chiu, Alexander Rush, and Volodymyr Kuleshov.
\newblock Simple and effective masked diffusion language models.
\newblock \emph{Advances in Neural Information Processing Systems},
  37:\penalty0 130136--130184, 2024.

\bibitem[Shabalin et~al.(2025)Shabalin, Meshchaninov, Chimbulatov, Lapikov,
  Kim, Bartosh, Molchanov, Markov, and Vetrov]{shabalin2025tencdm}
Alexander Shabalin, Viacheslav Meshchaninov, Egor Chimbulatov, Vladislav
  Lapikov, Roman Kim, Grigory Bartosh, Dmitry Molchanov, Sergey Markov, and
  Dmitry Vetrov.
\newblock Tencdm: Understanding the properties of the diffusion model in the
  space of language model encodings.
\newblock In \emph{Proceedings of the AAAI Conference on Artificial
  Intelligence}, volume~39, pp.\  25110--25118, 2025.

\bibitem[Shariatian et~al.(2025)Shariatian, Durmus, Simsekli, and
  Peluchetti]{shariatian2025ladd}
Dario Shariatian, Alain Durmus, Umut Simsekli, and Stefano Peluchetti.
\newblock Latent-augmented discrete diffusion models.
\newblock \emph{arXiv preprint arXiv:2510.18114}, 2025.
\newblock ICML 2026.

\bibitem[Shi et~al.(2024)Shi, Han, Wang, Doucet, and
  Titsias]{shi2024simplified}
Jiaxin Shi, Ke~Han, Zhijian Wang, Arnaud Doucet, and Michalis~K Titsias.
\newblock Simplified and generalized masked diffusion for discrete data.
\newblock In \emph{NeurIPS}, 2024.

\bibitem[Song et~al.(2021)Song, Meng, and Ermon]{song2021denoising}
Jiaming Song, Chenlin Meng, and Stefano Ermon.
\newblock Denoising diffusion implicit models.
\newblock In \emph{International Conference on Learning Representations}, 2021.

\bibitem[Uziel et~al.(2025)Uziel, Chelly, Freifeld, and
  Pakman]{uziel2025clustering}
Roy Uziel, Irit Chelly, Oren Freifeld, and Ari Pakman.
\newblock Clustering via self-supervised diffusion.
\newblock \emph{arXiv preprint arXiv:2507.04283}, 2025.

\bibitem[Von~R{\"u}tte et~al.(2025)Von~R{\"u}tte, Fluri, Ding, Orvieto,
  Sch{\"o}lkopf, and Hofmann]{von2025generalized}
Dimitri Von~R{\"u}tte, Janis Fluri, Yuhui Ding, Antonio Orvieto, Bernhard
  Sch{\"o}lkopf, and Thomas Hofmann.
\newblock Generalized interpolating discrete diffusion.
\newblock \emph{arXiv preprint arXiv:2503.04482}, 2025.

\bibitem[Wang et~al.(2024)Wang, Bai, Tan, Wang, Fan, Bai, Chen, Liu, Wang, Ge,
  et~al.]{wang2024qwen2}
Peng Wang, Shuai Bai, Sinan Tan, Shijie Wang, Zhihao Fan, Jinze Bai, Keqin
  Chen, Xuejing Liu, Jialin Wang, Wenbin Ge, et~al.
\newblock Qwen2-vl: Enhancing vision-language model's perception of the world
  at any resolution.
\newblock \emph{arXiv preprint arXiv:2409.12191}, 2024.

\bibitem[Wu et~al.(2025{\natexlab{a}})Wu, Zhang, Xue, Diao, Fu, Liu, Molchanov,
  Luo, Han, and Xie]{wu2025fast}
Chengyue Wu, Hao Zhang, Shuchen Xue, Shizhe Diao, Yonggan Fu, Zhijian Liu,
  Pavlo Molchanov, Ping Luo, Song Han, and Enze Xie.
\newblock Fast-dllm v2: Efficient block-diffusion llm.
\newblock \emph{arXiv preprint arXiv:2509.26328}, 2025{\natexlab{a}}.

\bibitem[Wu et~al.(2025{\natexlab{b}})]{wu2025fastv1}
Chengyue Wu et~al.
\newblock {Fast-dLLM}: Training-free acceleration of diffusion {LLM} by
  enabling {KV} cache and parallel decoding.
\newblock \emph{arXiv preprint arXiv:2505.22618}, 2025{\natexlab{b}}.

\bibitem[Xie et~al.(2025)Xie, Xue, Feng, Hu, Sun, Li, and
  Zhang]{xie2025variational}
Tianyu Xie, Shuchen Xue, Zijin Feng, Tianyang Hu, Jiacheng Sun, Zhenguo Li, and
  Cheng Zhang.
\newblock Variational autoencoding discrete diffusion with enhanced dimensional
  correlations modeling.
\newblock \emph{arXiv preprint arXiv:2505.17384}, 2025.

\bibitem[Xu et~al.(2025)Xu, Geffner, Kreis, Nie, Xu, Leskovec, Ermon, and
  Vahdat]{xu2025energy}
Minkai Xu, Tomas Geffner, Karsten Kreis, Weili Nie, Yilun Xu, Jure Leskovec,
  Stefano Ermon, and Arash Vahdat.
\newblock Energy-based diffusion language models for text generation.
\newblock In \emph{The Thirteenth International Conference on Learning
  Representations}, 2025.

\bibitem[Yang et~al.(2023)Yang, Zhang, Song, Hong, Xu, Zhao, Shao, Zhang, Cui,
  and Yang]{yang2023diffusion}
Ling Yang, Zhilong Zhang, Yang Song, Shenda Hong, Runsheng Xu, Yue Zhao, Wentao
  Shao, Wentao Zhang, Bin Cui, and Ming-Hsuan Yang.
\newblock Diffusion models: A comprehensive survey of methods and applications.
\newblock \emph{ACM Computing Surveys}, 56\penalty0 (4):\penalty0 1--39, 2023.

\bibitem[Ye et~al.(2025)Ye, Xie, Zheng, Gao, Wu, Jiang, Li, and
  Kong]{ye2025dream}
Jiacheng Ye, Zhihui Xie, Lin Zheng, Jiahui Gao, Zirui Wu, Xin Jiang, Zhenguo
  Li, and Lingpeng Kong.
\newblock Dream 7b: Diffusion large language models.
\newblock \emph{arXiv preprint arXiv:2508.15487}, 2025.

\bibitem[Yi et~al.(2024)Yi, Chen, Zhang, Zhou, Zhu, and Kong]{yi2024diffusion}
Qiuhua Yi, Xiangfan Chen, Chenwei Zhang, Zehai Zhou, Linan Zhu, and Xiangjie
  Kong.
\newblock Diffusion models in text generation: a survey.
\newblock \emph{PeerJ Computer Science}, 10:\penalty0 e1905, 2024.

\bibitem[Zhang et~al.(2022)Zhang, Zhang, Zhang, et~al.]{zhang2022concrete}
Chenlin Zhang, Jiaming Zhang, Xian Zhang, et~al.
\newblock Concrete score matching: Generalized score matching for discrete
  data.
\newblock In \emph{NeurIPS}, 2022.

\bibitem[Zhang et~al.(2019)Zhang, Kishore, Wu, Weinberger, and
  Artzi]{zhang2019bertscore}
Tianyi Zhang, Varsha Kishore, Felix Wu, Kilian~Q Weinberger, and Yoav Artzi.
\newblock Bertscore: Evaluating text generation with bert.
\newblock \emph{arXiv preprint arXiv:1904.09675}, 2019.

\bibitem[Zhang et~al.(2025)Zhang, He, Levine, Zhao, Zhang, Rizvi, Zhang,
  Zappala, Ying, and van Dijk]{zhang2025non}
Yangtian Zhang, Sizhuang He, Daniel Levine, Lawrence Zhao, David Zhang, Syed~A.
  Rizvi, Shiyang Zhang, Emanuele Zappala, Rex Ying, and David van Dijk.
\newblock Non-markovian discrete diffusion with causal language models.
\newblock \emph{arXiv preprint arXiv:2502.09767}, 2025.

\bibitem[Zheng et~al.(2025)Zheng, Gong, Zhang, Chen, Gu, Zhou, Jaitly, and
  Zhang]{zheng2025cadd}
Huangjie Zheng, Shansan Gong, Ruixiang Zhang, Tianrong Chen, Jiatao Gu,
  Mingyuan Zhou, Navdeep Jaitly, and Yizhe Zhang.
\newblock Continuously augmented discrete diffusion model for categorical
  generative modeling.
\newblock \emph{arXiv preprint arXiv:2510.01329}, 2025.

\bibitem[Zhou et~al.(2026)Zhou, Yang, Hu, Wang, Zhang, Zhang, Mackey, Jaakkola,
  Bates, and Zhang]{zhou2026ccdd}
Cai Zhou, Chenxiao Yang, Yi~Hu, Chenyu Wang, Chubin Zhang, Muhan Zhang, Lester
  Mackey, Tommi Jaakkola, Stephen Bates, and Dinghuai Zhang.
\newblock Coevolutionary continuous discrete diffusion: Make your diffusion
  language model a latent reasoner.
\newblock In \emph{Proceedings of the 43rd International Conference on Machine
  Learning}, 2026.

\bibitem[Zhou et~al.(2023)Zhou, Alon, Agarwal, and
  Neubig]{zhou2023codebertscore}
Shuyan Zhou, Uri Alon, Sumit Agarwal, and Graham Neubig.
\newblock Codebertscore: Evaluating code generation with pretrained models of
  code.
\newblock In \emph{2023 Conference on Empirical Methods in Natural Language
  Processing, EMNLP 2023}, pp.\  13921--13937. Association for Computational
  Linguistics (ACL), 2023.

\end{thebibliography}

\clearpage
\setlength{\tabcolsep}{3pt}
\appendix
\onecolumn

\etocdepthtag.toc{appendixtoc}
\begingroup
\etocsettocstyle{%
  {\noindent\Large\bfseries Appendix Contents}\par%
  \vspace{0.3em}\hrule\vspace{0.6em}%
}{%
  \vspace{0.4em}\hrule\vspace{1.2em}%
}
\etocsettagdepth{mainpaper}{none}%
\etocsettagdepth{appendixtoc}{subsection}%
\tableofcontents
\endgroup
\newpage

\section{Diffusion models for language -- background}
\label{app:background}

In the past five years, diffusion models have taken the lead in image, video and audio synthesis tasks~\citep{ho2020denoising,dhariwal2021diffusion,rombach2022high,ho2022video,kong2021diffwave,yang2023diffusion,croitoru2023diffusion}. These algorithms all rely on a continuous formulation of diffusion algorithms, mostly assuming a Gaussian noise contamination and a corresponding denoising network, serving as foundational elements. Bringing these methods to handle language is far from trivial, as text is discrete and sequential. This gap poses a challenge that many recent papers have attempted to resolve. Roughly speaking, there are three general strategies in constructing a bridge between continuum-based diffusion models and language: 

\begin{itemize}
 \item {\bf Go Discrete:} The rationale of diffusion algorithms can be brought to the discrete domain, by intuitively replacing the Gaussian noise contamination by a masking operation or via random replacements of tokens. This line of work has been drawing much attention recently, yielding appealing results. MDM methods, such as MDLM~\citep{sahoo2024simple}, LLaDA~\citep{nie2025large} and Dream~\citep{ye2025dream}, belong to this thread of work. The following papers are additional representatives of this group:~\citep{austin2021structured,hoogeboom2021argmax,shi2024simplified,zhang2025non}. 

\item {\bf Go Continuum:} Text synthesis can be performed with classical (continuous) diffusion models, assuming that text could be embedded to the continuum, and decoding it back to tokens is within reach. This approach has been explored in a series of papers, with partial success -- the following are a few representatives of this group:~\citep{li2022diffusion,dieleman2022continuous,gulrajani2024plaid, gao2024empowering}. 

\item {\bf Go Midway:} Diffusion models can be reformulated thoroughly and rigorously while focusing on discrete data. This, for example, is the approach taken by the work on the Concrete-Score, which leans on ratios of probabilities. Another related option operates in the logits domain, forming simplex-based diffusion alternatives. Examples of this approach are the work reported in~\citep{zhang2022concrete,han2023ssd,lou2024discrete, mahabadi2024tess}. 

\end{itemize} 

For newcomers to this domain, the impression is likely to be that masked diffusion models (MDMs) are the preferred algorithms, as they have taken the lead in bringing diffusion models to language. We argue that ``the jury is still out'' on this question, as MDM still faces critical challenges, while the above-described alternatives are underexplored to a large extent. Our work offers a novel fusion of continuum and discrete, which preserves the core essence of MDM algorithms, while boosting them via a continuum embedding. This brings us to describe several closely related papers in \rev{\cref{sec:related}} that take a similar, yet different, path towards addressing similar goals.

\section{Theoretical analysis}\label{app: theoretical analysis}

\subsection{Masked diffusion: per-coordinate reverse transition and demasker details}
\label{app:mdm_details}

Here we provide the formal details deferred from~\cref{sec:background}.

\paragraph{Factorized reverse transition.}
Conditioned on a clean sample $\rvx_0$, the reverse transition factorizes 
across token positions without loss of generality,
\begin{equation}
\label{eq:fac_reverse}
    q(\rvx_s \mid \rvx_t, \rvx_0) := \prod_{i=1}^n q(x_s^i \mid x_t^i, x_0^i).
\end{equation}
For MDLM~\citep{sahoo2024simple}, each per-coordinate factor admits the following 
closed form:
\begin{equation}
\label{eq:factor_i}
    q(x_s^i \mid x_t^i, x_0^i) :=
    \begin{cases}
        \mathbf{e}_{x_t^i} & \text{if } x_t^i \neq \texttt{[M]}, \\[2pt]
        \dfrac{1 - \alpha_s}{1 - \alpha_t}\, \mathbf{e}_{\texttt{[M]}}
        + \dfrac{\alpha_s - \alpha_t}{1 - \alpha_t}\, \mathbf{e}_{x_0^i}
        & \text{if } x_t^i = \texttt{[M]}.
    \end{cases}
\end{equation}
That is, already-revealed tokens remain fixed, while masked positions either 
stay masked or are revealed to $x_0^i$ with probability proportional to the 
schedule increment.

\paragraph{Demasker.}
The demasker is a function 
$f_\theta : (\mathcal{V} \cup \{\texttt{[M]}\})^n \times \mathbb{R} 
\rightarrow \mathbb{R}^{n \times |\mathcal{V}|}$, parameterized by $\theta$, 
that estimates marginal distributions for masked positions. Concretely, 
$f_\theta^i(\rvx_t, t) := p_\theta(x_0^i \mid \rvx_t)$ approximates 
$q(x_0^i \mid \rvx_t)$ when $x_t^i = \texttt{[M]}$, and returns 
$\mathbf{e}_{x_t^i}$ otherwise. Substituting independent samples 
$\hat{x}_0^i \sim f_\theta^i(\rvx_t, t)$ into~\cref{eq:fac_reverse} 
yields the effective posterior of~\cref{eq:effective_reverse}.

\subsection{Limitations of the factorized approximation of joint token distributions}
\label{app:limitations_of_standard}

We now formally characterize the gap between the true joint distribution $q(\rvx_0 | \rvx_t)$ and the factorized approximation $p_\theta(\rvx_0 | \rvx_t) := \prod_{i=1}^n p_\theta(x_0^i | \rvx_t)$ typically employed in MDMs.
Here, $p_\theta(x_0^i | \rvx_t)$ represents the demasking model whose task is to recover masked tokens within the corrupted sequence $\rvx_t$.

{\bf Independence Gap: } We demonstrate that even if one is equipped with an optimal demasking model, $p_{\theta^*}(x_0^i | \rvx_t) = q(x_0^i | \rvx_t)$, the assumption of conditional independence inherently results in a loss of structural information. 

Let $p_{\theta^*}(\rvx_0 | \rvx_t) = \prod_{i=1}^n q(x_0^i | \rvx_t)$ denote the optimal factorized approximation of the true joint distribution $q(\rvx_0 | \rvx_t)$.
The \emph{Independence Gap}, which corresponds to the conditional total correlation of the tokens $\rvx_0$ given the partially masked sequence $\rvx_t$, is defined as:
\begin{align}
    C(\mathbf{x}_0 | \mathbf{x}_t) := D_{\text{KL}}\left( q(\mathbf{x}_0 | \mathbf{x}_t) \parallel \prod_{i=1}^n q(x_0^i | \mathbf{x}_t) \right)
\end{align}
Since natural language is governed by high-order dependencies (e.g., long-range syntactic constraints and local semantic coherence), the term $C(\mathbf{x}_0 | \mathbf{x}_t)$ is strictly positive.
This implies that the factorized model $p_\theta$ is theoretically incapable of perfectly recovering the true data distribution $q$ unless the tokens are truly conditionally independent. 

{\bf Semantic Incoherence: } Consider a simple case where the data distribution consists of two equally likely sequences: $\mathbf{s}_1 = (\text{cat}, \text{meows})$ and $\mathbf{s}_2 = (\text{dog}, \text{barks})$. If at time $t$ both tokens are masked ($\rvx_t = (\texttt{[M]}, \texttt{[M]})$), the true joint is:
\begin{equation}
    q(\rvx_0 | \rvx_t) = 0.5 \delta(\rvx_0 - \mathbf{s}_1) + 0.5 \delta(\rvx_0 - \mathbf{s}_2).
\end{equation}
The optimal marginals are $q(x_0^1 = \text{cat} | \rvx_t) = 0.5$ and $q(x_0^2 = \text{meows} | \rvx_t) = 0.5$, and similarly for the other pair. However, the factorized model $p_{\theta^*}$ yields:
\begin{equation}
    p_{\theta^*}(\rvx_0 | \rvx_t) = (0.5 \mathbf{e}_{\text{cat}} + 0.5 \mathbf{e}_{\text{dog}}) \otimes (0.5 \mathbf{e}_{\text{meows}} + 0.5 \mathbf{e}_{\text{barks}}).
\end{equation}
This assigns a $25\%$ probability to $(\text{cat}, \text{barks})$ and $(\text{dog}, \text{meows})$, which are out-of-distribution (OOD) sequences. This ``marginal drift'' forces the generative process to navigate through regions of the token space that do not correspond to valid language, often leading to a loss of global coherence in long-form synthesis.

{\bf Propagation of Error in Reverse Sampling: } In the iterative reverse process, we sample $\hat{\rvx}_0 \sim p_\theta(\rvx_0 | \rvx_t)$ and use it to compute the next step $\rvx_s$. Because $p_\theta$ ignores the token dependencies as described above, $\hat{\rvx}_0$ is likely to be incoherent.
When this incoherent $\hat{\rvx}_0$ is plugged into the effective reverse transition $p_\theta(\rvx_s | \rvx_t)$, defined in~\cref{eq:effective_reverse}, it is guided towards an inconsistent clean sequence, accumulating the error across the sampling trajectory.

\subsection{Formal justification for guided factorization}
\label{app:justification_for_guidance}

In this section, we provide the formal justification for the factorized reverse transition used in the \emph{CRoCoDiL} framework.
\rev{The following result identifies an idealized sufficient condition under which the reverse transition factorizes over token positions, given a continuous latent representation $\rvz_0 \in \R^d$. The training objective does not impose this condition explicitly, and the theorem should not be interpreted as establishing that the learned registers satisfy exact conditional independence. \cref{fig:demasker-performance,fig:demasker-performance-CE} evaluate how conditioning affects token-level predictions, while \cref{tab:code-levels,tab:owt8b} report generation with sampled registers.}

\begin{theorem}
\label{theo:Factorizability}
Let $\rvx_0 \in \mathcal{V}^n$ be a discrete sequence and $\rvz_0 \in \mathbb{R}^d$ be its continuous representation that contains the information on the cross-token dependencies in $\rvx_0$, such that the clean tokens are conditionally independent given $\rvz_0$:
\begin{equation}
    p_\theta(\rvx_0 | \rvz_0) := \prod_{i=1}^n p_\theta(x^i_0 | \rvz_0).
\end{equation}
Then, the reverse transition conditioned on $\rvz_0$ factorizes over the token positions:
\begin{equation}
    p_\theta(\rvx_s | \rvx_t, \rvz_0) = \prod_{i=1}^n p_\theta(x^i_s | \rvx_t, \rvz_0).
\end{equation}
\end{theorem}

\begin{proof}
Let $s<t$ be arbitrary diffusion timestamps. In the following proof we rely on the standard assumption in diffusion models that the forward noise process factorizes over dimensions (tokens) and is Markovian. Specifically:
\begin{enumerate}
    \item {Forward Independence:} $q(\rvx_t | \rvx_s) = \prod_{i=1}^n q(x^i_t | x^i_s)$.
    \item {Markov Property:} $q(\rvx_t | \rvx_s, \rvz_0) = q(\rvx_t | \rvx_s)$.
\end{enumerate}

We start by a factorization of the marginal probability $p_\theta(\rvx_s | \rvz_0)$, by integrating out $\rvx_0$:
\begin{align}
\label{eq:11a}
    p_\theta(\rvx_s | \rvz_0) &= \int q(\rvx_s | \rvz_0, \rvx_0) p_\theta(\rvx_0 | \rvz_0) \, d\rvx_0 \\ \nonumber & = \int q(\rvx_s | \rvx_0) p_\theta(\rvx_0 | \rvz_0) \, d\rvx_0 \\ \nonumber 
    &= \int \left[ \prod_{i=1}^n q(x^i_s | x^i_0) \right] \left[ \prod_{j=1}^n p_\theta(x^j_0 | \rvz_0) \right] \, d\rvx_0 \\ \nonumber 
    &= \prod_{i=1}^n \left( \int q(x^i_s | x^i_0) p_\theta(x^i_0 | \rvz_0) \, dx^i_0 \right) \\ \nonumber 
    &= \prod_{i=1}^n p_\theta(x^i_s | \rvz_0).
\end{align}
In the first transition in the above derivation, we omitted the appearance of $\rvz_0$ from $q(\rvx_s | \rvz_0, \rvx_0)$ due to the Markov property. In the fourth transition that interchanges the order between the integral and the multiplication, we rely on the fact that the integrand is separable (Fubini's Theorem~\citep{rudin1987real}). 

We proceed by factorizing the Reverse Transition, via Bayes' rule:
\begin{equation*}
    p_\theta(\rvx_s | \rvx_t, \rvz_0) = \frac{q(\rvx_t | \rvx_s, \rvz_0) p_\theta(\rvx_s | \rvz_0)}{p_\theta(\rvx_t | \rvz_0)}.
\end{equation*}
Using the forward independence assumption and the result from~\cref{eq:11a} (applied to both $s$ and $t$), we substitute the terms:
\begin{align*}
    p_\theta(\rvx_s | \rvx_t, \rvz_0) &= \frac{\left( \prod_{i=1}^n q(x^i_t | x^i_s) \right) \left( \prod_{i=1}^n p_\theta(x^i_s | \rvz_0) \right)}{\prod_{i=1}^n p_\theta(x^i_t | \rvz_0)} \\
    &= \prod_{i=1}^n \left( \frac{q(x^i_t | x^i_s) p_\theta(x^i_s | \rvz_0)}{p_\theta(x^i_t | \rvz_0)} \right).
\end{align*}
The term inside the product corresponds exactly to the single-token posterior $p_\theta(x^i_s | x^i_t, \rvz_0)$ derived via Bayes' rule for the $i$-th dimension. Thus,
\begin{equation*}
    p_\theta(\rvx_s | \rvx_t, \rvz_0) = \prod_{i=1}^n p_\theta(x^i_s | \rvx_t, \rvz_0).
\end{equation*}
and this concludes the proof. 
\end{proof}

\section{The role of robustness in training the encoder-demasker}
\label{app:RobustnessAblation}

Our guided demasker is conditioned on a continuous register bank $\rvz_0$. Since this conditioning channel is later driven by \emph{synthesized} latents in our hybrid generators, it is important that decoding does not hinge on an overly precise register configuration. We therefore aim to test how performance changes when the registers are perturbed at inference. Concretely, in the following experiments we consider LLaDA-8B for demasking Python programs of length $1024$ tokens. For each held-out sequence $\rvx_0$ we compute $\rvz_0=h_\phi(\rvx_0)$ and add i.i.d.\ Gaussian noise $\rve\sim\mathcal{N}(0,\sigma^2 I)$ to the registers before feeding them to the demasker. We sweep $\sigma$ and report masked-token prediction quality (Top-1 accuracy and mean cross-entropy) for two training variants: (i) our default Guided Demasker, trained with register-noise augmentation (\cref{alg:training}), and (ii) an otherwise identical ablation trained without this augmentation. \cref{fig:latent_noise_ablation_acc_ce} shows that noise-augmented training yields a markedly smoother degradation as $\sigma$ increases, while the no-augmentation variant deteriorates much earlier and more sharply. Overall, these results indicate that latent-noise augmentation makes the conditioning interface more stable to perturbations in register space.

Accordingly, we set the target expected cosine similarity between $\rvz_0$ and $\rvz_0+\rve$ to $0.8$ (\cref{sec:guidance-learning}). This value preserves most of the semantic content of $\rvz_0$ while still regularizing the conditioning channel. In our sweeps, a lower target weakens the guidance, whereas a higher target reintroduces the train-test mismatch that the noise augmentation is designed to reduce.

We also examine how the learned register space behaves along continuous paths between real samples. We sample $1000$ random program pairs (each $\approx256$ tokens), encode each program into registers $\rvz^{(a)}_0$ and $\rvz^{(b)}_0$, and form linear interpolations
$\rvz_\alpha=(1-\alpha)\rvz^{(a)}_0+\alpha \rvz^{(b)}_0$ for $\alpha\in\{0,0.25,0.5,0.75,1\}$, followed by the same per-register normalization used throughout. Each $\rvz_\alpha$ is decoded with the MDM decoder, and we compute MAUVE and generation perplexity, aggregating over the $1000$ decoded samples per $\alpha$.
\cref{fig:interp_alpha} summarizes the results. Quality remains high across the entire interpolation path: even at the midpoint between two \emph{random} programs ($\alpha=0.5$), the decoded samples retain strong MAUVE and only a moderate increase in perplexity (given by the numbers in the graph). The endpoints are, unsurprisingly, the easiest points to decode, but the overall curve indicates that linear mixing of two unrelated register banks still lands in a region that the decoder can map to fluent, diverse code. This supports the view that the register representation is well-behaved under continuous changes and that decoding is not fragile to such latent perturbations.

\begin{figure}[t]
    \centering
    \begin{subfigure}[t]{0.49\linewidth}
        \centering
        \includegraphics[width=\linewidth]{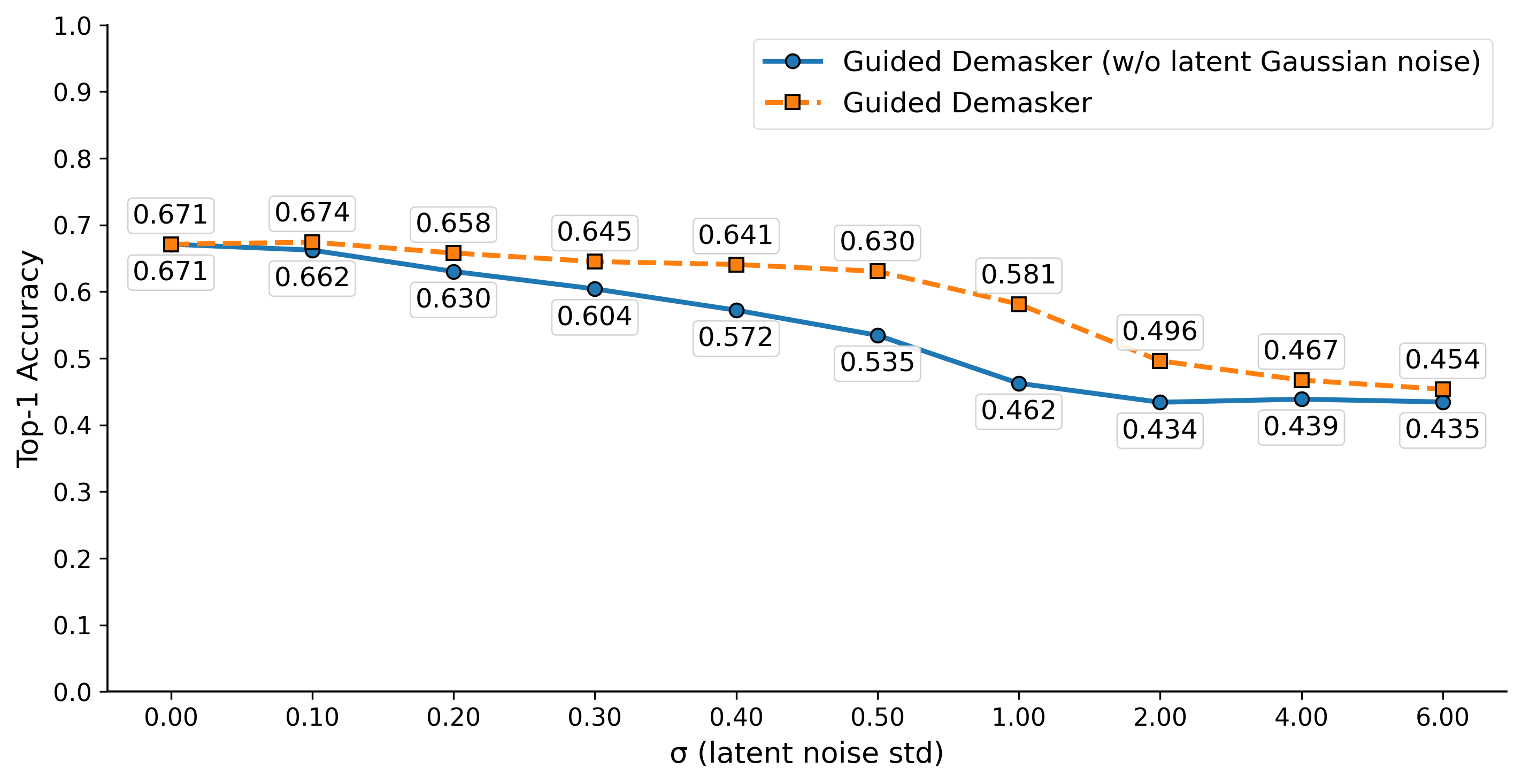}
        \caption{Top-1 accuracy.}
    \end{subfigure}
    \hfill
    \begin{subfigure}[t]{0.49\linewidth}
        \centering
        \includegraphics[width=\linewidth]{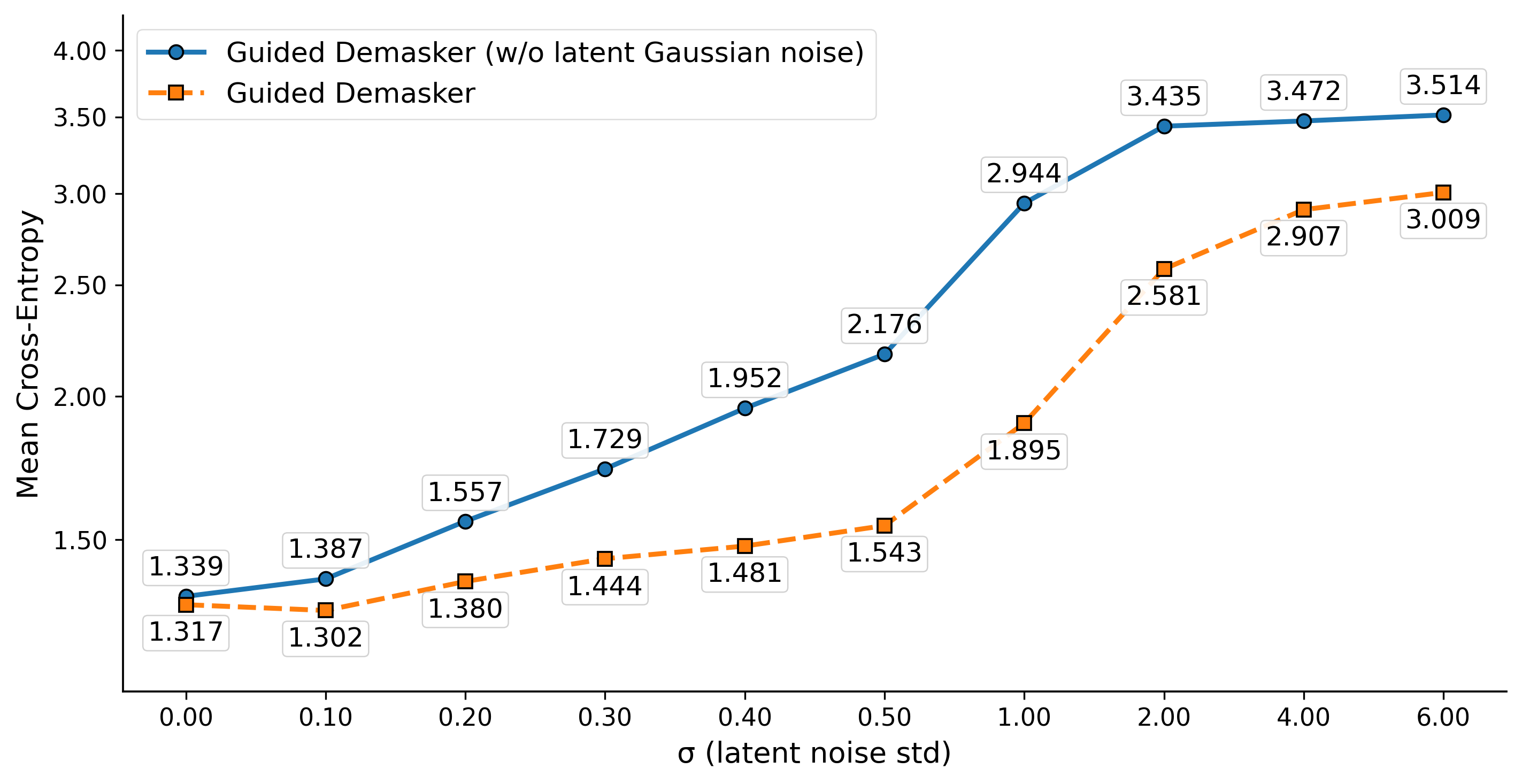}
        \caption{Mean cross-entropy.}
    \end{subfigure}
    \caption{Stability under latent perturbations. We add i.i.d.\ Gaussian noise of standard deviation $\sigma$ to the conditioning register bank at evaluation time and report masked-token Top-1 accuracy (left) and mean cross-entropy (right). The Guided Demasker trained with register-noise augmentation degrades gradually as $\sigma$ increases, while the variant trained without this augmentation deteriorates substantially earlier and more sharply. These results refer to $90\%$ masking of $1024$-tokens Python sequences, in the LLaDA-8B setting.}
    \label{fig:latent_noise_ablation_acc_ce}
\end{figure}

\begin{figure}[t]
    \centering
    \includegraphics[width=0.5\linewidth]{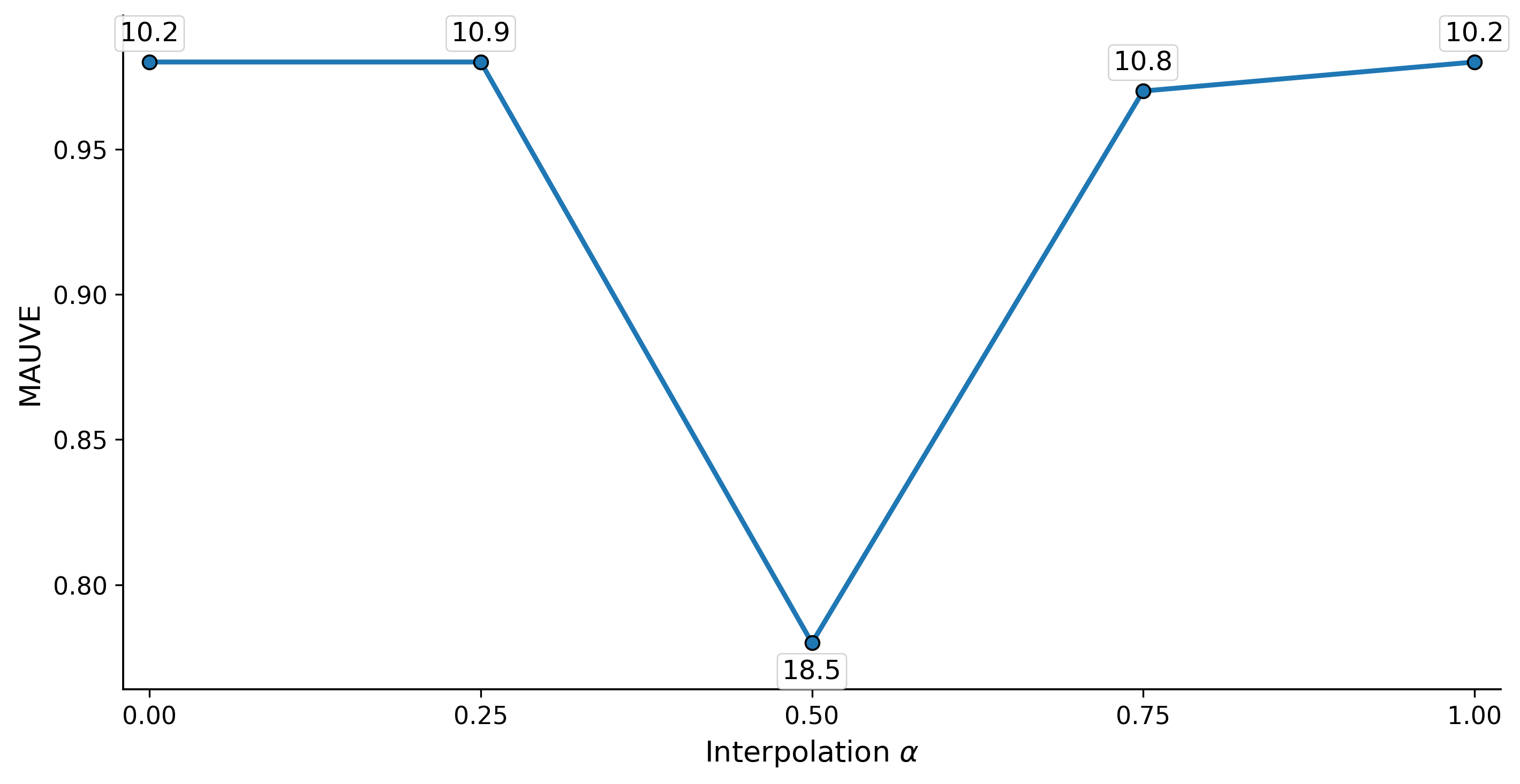}
    \caption{Register-space interpolation aggregated over $1000$ program pairs. We linearly interpolate between two encoded register banks and decode at $\alpha\in\{0,0.25,0.5,0.75,1\}$. We report MAUVE (higher is better); numbers denote Generative Perplexity (lower is better), averaged over the same decoded sets. Quality remains high across the interpolation path, including at the midpoint between two random programs.}
    \label{fig:interp_alpha}
\end{figure}

\section{Additional details on the continuous diffusion}
\label{app: Continuous Diffusion}

Given the trained encoder $h_\phi(\cdot)$, each discrete sequence $\rvx_0$ is mapped by it into a \emph{register bank}
$\rvz_0 = h_\phi(\rvx_0)\in\mathbb{R}^{K\times d}$, where nested dropout encourages a progressive organization across the register index.
We model the %
distribution of these registers by training a continuous diffusion model directly in this domain. \cref{alg:ContinuousDiffusion1} provides a description of this training procedure. The exact hyper-parameters of this stage are listed in~\cref{tab:training-mdlm} and~\cref{tab:training-llada} for the MDLM and LLaDA tracks, respectively.

More concretely, we define a standard forward noising process $q(\rvz_t\mid \rvz_0)$ over the register bank, producing $\rvz_t$ at diffusion time $t\in[0,1]$, and learn a DiT-style denoiser $g_\psi$.
As the denoiser backbone we use a PixArt-style Transformer adapted to 1D sequences of register tokens: a compact 28-layer variant with roughly $400M$ parameters~\citep{chen2023pixartalpha}.
We adopt the \emph{$\rvx_0$-prediction} parameterization, where the denoiser directly predicts the clean register bank,
\[
\hat \rvz_0 \;=\; g_\psi(\rvz_t,t),
\]
and train with a mean-squared error objective averaged over registers,
\[
\mathcal{L}_{\mathrm{cont}}(\psi)
\;=\;
\mathbb{E}_{\rvz_0,t}\Big[\;\big\|g_\psi(\rvz_t,t)-\rvz_0\big\|_2^2\Big],
\qquad \rev{\rvz_t \sim q(\rvz_t\mid \rvz_0)}.
\]
We use the standard DDPM~\citep{ho2020denoising} objective with a cosine noise schedule, and sample using DDIM~\citep{song2021denoising}. 

\begin{algorithm}[hb]
  \caption{Continuous-Then-Discrete Training}
  \label{alg:ContinuousDiffusion1}
  \begin{algorithmic}[1]
    \STATE {\bfseries Input:} data $\rvx_0 \sim P(\rvx)$, 
    \STATE Sample $t \sim \mathcal{U}([0,1])$
    \STATE $\rvz_0 = h_\phi(\rvx_0)$
    \STATE $\rev{\rvz_t \sim q_t(\rvz_t\mid\rvz_0)}$
    \STATE $\hat{\mathcal{L}}(\phi,\psi) = \|g_\psi(\rvz_t,t) - \rvz_0\|^2_2$
    \STATE Backpropagate on $\nabla_{\psi} \hat{\mathcal{L}} (\phi,\psi)$ and run optimizer
  \end{algorithmic}
\end{algorithm}

Nested dropout induces a progressive ordering over registers: early registers must remain informative even when later ones are absent, while later registers provide refinements. We mirror this structure in the continuous diffusion stage using \emph{block-wise timestep offsets} aligned with the same geometric grouping used by nested dropout. Concretely, we partition the $K$ registers into contiguous blocks whose sizes grow geometrically (e.g., $1,2,4,8,\ldots$). For a sampled base diffusion time $t$, we add a block-dependent offset $\Delta_b$ (increasing with the block index) and train all registers in block $b$ at the effective time
\[
t^{(b)} \;=\; \mathrm{clip}\big(t + \Delta_b\big), \qquad \Delta_1 < \Delta_2 < \cdots < \Delta_B .
\]
Importantly, all blocks are trained across the full range of diffusion times; the offsets do \emph{not} restrict a block to a narrower noise regime. Rather, they impose a consistent \emph{relative ordering} of effective times across blocks for the same base $t$: later blocks are always evaluated at a larger effective time than earlier blocks. This aligns the continuous diffusion objective with the progressive register ordering induced by nested dropout, encouraging the denoiser to recover the information carried by early registers before relying on later registers for refinements.

\emph{ConThenDisc} uses an \emph{unconditional} continuous diffusion prior over register banks. We first construct a dataset of clean registers $\rvz_0=h_\phi(\rvx_0)$ from the frozen encoder, and train a DDPM in register space with $\rvx_0$-prediction: we sample $t\in[0,1]$, draw $\rvz_t\sim q(\rvz_t\mid \rvz_0)$, and optimize
\[
\mathcal{L}_{\mathrm{cont}}(\psi)
\;=\;
\mathbb{E}\big[\|g_\psi(\rvz_t,t)-\rvz_0\|_2^2\big].
\]
At generation time, we sample $\hat \rvz_0$ by running the reverse process from Gaussian noise, and decode it into tokens using the guided demasker conditioned on $\hat \rvz_0$.

\emph{ConWithinDisc} trains a \emph{conditional} continuous diffusion model that produces register banks consistent with a partially observed sequence. Training proceeds as follows: Given a clean sequence $\rvx_0$, we sample a masking ratio $r\sim\mathrm{Unif}[0,1]$ and construct a partially masked sequence $\rvx^{(r)}$ by masking an $r$ fraction of tokens. We then reuse the same encoder $h_\phi$ to compute a conditioning embedding from $\rvx^{(r)}$, but apply an attention mask that blocks information flow through masked positions so that the conditioning signal depends only on visible tokens. The diffusion model is conditioned on this embedding through cross-attention layers, while the diffusion state is a noisy version of the \emph{clean} register bank $\rvz_0=h_\phi(\rvx_0)$. Concretely, we sample $t\in[0,1]$, draw $\rvz_t\sim q(\rvz_t\mid \rvz_0)$, and train
\[
\mathcal{L}_{\mathrm{cond}}(\psi)
\;=\;
\mathbb{E}\big[\|g_\psi(\rvz_t,t \mid h_\phi(\rvx^{(r)})) - \rvz_0\|_2^2\big].
\]
This objective teaches the model to denoise registers while respecting the partial evidence provided by visible tokens. At inference, we fix the masking ratio to $r=0.5$ to form the conditioning input, sample a register bank from the conditional diffusion model, and pass it to the guided demasker to complete the sequence.

To compare the unconditional (\emph{ConThenDisc}) and conditional (\emph{ConWithinDisc}) continuous diffusion models, we report precision--recall statistics using the PRDC framework of \citep{ferjad2020icml}. Specifically, we generate 5{,}000 samples and represent each sample by \(K{=}128\) register vectors. We treat each register position as an independent feature space: for every register index \(j\in\{1,\ldots,K\}\), we compute PRDC between generated and real samples using the corresponding vectors \(\{\mathbf{z}_0^{(j)}\}\), and report the final score by averaging each PRDC metric across all 128 registers.
For the unconditional model, we obtain \(\text{precision}=0.956\pm0.010\), \(\text{recall}=0.651\pm0.048\), \(\text{density}=1.105\pm0.099\), and \(\text{coverage}=0.870\pm0.021\). The conditional model substantially improves recall to \(\approx 0.74\) (with similar precision), indicating better coverage of the real-data manifold.

To directly verify that conditioning improves denoising, we measure reconstruction quality across diffusion times by computing PSNR between the denoiser prediction and the ground-truth registers as a function of $t$. Specifically, we compare $\mathrm{PSNR}(g_\psi(\rvz_t,t),\rvz_0)$ for the unconditional model against $\mathrm{PSNR}(g_\psi(\rvz_t,t\mid h_\phi(\rvx^{(r)})),\rvz_0)$ for the conditional model. As shown in~\cref{fig:psnr}, the conditional denoiser attains higher PSNR over a wide range of timesteps, indicating that it effectively leverages the partial-token evidence provided through $h_\phi(\rvx^{(r)})$.

\begin{figure}[t]
  \centering
  \includegraphics[width=0.6\linewidth]{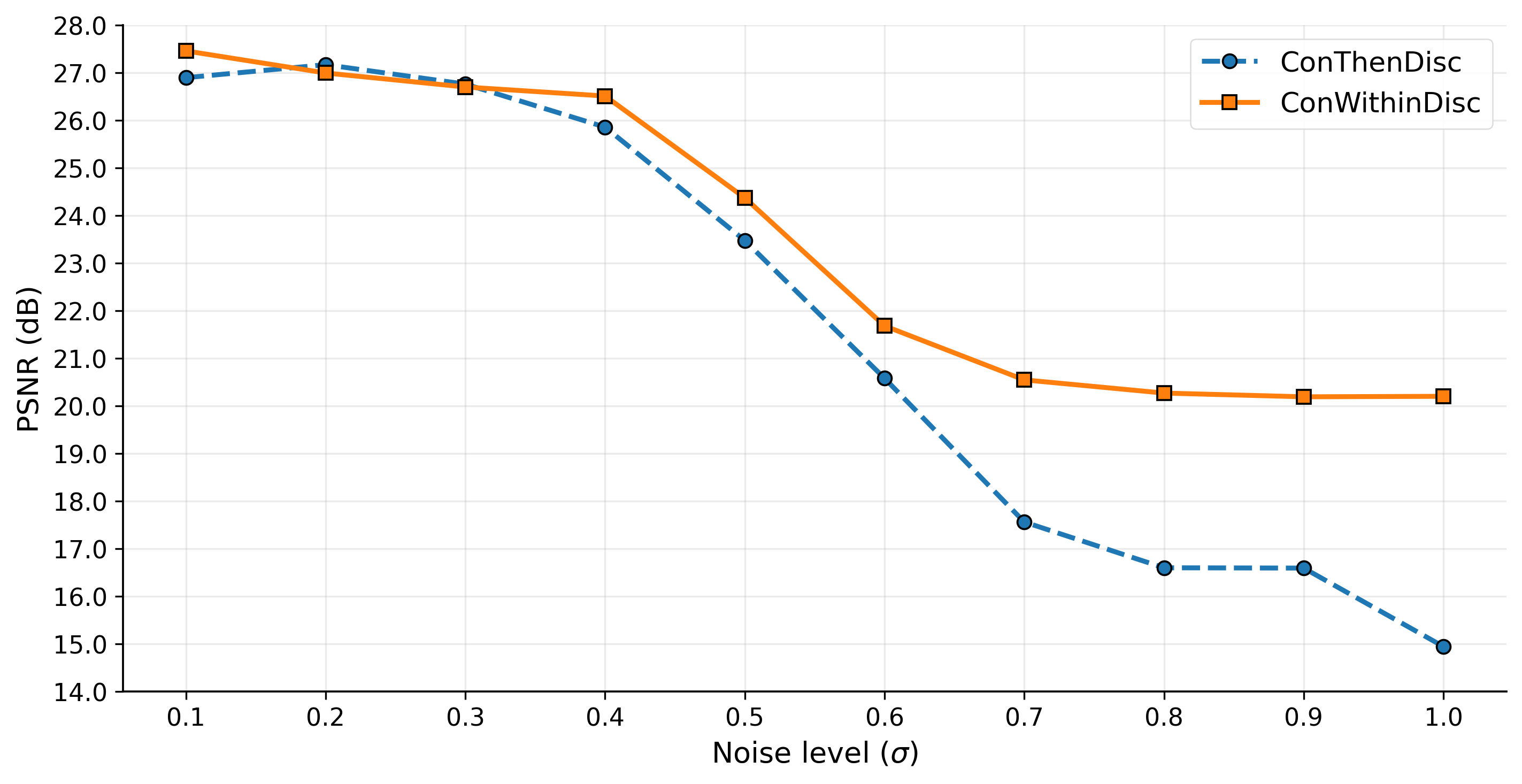}
  \caption{PSNR of denoiser predictions against ground-truth registers across diffusion time. Conditioning (\emph{ConWithinDisc}) yields consistently higher PSNR than the unconditional model (\emph{ConThenDisc}), demonstrating improved reconstruction from noisy registers. This experiment refers to the encoding attached with the LLaDA-8B demasker.}
  \label{fig:psnr}
\end{figure}

\section{Additional details on the guided-demasker training}
\label{app:Encoder-Demasker-Training}

In training the conditioned demasker, we jointly learn the parameters of an encoder $h_\phi$ and a demasking decoder $f_\theta$ to obtain a compact continuous representation that preserves the global constraints needed to reconstruct the discrete sequence, while remaining well-structured for the continuous generative model in the next stage. Given a clean sequence $\rvx_0$, the encoder produces a bank of continuous \emph{registers} $\rvz_0 = h_\phi(\rvx_0)$, and the demasker is trained to reconstruct $\rvx_0$ from a corrupted state $\rvx_t$ while conditioning on these registers.

We instantiate $h_\phi$ with a Qwen embedding backbone, but depart from standard single-vector pooling by introducing a small bank of learned suffix registers. Specifically, we append $K-1$ learned token embeddings to the input, and additionally retain the single summary vector corresponding to the final valid token position as usually used in Qwen, resulting in $K$ embedding layers. \rev{The $K-1$ appended registers are free learned vectors fed through \texttt{inputs\_embeds}, bypassing the embedding table, and the encoder keeps its causal mask, so each register attends to the whole text and to the registers before it:}
\[
\rvz_0 \;=\; \big(\rvz_0^{(1)},\ldots,\rvz_0^{(K)}\big), \qquad \rvz_0^{(j)} \in \mathbb{R}^d.
\]
Note that in the original Qwen embedder only one vector is used, but this may not be enough for our needs: 
A multi-vector bottleneck distributes capacity across coarse and fine-grained attributes of the sequence, whereas a single vector often forces an overly lossy compression. To better align these embeddings with the later-applied continuous diffusion, we normalize each register to unit length and rescale by $\sqrt{d}$:
\[
\rvz_0^{(j)} \leftarrow \sqrt{d}\cdot \frac{\rvz_0^{(j)}}{\|\rvz_0^{(j)}\|_2}\,.
\]
This is done as in previous work~\citep{uziel2025clustering}. 
While alternative normalizations are possible -- e.g., applying LayerNorm with learned affine parameters~\citep{gao2024empowering} -- this changes the scale of the register distribution and, in turn, the effective noise magnitude in the continuous diffusion stage. In practice, adopting such learned scaling typically requires additional calibration, either via a dedicated hyperparameter search over the noise/rescaling schedule or by estimating the appropriate normalization statistics (mean and variance) on a held-out set.

A key requirement for the subsequent latent diffusion model is that $\rvz_0$ is \emph{progressively organized}: earlier registers should remain informative even when later registers are absent. We enforce this property using nested dropout over registers \citep{flextok}.  For each example we sample $k \in \{1,\ldots,K\}$ and expose only the first $k$ registers to the demasker, withholding the rest. This induces an ordered decomposition in which global ``core'' information concentrates in early registers while residual details are delegated to later ones. Empirically, removing nested dropout typically has limited effect on reconstruction quality, but substantially harms unconditional generation after fitting a continuous diffusion prior over $\rvz_0$ (e.g., higher perplexity), indicating that the structured bottleneck primarily regularizes latent geometry for the continuous generative stage rather than improving raw auto-encoding fidelity.

The discrete backbone of the demasker $f_\theta$ depends on the track. For the LLaDA track we use the LLaDA-8B demasker; for the MDLM track we use the $\sim$400M MDLM demasker. In both cases, conditioning is implemented identically by prepending the available register prefix as continuous embeddings. Given a corrupted state $\rvx_t$, we form the demasker input as
\[
\big[\;\langle\mathrm{START}\rangle,\; \rvz_0^{(1)},\ldots,\rvz_0^{(k)},\; \langle\mathrm{END}\rangle,\; E(\rvx_t)\;\big],
\]
where $\langle\mathrm{START}\rangle$ and $\langle\mathrm{END}\rangle$ are learned boundary embeddings that delimit the conditioning channel, and $E(\cdot)$ is the token embedding lookup. A common failure mode in conditional demasking is that the demasker under-utilizes the conditioning signal whenever sufficient local context is available. To explicitly prevent this collapse, we employ an \emph{all-mask utilization schedule}: with probability of $0.25$ we set $\rvx_t$ to the fully masked sequence (all tokens replaced by $[M]$) while still providing $\rvz_0$, forcing $f_\theta$ to reconstruct from the register prefix rather than from token-to-token correlations.

Finally, we train $h_\phi$ and $f_\theta$ jointly under the standard masked-diffusion corruption process: we sample a masking level $t$, construct $\rvx_t$ by masking a subset of tokens, compute $\rvz_0=h_\phi(\rvx_0)$, and optimize cross-entropy on masked positions, i.e., maximize the conditional likelihood $p_\theta(\rvx_0 \mid \rvx_t, \rvz_0)$.

Prepending continuous registers changes the positional layout seen by the Transformer. We therefore keep RoPE as the base positional encoding and introduce a lightweight two-axis variant inspired by MRoPE~\citep{wang2024qwen2}. The model marks each input position as belonging either to the \emph{register prefix} (including the boundary tokens) or to the \emph{text stream}, and applies rotary embeddings accordingly. Text tokens use the standard RoPE rotation with a global absolute position index. For prefix tokens, we additionally define a prefix-local index that resets at the beginning of the prefix block (i.e., counts positions within the register segment). We then allocate a small subset of rotary frequency pairs to use the additional axis inside the prefix, while the remaining channels continue to use the global axis. This cleanly separates ``where'' within the register block from absolute positions in the text, stabilizing prefix--text attention and reducing positional shift artifacts under continuous-prefix conditioning.

To verify that the two-axis variant is responsible for the stability of guided demasking rather than being incidental, we ablate MRoPE on the MDLM track by reverting to standard RoPE on both register and text positions, holding all other training hyper-parameters fixed. Figure~\ref{fig:nomrope-ablation} reports the resulting top-1 accuracy and cross-entropy as a function of masking probability. With MRoPE, top-1 accuracy and cross-entropy are $0.997$ and $0.013$ at $10\%$ masking, and $0.988$ and $0.059$ at $90\%$ masking; without MRoPE these become $0.958$ and $0.082$ at $10\%$, and $0.743$ and $0.561$ at $90\%$. The degradation is largest at high masking ratios, the regime in which the demasker must rely most on the register prefix, confirming that the two-axis variant is essential for stable prefix-text attention under continuous-register conditioning.

\paragraph{Optimization (LLaDA track).} For the LLaDA backbone, we optimize using AdamW with learning rate $2.5\times 10^{-5}$ and weight decay $0.1$, for a total of $3$ epochs with global batch size $512$ (full configuration in~\cref{tab:training-llada}). 
Training is staged for stability: we first warm up the encoder by updating only $h_\phi$ for the first $4000$ optimization steps while keeping the decoder $f_\theta$ frozen, and then unfreeze $f_\theta$ and continue training both components jointly for the remainder of training.

\paragraph{Optimization (MDLM).} For the MDLM backbone we follow the same staged training protocol described above, with the same optimizer (AdamW), learning rate ($2.5\times 10^{-5}$), weight decay ($0.1$), global batch size ($512$), and $4000$-step encoder-only warmup. The only difference is the total training budget: we train for $10$ epochs instead of $3$, reflecting the smaller backbone, the smaller training corpus (OpenWebText compared to StarCoder), and the larger ratio of parameters that are updated jointly with the encoder. The full configuration is summarized in~\cref{tab:training-mdlm}.

\paragraph{Cross-entropy on both backbones.} \cref{fig:demasker-performance-CE} reports cross-entropy on a held-out validation set of $1000$ sequences for both the MDLM and LLaDA settings, drawn from OpenWebText and StarCoder respectively. We compare the unguided baseline against guided variants with different values of $k$ (sweeping $k\in\{2,4,8,16,32\}$ for MDLM and $k\in\{8,64,128\}$ for LLaDA, mirroring the main-paper sweep in~\cref{fig:demasker-performance}). Consistent with the top-1 accuracy results, cross-entropy deteriorates for all models as the masking ratio increases, while conditioning consistently improves performance, with a gap that widens as more registers are provided. The same trend holds on both backbones, showing that the guidance benefit is not specific to a single corpus or model scale.

\begin{figure}[t]
    \centering
    \begin{subfigure}[t]{0.49\linewidth}
        \centering
        \includegraphics[width=\linewidth]{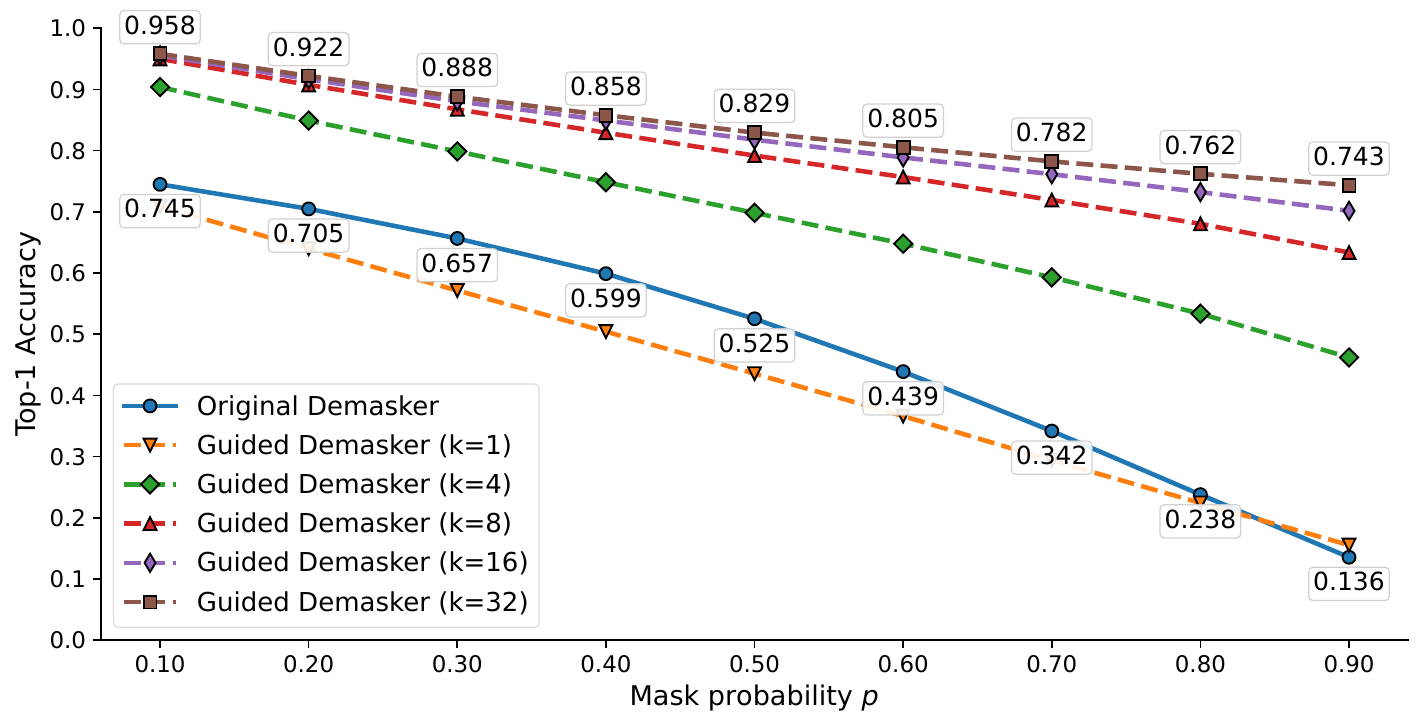}
        \caption{Top-1 token prediction accuracy.}
    \end{subfigure}
    \hfill
    \begin{subfigure}[t]{0.49\linewidth}
        \centering
        \includegraphics[width=\linewidth]{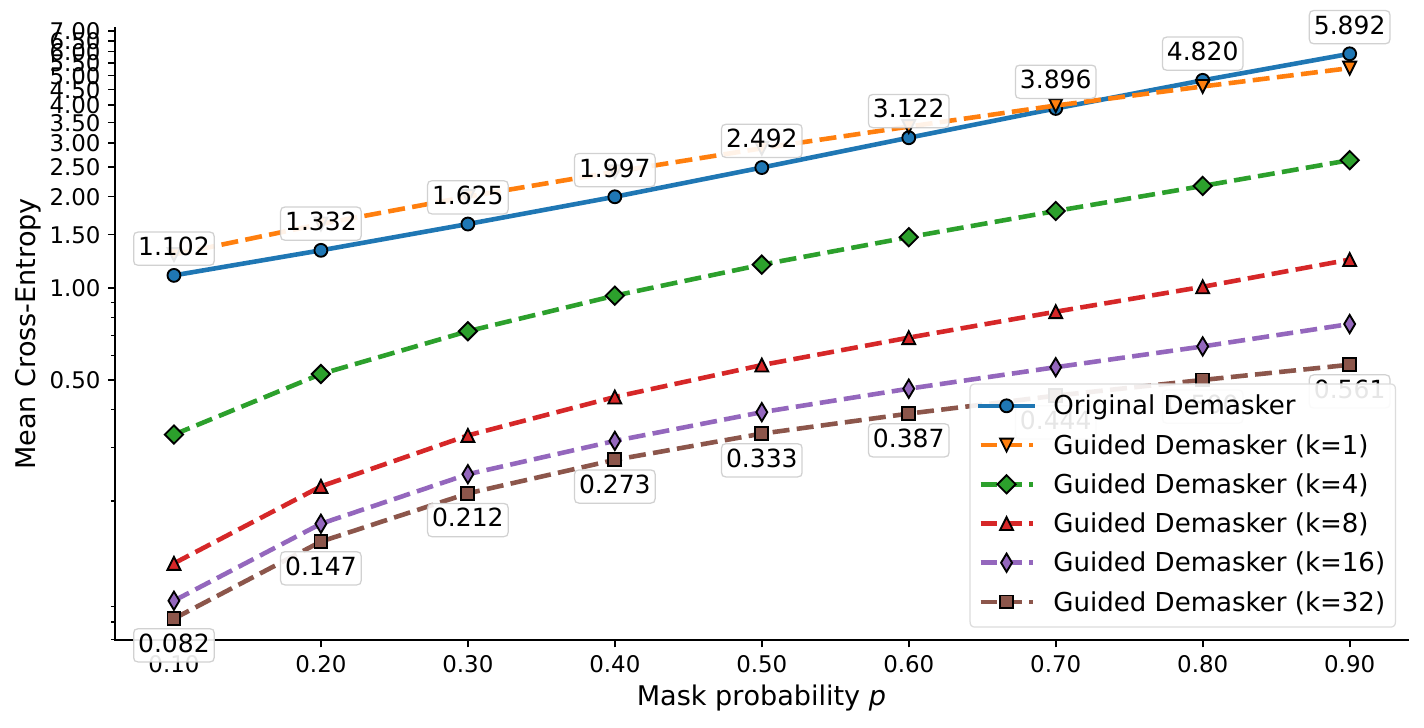}
        \caption{Cross-entropy.}
    \end{subfigure}
    \caption{\textbf{Ablation: guided demasker without MRoPE on continuous latents (MDLM, OpenWebText).} Top-1 token prediction accuracy (left) and cross-entropy (right) versus masking probability, with the two-axis MRoPE on the continuous-register prefix replaced by standard RoPE. }
    \label{fig:nomrope-ablation}
\end{figure}

\begin{figure}[t]
    \centering
    \begin{subfigure}[t]{0.49\linewidth}
        \centering
        \includegraphics[width=\linewidth]{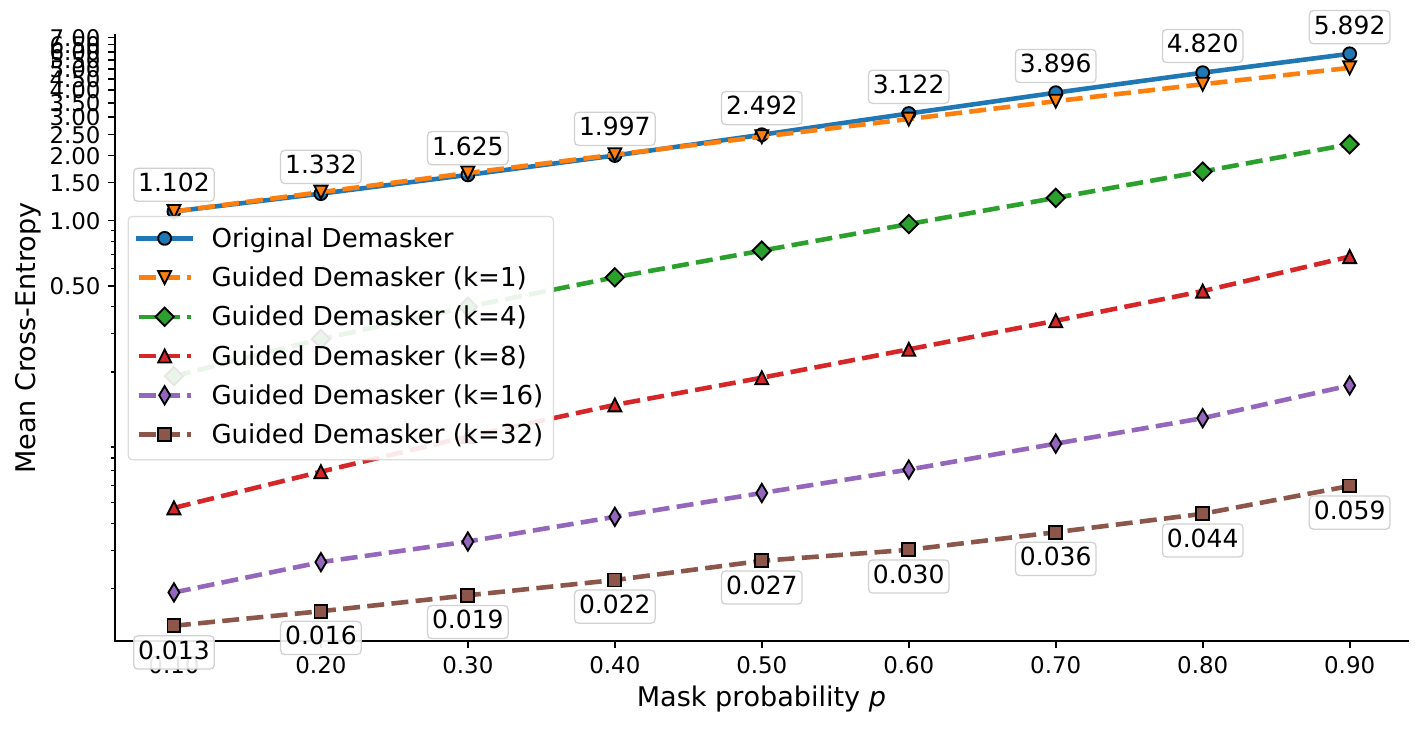}
        \caption{\textbf{MDLM (OpenWebText):} Cross-Entropy.}
    \end{subfigure}
    \hfill
    \begin{subfigure}[t]{0.49\linewidth}
        \centering
        \includegraphics[width=\linewidth]{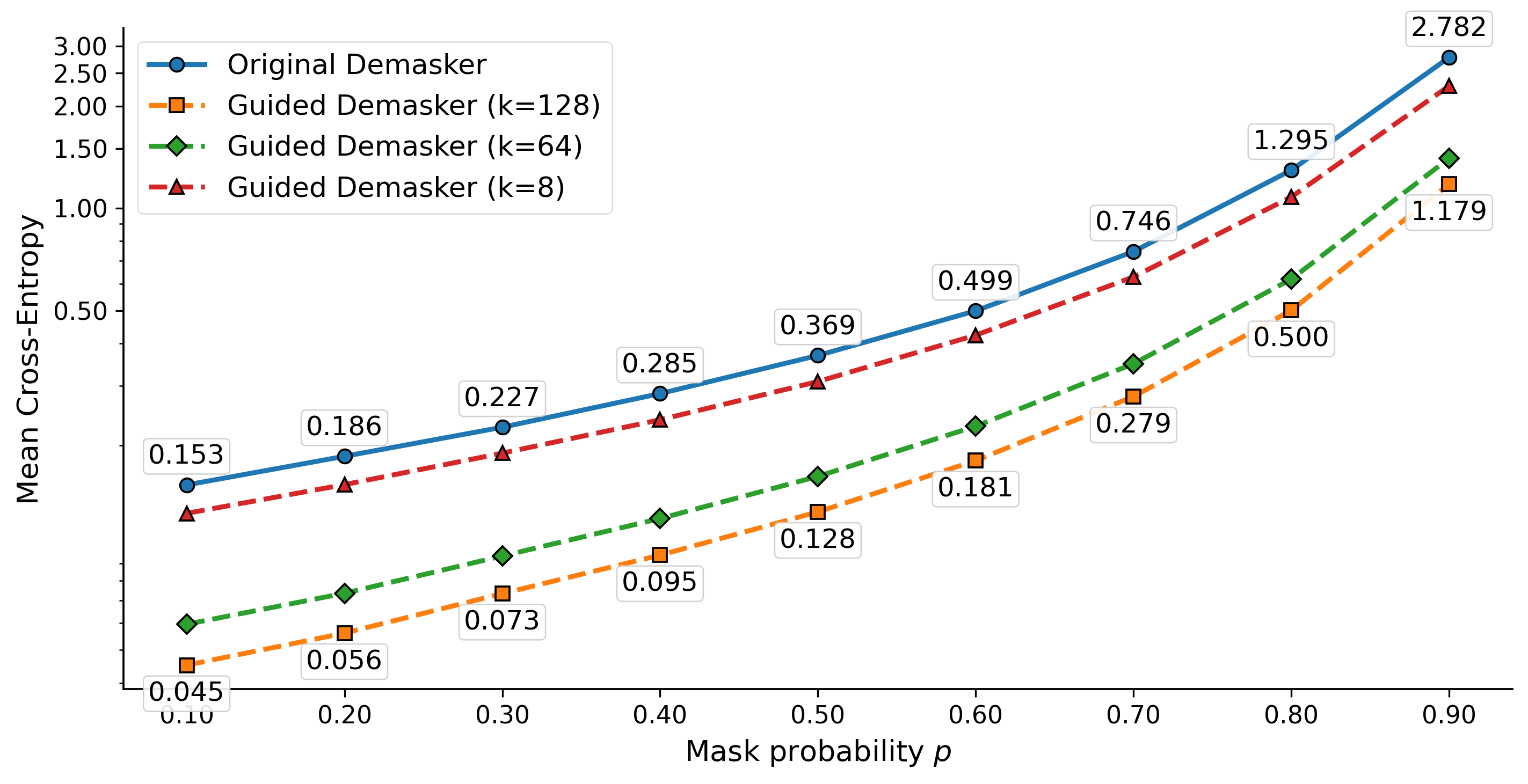}
        \caption{\textbf{LLaDA (StarCoder Python):} Cross-Entropy.}
    \end{subfigure}
    \caption{\textbf{Guided-Demasker:} Cross-Entropy in the MDLM (left) and LLaDA (right) settings versus masking probability ($0$ for no-mask), measured on a validation set of $1000$ sequences. Curves correspond to the unguided baseline and to guided variants with different numbers of registers $k$.}
    \label{fig:demasker-performance-CE}
\end{figure}

\section{Additional details on the MDM-based autoencoder}
\label{app: autoencoder}
\cref{fig:AutoEncoder} and~\cref{alg:auto-encoding} present the MDM-based autoencoding procedure. Note that lines 3-9 in this algorithm represent a regular MDM, empowered by the $\rvz_0$ guidance.

\begin{figure}[t]
    \centering
    \includegraphics[width=0.52\linewidth]{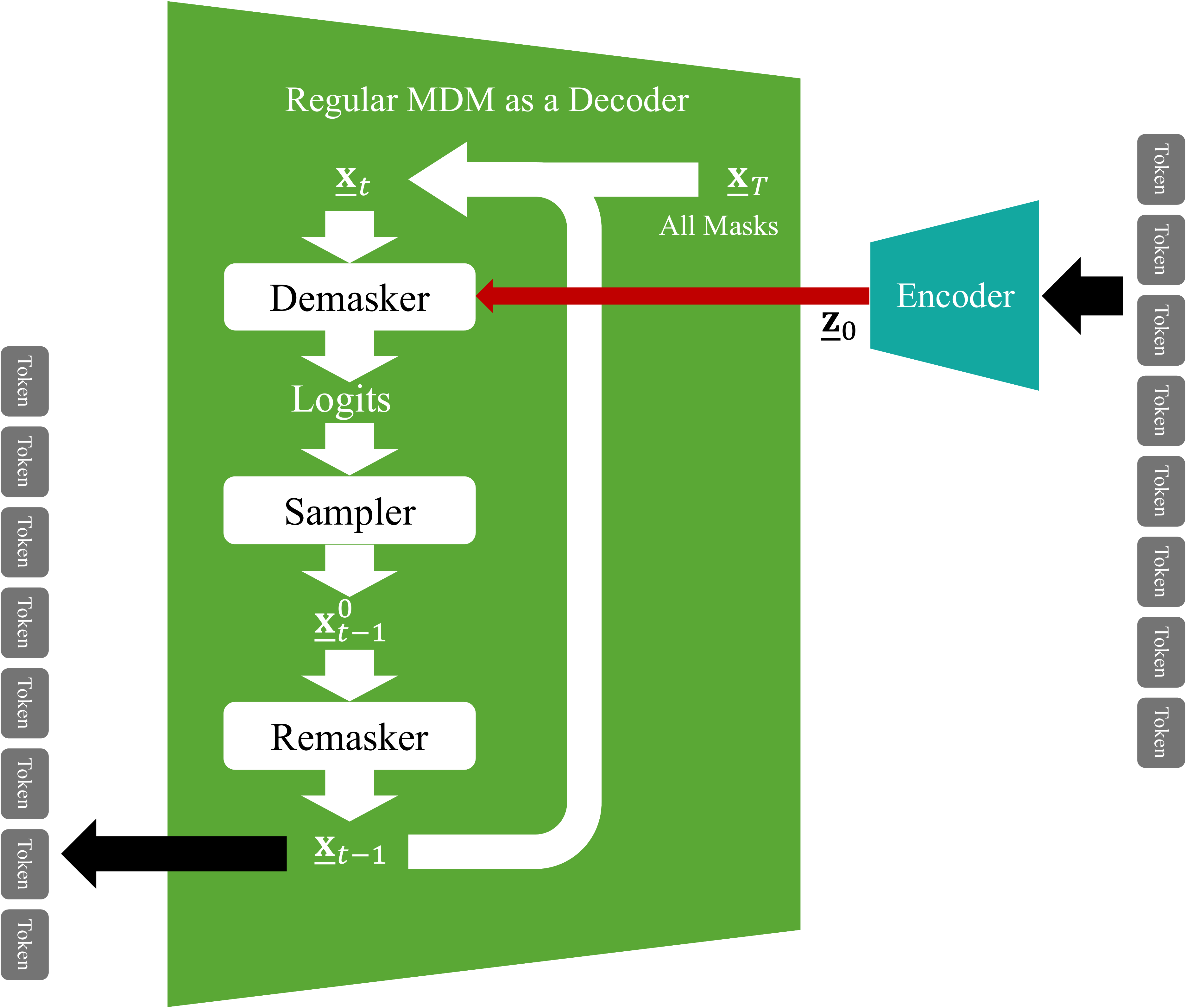}
    \caption{The proposed autoencoder: a token sequence $\rvx_0$ is encoded to a latent representation $\rvz_0$, and its decoding is done via a regular MDM.}
    \label{fig:AutoEncoder}
\end{figure}

\begin{algorithm}[htb]
  \caption{Discrete-Continuum-Discrete Autoencoder}
  \label{alg:auto-encoding}
  \begin{algorithmic}[1]
    \STATE {\bfseries Input:} Clean sequence $\rvx_0 \sim p_{\text{data}}$, encoder $h_\phi$, demasker $f_\theta$, number of steps $T$
    \STATE Encode to continuous latent space: $\rvz_0 \leftarrow h_\phi(\rvx_0)$
    \STATE Initialize from fully masked vector: $t \leftarrow 1$, $\rvx_t \leftarrow \rvm$
    \WHILE{$t > 0$}
      \STATE Predict clean sample: $\hat{\rvx}_0 \sim f_\theta(\rvx_t, t, \rvz_0)$ 
      \STATE Decrement time: $t \leftarrow t - 1/T$
      \STATE Apply masking: $\rvx_{t} \leftarrow \rev{\operatorname{Mask}}(\hat{\rvx}_0, t)$
    \ENDWHILE
    \STATE \textbf{return} Reconstructed sequence $\hat{\rvx}_0$
  \end{algorithmic}
\end{algorithm}

\begin{figure}[t]
    \centering
    \includegraphics[width=0.99\linewidth]{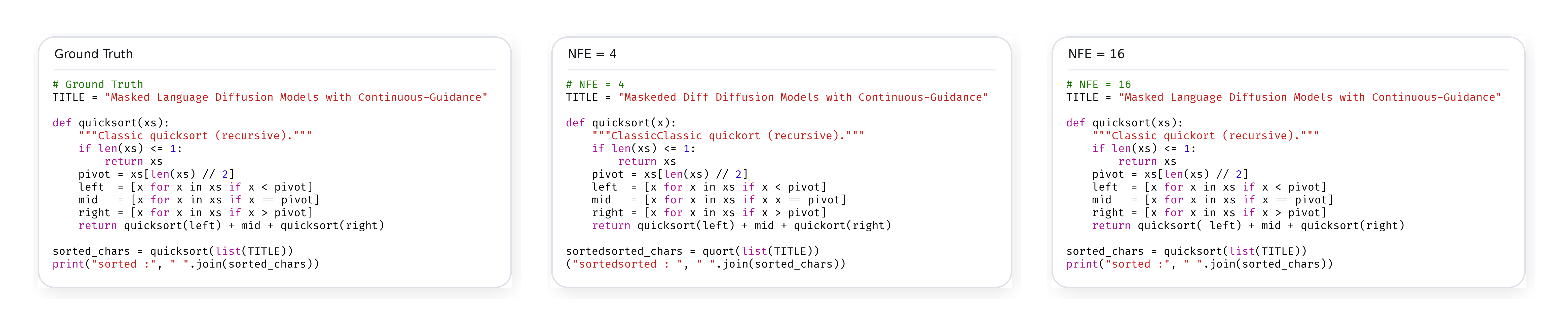}
    \caption{\textbf{Autoencoder:} A source (left) Python program and its encoded-decoded outcome (Gen-Length=$256$ tokens as one block) with the guided LLaDA-8B demasker. The two outcomes correspond to NFE=$4$ (CER=$0.12$) in the middle and NFE=$16$ (CER=$0.03$) at the right.} %
    \label{fig:ProgramIn-ProgramOut-More}
\end{figure}

\cref{fig:ProgramIn-ProgramOut-More} shows an example source program and its encoded-decoded result, which is nearly perfectly reconstructed even when the MDM is applied with a few NFEs.

\paragraph{LLaDA autoencoder, length 256.}
\cref{tab:gen_metrics256} reports representative autoencoder results for sequences of length $256$ tokens using a latent of size $128\times 1024$, while varying the MDM decoder hyper-parameters.

\begin{table}[h]
\centering
    \caption{\textbf{LLaDA-8B-based Autoencoder:} Performance measured via Generative Perplexity, and recovery error evaluated via CodeBERTScore and Character Error Rate (CER). The table explores varying hyper-parameters of the MDM decoder, for a generative length of $256$ tokens. }
    \label{tab:gen_metrics256}
    \small
\setlength{\tabcolsep}{6pt}
\begin{tabular}{rrrrr}
\toprule
\textbf{Block} & \textbf{NFE} & \textbf{Gen-PPL} & \textbf{CodeBERTScore} & \textbf{CER} \\ \midrule
32  & 8   & 59.538 & 0.901 & 0.422 \\
32  & 16  & 20.263 & 0.936 & 0.210 \\
32  & 32  & 13.525 & 0.957 & 0.150 \\
32  & 64  & 11.301 & 0.968 & 0.130 \\
32  & 128 & 10.401 & 0.970 & 0.123 \\
32  & 256 & 10.085 & 0.972 & 0.118 \\ \midrule
64  & 8   & 28.172 & 0.925 & 0.265 \\
64  & 16  & 15.553 & 0.951 & 0.171 \\
64  & 32  & 12.463 & 0.963 & 0.140 \\
64  & 64  & 10.971 & 0.970 & 0.124 \\ \midrule
128 & 8   & 18.027 & 0.946 & 0.190 \\
128 & 16  & 13.244 & 0.960 & 0.149 \\
128 & 32  & 11.566 & 0.968 & 0.130 \\
128 & 64  & 10.639 & 0.971 & 0.122 \\ \midrule
256 & 4   & 19.221 & 0.939 & 0.205 \\
256 & 8   & 13.937 & 0.957 & 0.155 \\
256 & 16  & 11.981 & 0.965 & 0.132 \\
256 & 32  & 10.962 & 0.970 & 0.123 \\
256 & 64  & 10.458 & 0.973 & 0.118 \\ \bottomrule
\end{tabular}
\end{table}

\paragraph{LLaDA autoencoder, length 512.}
\cref{tab:gen_metrics512} is analogous to~\cref{tab:gen_metrics256}, but refers to sequences of length $512$ tokens. The conclusions drawn from both tables are similar: the autoencoder obtains near-perfect reconstructions even for low NFE and larger block sizes.

\begin{table}[h]
\centering
    \caption{\textbf{LLaDA-8B-based Autoencoder:} Performance measured via Generative Perplexity, and recovery error evaluated via CodeBERTScore and Character Error Rate (CER). The table explores varying hyper-parameters of the MDM decoder, for a generative length of $512$ tokens.}
    \label{tab:gen_metrics512}
    \small
\setlength{\tabcolsep}{6pt}
\begin{tabular}{rrrrr}
\toprule
\textbf{Block} & \textbf{NFE} & \textbf{Gen-PPL} & \textbf{CodeBERTScore} & \textbf{CER} \\ \midrule
32 & 16  & 35.477 & 0.833 & 1.149 \\
32 & 32  & 15.554 & 0.877 & 0.237 \\
32 & 64  & 9.630  & 0.904 & 0.146 \\
32 & 128 & 8.170  & 0.959 & 0.118 \\
32 & 256 & 7.583  & 0.960 & 0.105 \\ \midrule
64 & 8   & 33.738 & 0.764 & 1.100 \\
64 & 16  & 29.314 & 0.821 & 1.196 \\
64 & 32  & 11.834 & 0.935 & 0.183 \\
64 & 64  & 8.989  & 0.948 & 0.132 \\
64 & 128 & 7.989  & 0.955 & 0.113 \\
64 & 256 & 7.534  & 0.959 & 0.103 \\ \midrule
128 & 8   & 35.547 & 0.856 & 1.180 \\
128 & 16  & 18.715 & 0.945 & 1.229 \\
128 & 32  & 10.806 & 0.953 & 1.218 \\
128 & 64  & 7.497  & 0.957 & 1.152 \\ \midrule
256 & 8   & 22.527 & 0.935 & 1.240 \\
256 & 16  & 11.828 & 0.948 & 0.218 \\
256 & 32  & 9.277  & 0.955 & 0.176 \\
256 & 64  & 8.178  & 0.959 & 0.158 \\
256 & 128 & 7.646  & 0.960 & 0.149 \\ \midrule
512 & 8   & 13.146 & 0.877 & 1.337 \\
512 & 16  & 9.820  & 0.904 & 0.189 \\
512 & 32  & 8.447  & 0.959 & 0.163 \\
512 & 64  & 7.832  & 0.960 & 0.152 \\
512 & 128 & 7.517  & 0.957 & 0.145 \\ \bottomrule
\end{tabular}
\end{table}

\paragraph{MDLM autoencoder.}
We additionally evaluate the autoencoder built around the MDLM backbone, which is trained on OpenWebText with the configuration of~\cref{tab:training-mdlm}. The protocol matches the LLaDA tables above (same block-size / NFE sweep, same Gen-PPL and CER metrics), except that semantic similarity is measured by BERTScore rather than CodeBERTScore, since the MDLM decoder is trained on natural-language text rather than Python code. \cref{tab:gen_metrics_mdlm_tbd} reports the corresponding sweep for sequences of length $512$ tokens.

\begin{table}[h]
\centering
    \caption{\textbf{MDLM-400M-based Autoencoder:} Performance measured via Generative Perplexity, and recovery error evaluated via BERTScore and Character Error Rate (CER). The table explores varying hyper-parameters of the MDM decoder, for a generative length of $512$ tokens.}
    \label{tab:gen_metrics_mdlm_tbd}
    \small
\setlength{\tabcolsep}{6pt}
\begin{tabular}{rrrrr}
\toprule
\textbf{Block} & \textbf{NFE} & \textbf{Gen-PPL} & \textbf{BERTScore} & \textbf{CER} \\ \midrule
32  & 16  & 16.45 & 0.9696 & 0.0400 \\
32  & 32  & 15.74 & 0.9805 & 0.0287 \\
32  & 64  & 15.28 & 0.9845 & 0.0242 \\
32  & 128 & 14.96 & 0.9867 & 0.0224 \\
32  & 256 & 14.79 & 0.9873 & 0.0217 \\
32  & 512 & 14.68 & 0.9879 & 0.0214 \\ \midrule
64  & 16  & 15.93 & 0.9794 & 0.0296 \\
64  & 32  & 15.42 & 0.9840 & 0.0249 \\
64  & 64  & 15.10 & 0.9860 & 0.0228 \\
64  & 128 & 14.88 & 0.9870 & 0.0218 \\
64  & 256 & 14.75 & 0.9876 & 0.0213 \\
64  & 512 & 14.66 & 0.9879 & 0.0212 \\ \midrule
128 & 16  & 15.61 & 0.9832 & 0.0258 \\
128 & 32  & 15.20 & 0.9857 & 0.0230 \\
128 & 64  & 14.96 & 0.9867 & 0.0218 \\
128 & 128 & 14.79 & 0.9875 & 0.0213 \\
128 & 256 & 14.70 & 0.9878 & 0.0210 \\
128 & 512 & 14.65 & 0.9880 & 0.0210 \\ \midrule
256 & 16  & 15.31 & 0.9852 & 0.0237 \\
256 & 32  & 15.06 & 0.9866 & 0.0222 \\
256 & 64  & 14.83 & 0.9874 & 0.0213 \\
256 & 128 & 14.72 & 0.9878 & 0.0211 \\
256 & 256 & 14.66 & 0.9879 & 0.0209 \\
256 & 512 & 14.62 & 0.9881 & 0.0208 \\ \midrule
512 & 16  & 15.08 & 0.9863 & 0.0223 \\
512 & 32  & 14.92 & 0.9871 & 0.0216 \\
512 & 64  & 14.78 & 0.9875 & 0.0211 \\
512 & 128 & 14.69 & 0.9879 & 0.0209 \\
512 & 256 & 14.63 & 0.9881 & 0.0208 \\
512 & 512 & 14.60 & 0.9881 & 0.0208 \\ \bottomrule
\end{tabular}
\end{table}

\section{More on \emph{ConThenDisc} and \emph{ConWithinDisc}}
\label{app: More on ConThenDisc and ConWithinDisc}

\cref{fig:Generators} presents a block-diagram of the \emph{ConThenDisc} and \emph{ConWithinDisc} algorithms, highlighting their difference -- the additional conditioning of the continuous diffusion on the embedding of an intermediate generated text, $\rvx_t$, which enables an update of the guidance $\rvz_0$ to the MDM.

\begin{figure}[t]
    \centering
    \includegraphics[width=0.62\linewidth]{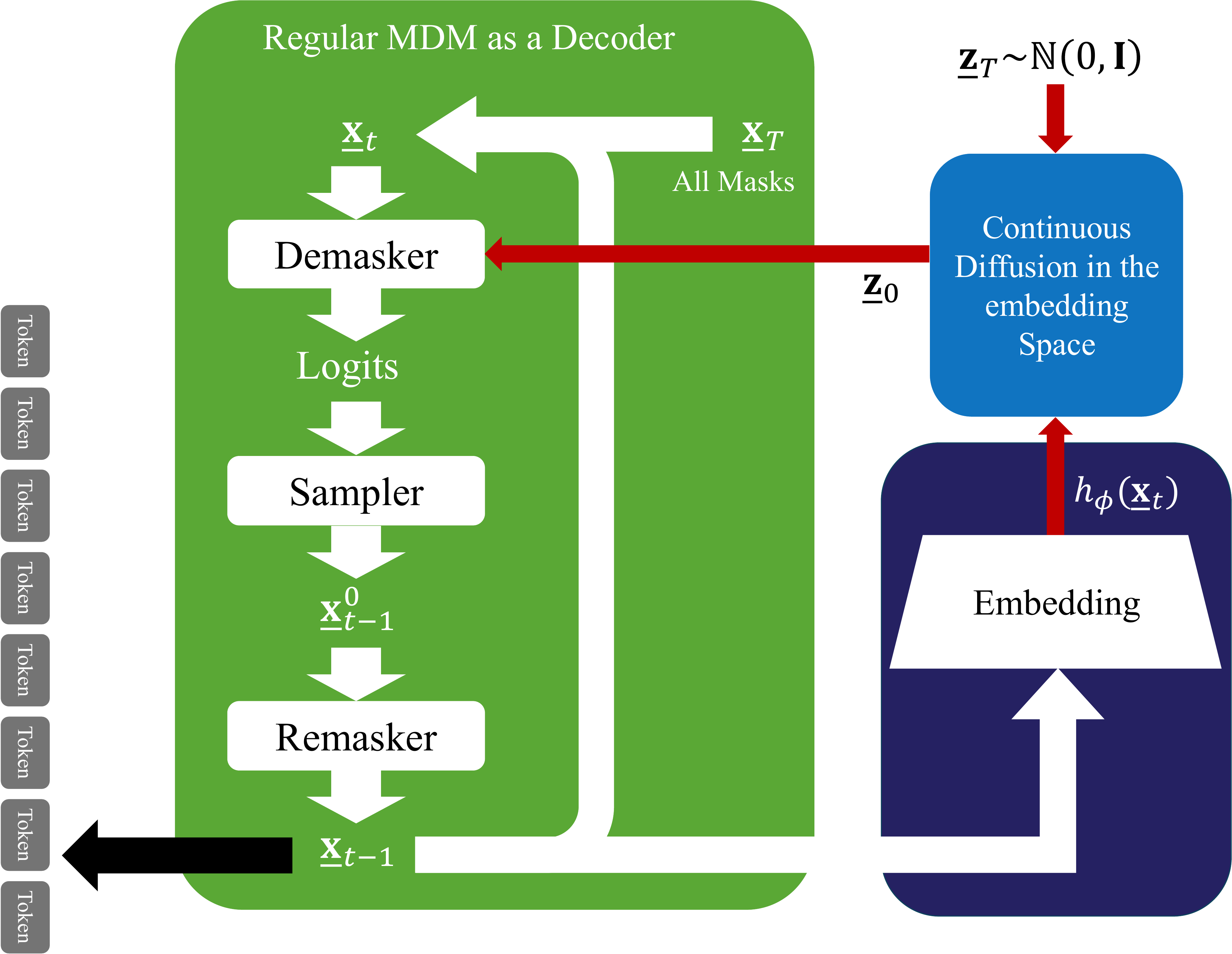}
    \caption{The \emph{ConThenDisc} and \emph{ConWithinDisc} text generation algorithms. In both, a continuous diffusion model generates a starting latent $\rvz_0$, which is decoded to tokens by a regular MDM. \emph{ConWithinDisc} includes a refinement of the latent (the purple part) based on the partially synthesized text.}
    \label{fig:Generators}
\end{figure}

\cref{fig:ablations_512} summarizes three ablations that characterize the interaction between our continuous-register conditioning and the discrete LLaDA demasking stage, all referring to sequence length of $512$ tokens of Python programs. More specifically, 

(i) \textbf{Varying the number of registers (ConThenDisc baseline).}
We first vary the number of registers used to represent the continuous conditioning under \emph{ConThenDisc}, taking $k{=}8$ as the default setting. Increasing the register budget degrades the performance: $k{=}16$ remains close to the baseline, while larger settings ($k{=}32$ and $k{=}64$) consistently degrade generation quality.
A plausible explanation is that learning additional \emph{fine} registers is harder and does not necessarily benefit the downstream demasking process. Once early registers capture the dominant global signal, later registers are expected to encode residual refinements; however, these higher-index registers are more weakly constrained, can become redundant or noisy, and may fail to provide complementary guidance. In this regime, adding registers can dilute the conditioning or introduce instability rather than improving the final sample quality.

We emphasize that the role of $k$ is task-dependent. In the autoencoder of~\cref{sec:auto-encoder-gen}, a larger register bank improves reconstruction capacity. In unconditional generation, however, increasing $k$ can make the latent distribution harder to model and may shift too much fine detail into the sketch, rather than leaving it to the iterative MDM decoder. This explains why smaller $k$ performs better in~\cref{fig:ablations_512} (left).

(ii) \textbf{When to refresh the conditioning (ConWithinDisc).}
Next, we ablate the timing of refreshing the continuous embedding during discrete demasking by updating after a fraction $\alpha \in \{0,0.25,0.50,0.75\}$ of the LLaDA steps. We include \texttt{@0} as a no-refresh reference corresponding to \emph{ConThenDisc} (plotted with a small horizontal shift for readability). The results indicate that the midpoint update (\texttt{@0.50}) is the most effective operating point. Updating too early (\texttt{@0.25}) behaves closer to \texttt{@0}: at the beginning of demasking, the discrete trajectory has not yet accumulated sufficient structure for a refreshed embedding to provide meaningfully different guidance. Updating too late (\texttt{@0.75}) offers limited benefit because fewer decisions remain; by that stage many tokens are already committed, and committed tokens remain unchanged during LLaDA demasking, reducing the degrees of freedom through which refreshed conditioning can influence the final generation. Overall, the midpoint refresh provides the best trade-off between having enough structure to exploit and still having sufficient flexibility to steer the remaining undecided tokens.

(iii) \textbf{Block-wise timestep offsets.}
Finally, we evaluate \emph{block-wise timestep offsets} in the continuous diffusion stage, introduced to promote a clearer hierarchical organization across registers and to better align training with the progressive ordering induced by nested dropout. Empirically, these offsets are important for stable training and for strengthening the intended coarse-to-fine relationship among registers, while their effect at inference is moderate (approximately $0.03$ MAUVE in our ablations).

\begin{figure}[t]
    \centering
    \begin{subfigure}[t]{0.49\linewidth}
        \centering
        \includegraphics[width=\linewidth]{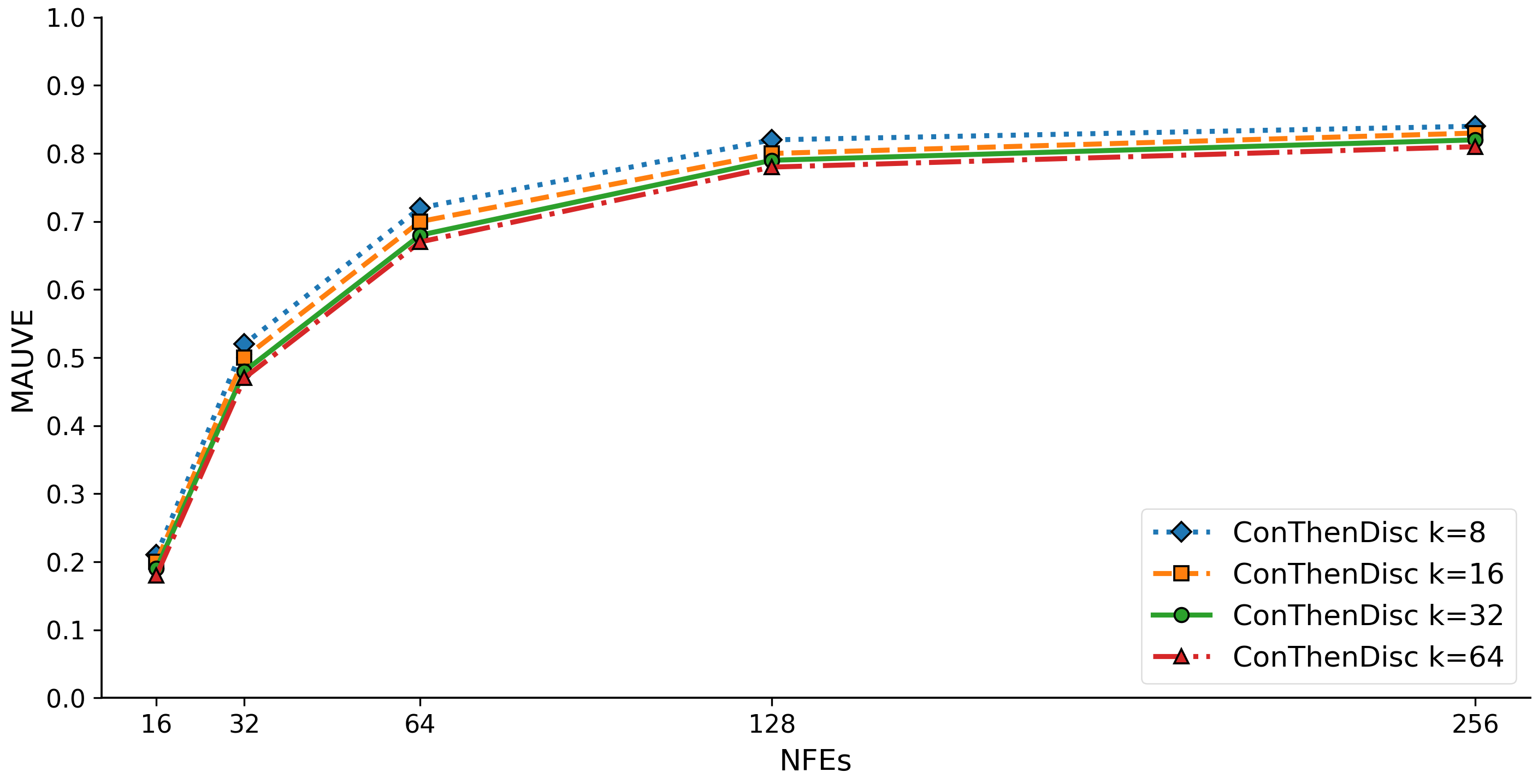}
        \caption{Register-count ablation under \emph{ConThenDisc}.}
    \end{subfigure}
    \hfill
    \begin{subfigure}[t]{0.49\linewidth}
        \centering
        \includegraphics[width=\linewidth]{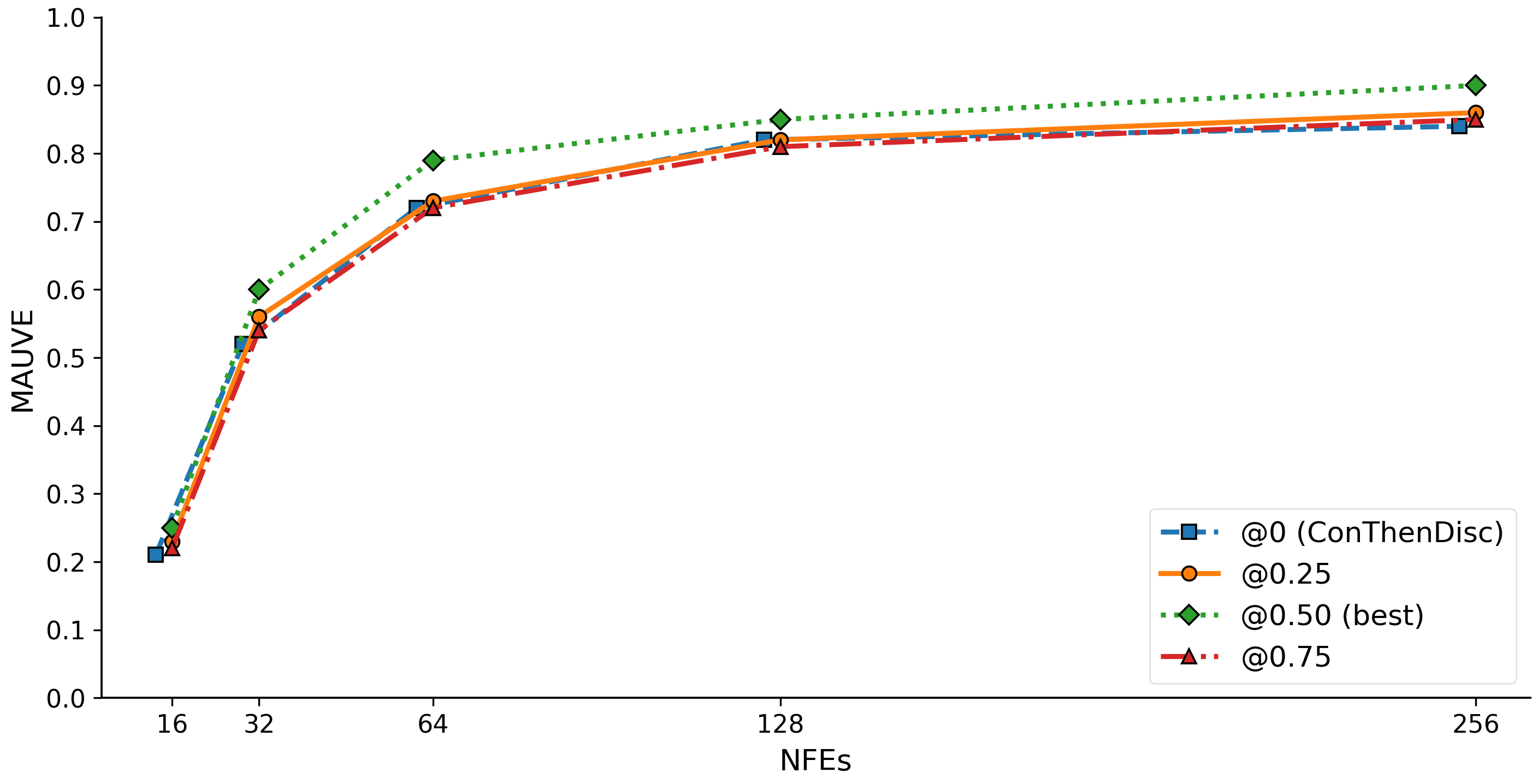}
        \caption{Conditioning refresh ablation.}
    \end{subfigure}
    \caption[Ablations at length 512.]{\textbf{Ablations at length 512.} \textbf{Left:} register-count ablation under \emph{ConThenDisc} (varying $K$). The role of $K$ is task-dependent: for the autoencoder, larger $K$ helps because it provides more capacity to reconstruct the input; for unconditional generation, larger $K$ does not always help and can even hurt, as the continuous diffusion is then forced to model a higher-dimensional latent and the sketch carries finer-grained information that the iterative MDM can recover from local context. Smaller $K$ keeps the sketch global and easier for the MDM to commit to, which is consistent with the trends shown here. \textbf{Right:} conditioning refresh ablation, updating after a fraction $\alpha$ of LLaDA steps. The midpoint update (\texttt{@0.50}) is most effective; updating too early resembles the no-refresh baseline (\texttt{@0}, i.e., \emph{ConThenDisc}), while updating too late leaves limited room to affect generation as more tokens are already committed and remain unchanged during demasking.}
    \label{fig:ablations_512}
\end{figure}

\rev{\cref{fig:Con256Disc} complements~\cref{fig:ConXDisc} with generated sequences of length $512$ tokens. At NFE$=40$, \emph{ConWithinDisc} reaches MAUVE $\approx 0.60$ and Gen-PPL $14.3$, compared with MAUVE $\approx 0.61$ and Gen-PPL $23.5$ for Base LLaDA at NFE$=512$. Thus, \emph{ConWithinDisc} matches Base LLaDA's MAUVE while improving Gen-PPL with a $\times 13$ reduction in demasker-equivalent NFE.}

\begin{figure}[t]
    \centering
    \includegraphics[width=0.7\linewidth]{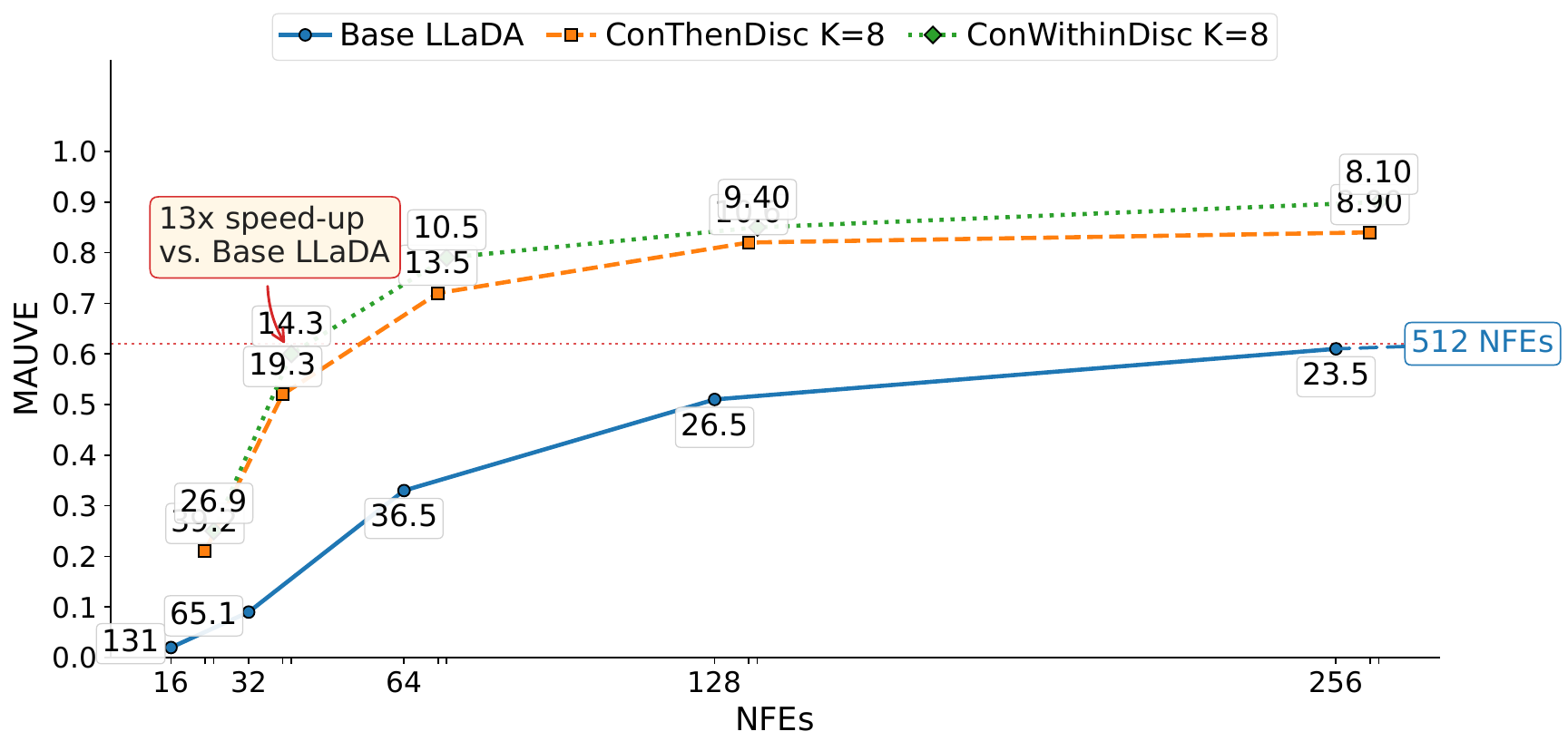}
    \caption{\textbf{Text Generation:} MAUVE and Generative Perplexity of generated text with Base LLaDA, \emph{ConThenDisc} and \emph{ConWithinDisc}. The generated sequence has length $512$ and the number of registers is $K=8$. \rev{CRoCoDiL decodes the sequence as one block, whereas Base LLaDA uses $32$-token blocks.}}
    \label{fig:Con256Disc}
\end{figure}

\clearpage
\begin{revblock}
\section{\texorpdfstring{\rev{Additional generation evaluation}}{Additional generation evaluation}}
\label{app:additional-evaluation}

This section presents additional evaluations of code quality, end-to-end speed, and natural-language generation at the 8B scale. Unless stated otherwise, CRoCoDiL uses $128$ continuous (DDIM) steps followed by the indicated number of discrete demasking steps. Each result is computed from $1{,}000$ unconditional generations. Base LLaDA uses one decoder step per token unless noted otherwise.

\paragraph{Sampling protocol.} CRoCoDiL decodes each sequence as a single block using temperature $0$ and low-confidence remasking. Base LLaDA and Base MDLM use block-wise decoding at temperature $0.7$; Base LLaDA uses $32$-token blocks and low-confidence remasking. We apply neither nucleus nor top-$k$ filtering.

\subsection{\texorpdfstring{\rev{Code-specific evaluation}}{Code-specific evaluation}}
\label{app:code-evaluation}

In addition to MAUVE and GPT-2-based Gen-PPL, we directly evaluate the syntactic and structural validity of the generated Python. We report three complementary validity levels tailored to unconditional generation, for which no prompt, target function, or hidden tests are defined. Fixed-length generation may end inside the final statement; none of these checks repairs or adds code:
\begin{itemize}
    \setlength{\itemsep}{0.25em}
    \setlength{\parsep}{0pt}
    \setlength{\topsep}{0.4em}
    \item \textbf{Level 1, whole file:} the entire output must compile with Python's \texttt{compile()}, which requires successful parsing and bytecode compilation.
    \item \textbf{Level 2, complete functions:} at least one complete function must compile exactly as generated; we separately test whether it loads successfully when executed in a sandbox.
    \item \textbf{Level 3, code before the cut (boundary-aware):} \rev{we retain the longest part that already parses, composed of complete top-level units, and test whether it compiles, similar to the stop-sequence truncation used in HumanEval-style evaluation~\citep{chen2021codex}.} To exclude empty or degenerate outputs, the retained part must contain a class, function, or assignment.
\end{itemize}

\paragraph{Level 1: whole-file compilation.} \cref{tab:code-wholefile} reports whole-file compile rates at length $256$.

\begin{table}[H]
\centering
\caption{Whole-file compile rates at length $256$, with both models using the same number of decoder steps (tokens per step in parentheses).}
\label{tab:code-wholefile}
\begin{tabular}{rc@{\hspace{2.5em}}c}
\toprule
Decoder steps & CRoCoDiL & Base LLaDA \\
\midrule
64 (4) & \textbf{23.7\%} & 7.7\% \\
128 (2) & \textbf{25.1\%} & 20.8\% \\
256 (1) & 96.9\% & \textbf{97.6\%} \\
\bottomrule
\end{tabular}
\end{table}

With one token per step, both models finish their last statement before the length limit and compile about $97\%$ of the time. With several tokens per step, the end of the file is written in parallel with the rest, and the last statement is often cut off, as in these two CRoCoDiL outputs ($64$ steps):
\begin{verbatim}
for i in range(0, len(p)):
    print(np.float(2) + nl[i]) * p[i
\end{verbatim}
\begin{verbatim}
self.box = wx.lib.masked.MaskedBox(
    self.canvas, x=2, y=2, width=window_width, height=window_height)
if self.box.SetPen("center"):
\end{verbatim}
A single unfinished line rejects the whole file, even when everything before it is valid: at $64$ steps, $23.7\%$ of CRoCoDiL's outputs compile as whole files, but $55.7\%$ compile once the unfinished end is dropped (Level 3). Because whole-file compilation mostly reflects this unfinished ending, we report the full range of lengths and step budgets with Levels 2 and 3.

Additional measurements at length $256$, with both CRoCoDiL and Base LLaDA using $128$ decoder steps, also favor CRoCoDiL: structural Vendi score $29.19$ vs.\ $24.65$; unique compiling programs $25.1\%$ vs.\ $20.8\%$; files without lint warnings $9.1\%$ vs.\ $7.1\%$; files without redefinitions $21.7\%$ vs.\ $16.6\%$; and severely repetitive outputs (more than half of all 4-grams repeated) $2.7\%$ vs.\ $31.7\%$.

\paragraph{Levels 2 and 3.} \cref{tab:code-levels} reports both levels at every length and step budget.

\begin{table}[H]
\centering
\caption{Complete functions and code before the cut (CRoCoDiL / Base LLaDA), at the same number of decoder steps (tokens per step in parentheses).}
\label{tab:code-levels}
\small
\begin{tabular}{rrccc}
\toprule
Length & Decoder steps & Level 2: function compiles & Level 2: loads in sandbox & Level 3: code before the cut \\
\midrule
256 & 64 (4) & \textbf{22.4\%} / 7.9\% & \textbf{19.7\%} / 7.5\% & \textbf{55.7\%} / 23.8\% \\
256 & 128 (2) & \textbf{28.6\%} / 20.0\% & \textbf{26.2\%} / 17.6\% & \textbf{69.8\%} / 53.6\% \\
256 & 256 (1) & \textbf{37.7\%} / 29.6\% & \textbf{33.6\%} / 25.8\% & \textbf{81.3\%} / 79.3\% \\
\midrule
512 & 64 (8) & \textbf{12.9\%} / 0.1\% & \textbf{12.0\%} / 0.1\% & \textbf{40.4\%} / 4.2\% \\
512 & 128 (4) & \textbf{23.1\%} / 7.8\% & \textbf{21.4\%} / 7.4\% & \textbf{56.3\%} / 18.7\% \\
512 & 256 (2) & \textbf{36.0\%} / 21.1\% & \textbf{32.5\%} / 20.7\% & \textbf{72.0\%} / 45.6\% \\
\midrule
1024 & 64 (16) & \textbf{16.4\%} / 0.0\% & \textbf{15.6\%} / 0.0\% & \textbf{35.4\%} / 0.3\% \\
1024 & 128 (8) & \textbf{17.9\%} / 0.2\% & \textbf{17.1\%} / 0.2\% & \textbf{36.9\%} / 1.6\% \\
1024 & 256 (4) & \textbf{29.9\%} / 5.1\% & \textbf{27.7\%} / 5.0\% & \textbf{54.5\%} / 11.6\% \\
\bottomrule
\end{tabular}
\end{table}

\begin{itemize}
    \setlength{\itemsep}{0.25em}
    \setlength{\parsep}{0pt}
    \setlength{\topsep}{0.4em}
    \item \textbf{CRoCoDiL leads in all $27$ comparisons, and the gap grows with length.} At length $1024$ with $256$ steps, it produces complete compiling functions $5.9\times$ more often than Base LLaDA ($29.9\%$ vs.\ $5.1\%$), functions that load in a sandbox $5.5\times$ more often ($27.7\%$ vs.\ $5.0\%$), and $4.7\times$ more code that compiles before the cut ($54.5\%$ vs.\ $11.6\%$).
    \item \textbf{Against Base LLaDA's full budget.} With $4\times$ fewer decoder steps, CRoCoDiL approaches full-budget Base LLaDA on code that compiles before the cut at length $512$ ($56.3\%$ vs.\ $60.1\%$) and slightly exceeds it at length $1024$ ($54.5\%$ vs.\ $51.1\%$). Even with $16\times$ fewer steps at length $1024$, it reaches about $70\%$ of Base LLaDA's full-budget rate ($35.4\%$ vs.\ $51.1\%$).
    \item \textbf{Per second, the gap is larger, even at the shortest length.} At length $256$, where CRoCoDiL's measured speedup is smallest (\cref{tab:wallclock}), CRoCoDiL at $64$ steps produces $2.93\times$ more compiling functions per second, and $2.72\times$ more samples whose code before the cut compiles, than Base LLaDA at its full $256$ steps.
\end{itemize}
At one token per step, Level 3 is lower than Level 1 ($81.3\%$ vs.\ $96.9\%$) because it also requires a class, a function, or an assignment.

\subsection{\texorpdfstring{\rev{End-to-end latency and throughput}}{End-to-end latency and throughput}}
\label{app:wallclock}

We time CRoCoDiL at the operating point selected in the paper ($128$ continuous steps followed by $64$ discrete steps) and compare it with the standard full-budget Base LLaDA, which uses as many decoder steps as tokens.

\begin{table}[H]
\centering
\caption{End-to-end speed on an H100-80GB GPU, $300$ timed generations per setting. Decoder steps are given with tokens per step in parentheses.}
\label{tab:wallclock}
\small
\begin{tabular}{rcccc}
\toprule
Length & Decoder steps & Batch-1 & Outputs/s at largest & Throughput \\
 & (CRoCoDiL vs.\ Base LLaDA) & speedup & common batch & speedup \\
\midrule
256 & 64 (4) vs.\ 256 (1) & $\mathbf{1.21\times}$ & B32: 1.455 / 0.376 & $\mathbf{3.87\times}$ \\
512 & 64 (8) vs.\ 512 (1) & $\mathbf{3.45\times}$ & B16: 0.764 / 0.092 & $\mathbf{8.30\times}$ \\
1024 & 64 (16) vs.\ 1024 (1) & $\mathbf{9.48\times}$ & B8: 0.371 / 0.0226 & $\mathbf{16.39\times}$ \\
\bottomrule
\end{tabular}
\end{table}

At length $1024$ and batch size $1$, a sample takes $4.9$\,s with CRoCoDiL and $46.8$\,s with Base LLaDA. The speedup grows with length because the cost of the small continuous model stays low, while the number of decoder steps of Base LLaDA grows with the sequence. The timed region covers everything from sampling the registers to the decoded output (continuous model, decoder, and conditioning), so kernel-launch and memory-bandwidth overheads are included; only model loading and warmup are excluded. Both models use bf16, Flash-SDPA, PyTorch~2.8, and CUDA~12.9 on the same hardware. The largest common batch is the largest batch both models completed in our sweep: B32 was the top of the sweep at length $256$, and at lengths $512$ and $1024$ Base LLaDA ran out of memory at the next size.

\subsection{\texorpdfstring{\rev{Training-free LLaDA samplers}}{Training-free LLaDA samplers}}
\label{app:samplers}

We apply three training-free samplers to Base LLaDA without retraining: Entropy-Bounded (EB) Unmasking~\citep{benhamu2025eb}, paper-exact Position Contrastive Guidance (PCG)~\citep{patel2025pcg}, and Fast-dLLM v1~\citep{wu2025fastv1}. We compare them with unaccelerated CRoCoDiL at the operating point of~\cref{tab:wallclock}. All runs use sequence length $1024$ and batch size $4$.

\begin{table}[H]
\centering
\caption{End-to-end throughput at sequence length $1024$ and batch size $4$. CRoCoDiL uses $128$ continuous steps followed by $64$ discrete steps; Base LLaDA uses the sampler indicated in each row. All comparisons are measured at the same batch size.}
\label{tab:samplers}
\begin{tabular}{lcc}
\toprule
Base LLaDA sampler & Speedup over vanilla Base LLaDA & Remaining CRoCoDiL advantage \\
\midrule
Vanilla ($1024$ steps) & $1.00\times$ & $\mathbf{14.85\times}$ \\
Fast-dLLM v1 & $2.63\times$ & $\mathbf{5.65\times}$ \\
PCG, paper-exact & $2.89\times$ & $\mathbf{5.14\times}$ \\
EB Unmasking ($\epsilon=0.1$) & $1.34\times$ & $\mathbf{11.11\times}$ \\
\bottomrule
\end{tabular}
\end{table}

At a common batch size of $4$, the training-free samplers accelerate Base LLaDA by up to $2.89\times$, while unaccelerated CRoCoDiL remains $5.14\times$ to $11.11\times$ faster. On $1{,}000$ samples per setting, Fast-dLLM preserves Base LLaDA's quality (whole-file compilation goes from $95.7\%$ to $93.9\%$, and StarCoder2-3B perplexity is unchanged), whereas PCG reduces quality, mainly by making outputs more repetitive.

\subsection{\texorpdfstring{\rev{Natural-language generation at the 8B scale}}{Natural-language generation at the 8B scale}}
\label{app:owt-8b}

To test whether the guidance causes semantic drift at the 8B scale, we trained both CRoCoDiL-8B and Base LLaDA-8B on OpenWebText for $150$k iterations, following the COSMOS training setup, and compared them on unconditional OpenWebText generation of length $512$, the same setting as our 400M experiments.

\begin{table}[H]
\centering
\caption{Unconditional OpenWebText-512 generation at the 8B scale.}
\label{tab:owt8b}
\begin{tabular}{lcc}
\toprule
 & Base LLaDA-8B & CRoCoDiL-8B \\
\midrule
Decoder steps (tokens per step) & 512 (1) & 64 (8) \\
MAUVE $\uparrow$ & 0.68 & \textbf{0.70} \\
GPT-2-L Gen-PPL $\downarrow$ & 37.8 & \textbf{32.3} \\
\bottomrule
\end{tabular}
\end{table}

CRoCoDiL slightly exceeds Base LLaDA's MAUVE and has lower Gen-PPL while using $8\times$ fewer decoder steps, which suggests that the guidance does not cause semantic drift on natural language at this scale.

\subsection{\texorpdfstring{\rev{Parameter-weighted comparison with COSMOS}}{Parameter-weighted comparison with COSMOS}}
\label{app:cosmos-accounting}

As in~\cref{sec:uncond}, we report sampling cost in demasker-equivalent NFE. Each continuous-diffusion pass is weighted by the denoiser-to-demasker parameter ratio. Our \emph{ConThenDisc} point with DDIM~$16$, which reaches MAUVE $\approx 0.66$ and Gen-PPL $\approx 50.0$, costs $49.3$ NFE. \cref{sec:uncond} counts each of the $200$ COSMOS sampling steps as one demasker-equivalent pass (NFE $\approx 200$), giving the reported $\times 4.1$ reduction ($200/49.3$).

To show how this comparison depends on the COSMOS denoiser size, \cref{tab:cosmos-cost} instead weights each COSMOS step by the assumed denoiser-to-demasker parameter ratio. The released checkpoint has a $129.8$M-parameter denoiser, but with the published recipe it gives MAUVE $0.41$ and Gen-PPL $100.3$, rather than the reported MAUVE $0.702$ and Gen-PPL $55.0$. Because this checkpoint does not reproduce the published metrics, the denoiser size underlying the published result remains uncertain; we therefore report the released checkpoint together with $250$M and $500$M reference sizes.

\begin{table}[H]
\centering
\caption{Parameter-weighted cost of COSMOS ($200$ sampling steps) relative to CRoCoDiL with DDIM~$16$ ($49.3$ demasker-equivalent passes).}
\label{tab:cosmos-cost}
\begin{tabular}{lcc}
\toprule
COSMOS denoiser & Cost (demasker-equivalent passes) & CRoCoDiL cost advantage \\
\midrule
$129.8$M (released checkpoint) & $61.2$ & $1.24\times$ \\
$250$M & $117.9$ & $2.4\times$ \\
$500$M & $235.8$ & $4.8\times$ \\
\bottomrule
\end{tabular}
\end{table}
\end{revblock}

\section{COSMOS-style oracle decoder (methodology)}\label{app:cosmos-oracle}

COSMOS~\citep{meshchaninov2025compressed} is an unconditional text generator, developed and assessed at small scale on OpenWebText, paralleling MDLM. COSMOS starts by training an encoder that converts a text sequence to a continuum representation, followed by a single-step decoder for returning to the sequence's tokens. Text synthesis via COSMOS is obtained by first applying a continuous diffusion to generate a valid latent representation, and then converting this latent to tokens by an activation of the decoder. 

Comparisons of the original COSMOS to our guided MDLM (400M) are provided in~\cref{sec:uncond}. In order to extend these comparisons to the LLaDA-8B scale, we offer a COSMOS-style oracle experiment. Our goal in this experiment is to identify the gap between iterative MDM and single-step decoding.

Building on CRoCoDiL's guided LLaDA-8B demasker $f_\theta(\rvx_t,t,\rvz_0)$ as our decoder, and $h_\phi(\rvx_0)$ as its corresponding encoder, we retrain this pair for time step $T$ only, implying that the decoder gets a fully masked sequence, and thus needs to recover $\rvx_0$ from its latent, $\rvz_0 = h_\phi(\rvx_0)$.
In this training, the encoder is initialized from Qwen-embedding-0.6B and trained on the same StarCoder Python corpus. Its latent $\rvz_0$ has dimension $256 \times 1024$, giving the one-step decoder a generous latent budget.

Given this encoder-decoder machine, we evaluate oracle text generation by feeding $h_\phi(\cdot)$ with valid Python programs, and assess the quality of the decoded outcome. We refer to this recovery as an oracle generation experiment since it avoids the use of a continuous diffusion for synthesizing a valid embedding, posing an upper bound on the achievable text quality.

Since this COSMOS-like decoder is one-step, it cannot spend additional NFE to improve the sample. Thus, the oracle result (Gen-PPL $\approx 52$, MAUVE $\approx 0.07$) reflects the practical ceiling of this design, whereas CRoCoDiL can use additional MDM steps to improve token quality.

\section{Training configurations}
\label{app:training-configs}

\rev{Training proceeds in three stages: (1) encoder warmup followed by joint encoder-demasker training (\cref{alg:training}); (2) training the continuous diffusion on the resulting registers (\cref{app: Continuous Diffusion}); and (3) training the conditional diffusion used by \emph{ConWithinDisc} (\cref{alg:conwithindisc-training}).}

This section consolidates the training hyperparameters of the two backbones evaluated in the paper, so that the experimental sections can be read together with the exact configuration used for each track. \cref{tab:training-mdlm} and~\cref{tab:training-llada} report the encoder/guided-demasker stage (\cref{sec:guidance-learning}) and the continuous-diffusion stage (\cref{app: Continuous Diffusion}), for the MDLM ($\sim$400M, OpenWebText) and the LLaDA-8B (StarCoder/Python) tracks respectively. The two tracks share the same encoder design, the same DiT-style continuous denoiser, and the same training recipe; the controlled differences are the discrete backbone, the training corpus, and the total training budget.

\begin{table}[!h]
\centering
\caption{\textbf{MDLM setting (OpenWebText, $\sim$400M):} training configuration of the encoder/guided-demasker stage (top) and the continuous-diffusion stage (bottom).}
\label{tab:training-mdlm}
\small
\begin{tabular}{lll}
\hline
\textbf{Stage} & \textbf{Hyper-parameter} & \textbf{Value} \\ \hline
\multirow{12}{*}{\textit{Encoder + Guided Demasker}}
 & Discrete backbone           & MDLM ($\sim$400M params) \\
 & Encoder initialization      & Qwen-embedding-0.6B \\
 & Training corpus             & OpenWebText-512 \\
 & Register dimension $d$      & 1024 \\
 & Number of registers $K$     & 32 \\
 & Optimizer                   & AdamW \\
 & Learning rate               & $2.5\times 10^{-5}$ \\
 & Weight decay                & $0.1$ \\
 & Global batch size           & $512$ \\
 & Encoder-only warmup         & $4000$ steps \\
 & Training epochs             & $10$  \\
 & Register-noise target $\mathbb{E}[\cos]$ & $0.8$ \\
 & All-mask probability        & $0.25$ \\ 
 & Compute resources           & 4 nodes equipped with 8 NVIDIA H100 80GB GPUs \\ \hline
\multirow{8}{*}{\textit{Continuous Diffusion}}
 & Backbone                    & PixArt-style DiT, 28 layers, $\sim$400M params \\
 & Parameterization            & $\rvx_0$-prediction \\
 & Noise schedule              & DDPM, cosine \\
 & Inference sampler           & DDIM \\
 & Block-wise timestep offsets & enabled (geometric grouping) \\
 & Latent training set         & encoded OpenWebText latents \\
 & Optimizer / LR / batch      & AdamW / $2.5\times 10^{-5}$ / $512$ \\
 & DDPM training timesteps     & $1000$ \\ 
 & Training epochs             & $10$ \\ 
 & Compute resources           & 2 nodes equipped with 8 NVIDIA H100 80GB GPUs \\ \hline
\end{tabular}
\end{table}

\begin{table}[!h]
\centering
\caption{\textbf{LLaDA setting (StarCoder Python, 8B):} training configuration of the encoder/guided-demasker stage (top) and the continuous-diffusion stage (bottom).}
\label{tab:training-llada}
\small
\begin{tabular}{lll}
\hline
\textbf{Stage} & \textbf{Hyper-parameter} & \textbf{Value} \\ \hline
\multirow{12}{*}{\textit{Encoder + Guided Demasker}}
 & Discrete backbone           & LLaDA-8B \\
 & Encoder initialization      & Qwen-embedding-0.6B \\
 & Training corpus             & StarCoder Python ($12$M programs, $[0,4096]$ tokens) \\
 & Register dimension $d$      & $1024$ \\
 & Number of registers $K$     & $128$ \\
 & Optimizer                   & AdamW \\
 & Learning rate               & $2.5\times 10^{-5}$ \\
 & Weight decay                & $0.1$ \\
 & Global batch size           & $512$ \\
 & Encoder-only warmup         & $4000$ steps \\
 & Training epochs             & $3$ \\
 & Register-noise target $\mathbb{E}[\cos]$ & $0.8$ \\
 & All-mask probability        & $0.25$ \\ 
 & Compute resources           & 16 nodes equipped with 8 NVIDIA H100 80GB GPUs \\ \hline
\multirow{8}{*}{\textit{Continuous Diffusion}}
 & Backbone                    & PixArt-style DiT, 28 layers, $\sim$400M params \\
 & Parameterization            & $\rvx_0$-prediction \\
 & Noise schedule              & DDPM, cosine \\
 & Inference sampler           & DDIM \\
 & Block-wise timestep offsets & enabled (geometric grouping) \\
 & Latent training set         & $2$M encoded StarCoder sequences ($128\times 1024$) \\
 & Optimizer / LR / batch      & AdamW / $2.5\times 10^{-5}$ / $512$ \\
 & DDPM training timesteps     & $1000$ \\
 & Training epochs             & $10$ \\ 
 & Compute resources           & 2 nodes equipped with 8 NVIDIA H100 80GB GPUs \\ \hline
\end{tabular}
\end{table}

\section{Qualitative code-generation examples}
\label{app:qualitative-examples}

We present qualitative results, showing Python programs generated by our guided LLaDA-8B \emph{ConThenDisc} pipeline at two target lengths and varying numbers of discrete demasking steps.
\cref{fig:len256_disc256_12},~\cref{fig:len256_disc128_12},~\cref{fig:len256_disc64_12},~\cref{fig:len256_disc32_12} focus on programs of length $256$ tokens, while~\cref{fig:len512_disc256_12},~\cref{fig:len512_disc128_12},~\cref{fig:len512_disc64_12},~\cref{fig:len512_disc32_12} show $512$-length
programs. For all experiments, we first sample continuous registers using DDIM for 128 steps. We then decode the resulting continuous registers into code using the LLaDA demasking decoder with a single block (full-length block size), conditioning the decoder on 8 continuous registers. We vary the number of \emph{discrete} demasking steps and compare 256, 128, 64, and 32 discrete denoising steps for each target length (256 and 512 tokens), reporting two samples per setting. Each figure caption indicates the configuration so that quality differences across NFE budgets can be inspected directly.

\begin{figure}[!h]\centering
  \begin{subfigure}[t]{0.49\textwidth}\centering
    \includegraphics[width=\linewidth,height=0.42\textheight,keepaspectratio]{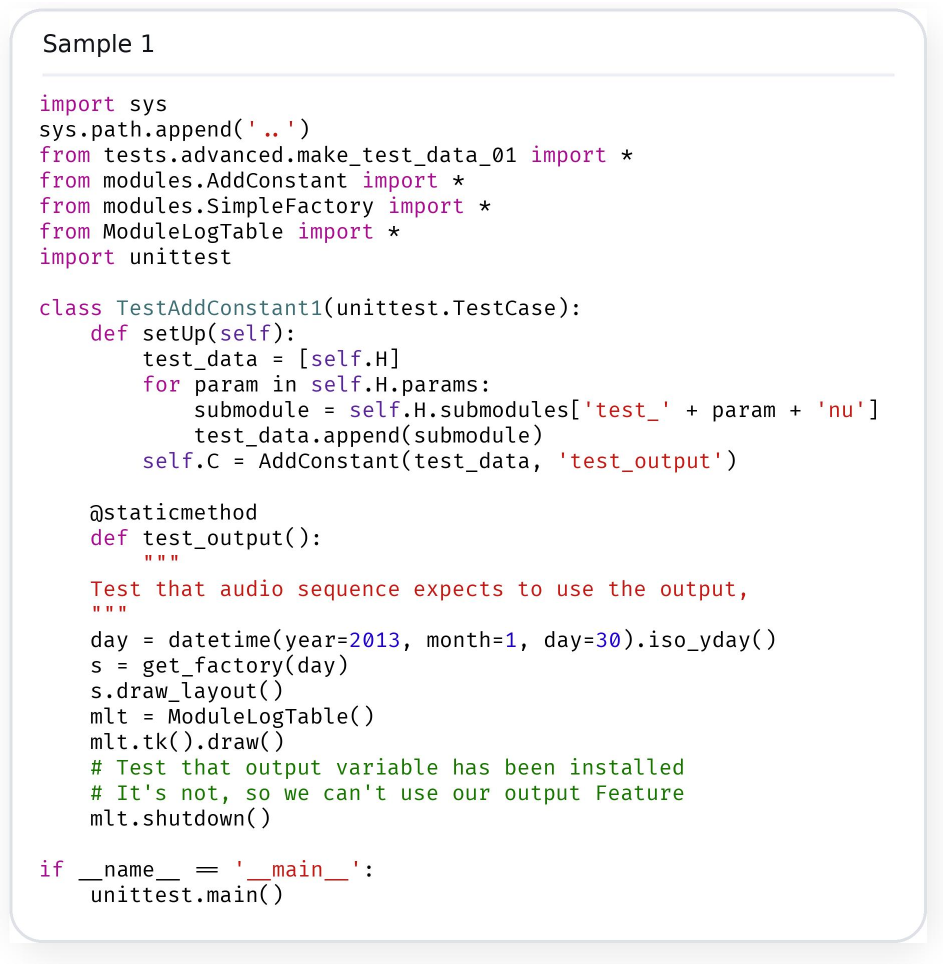}%
    \caption{Sample 1}\label{fig:len256_disc256_s1}
  \end{subfigure}\hfill
  \begin{subfigure}[t]{0.49\textwidth}\centering
    \includegraphics[width=\linewidth,height=0.42\textheight,keepaspectratio]{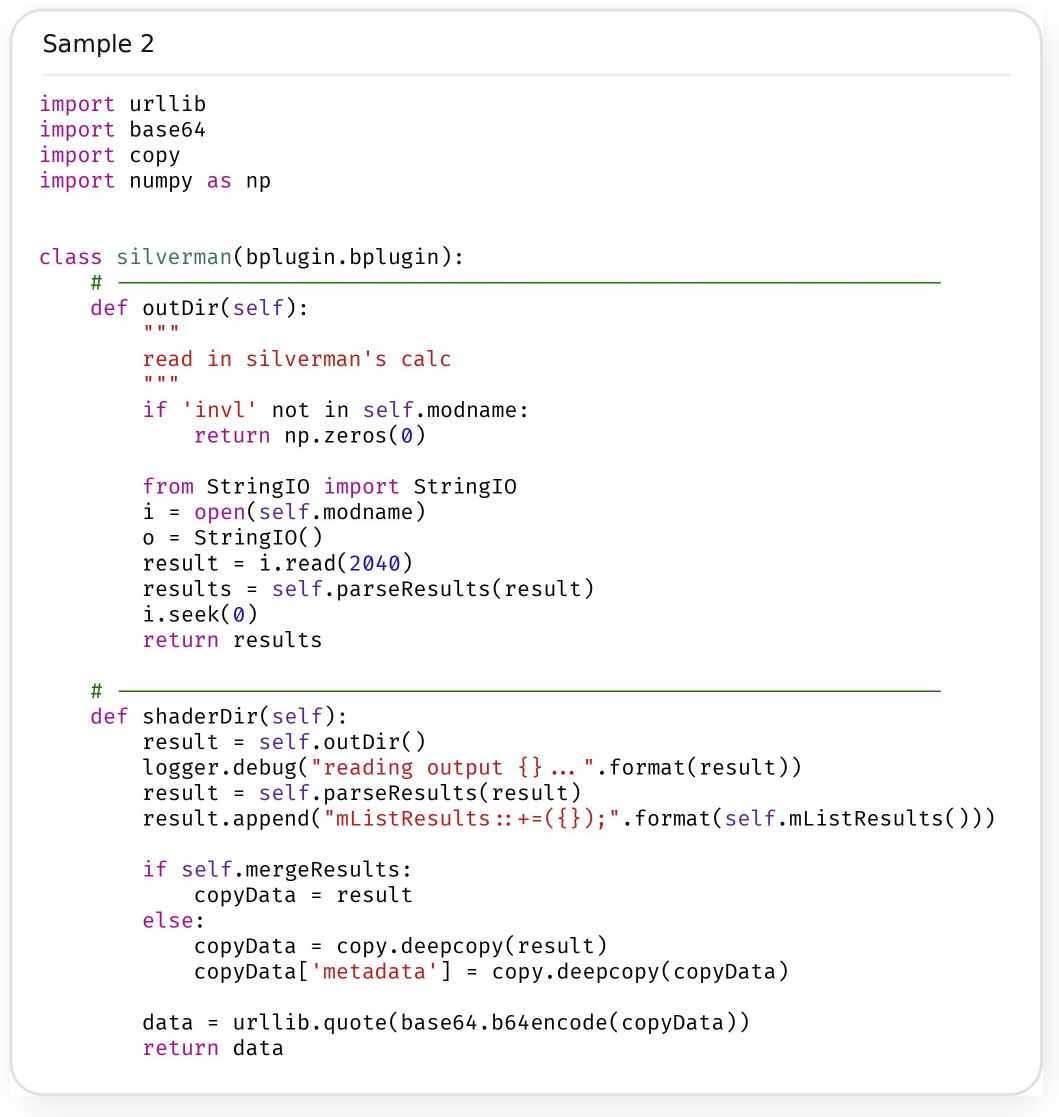}%
    \caption{Sample 2}\label{fig:len256_disc256_s2}
  \end{subfigure}
  \caption{\textbf{Qualitative code generations (length 256), samples 1--2.}
  Continuous registers are sampled with \emph{ConThenDisc} (128 DDIM steps) and decoded with LLaDA (256 discrete denoising steps, single block).}
  \label{fig:len256_disc256_12}
\end{figure}

\begin{figure}[t]\centering
  \begin{subfigure}[t]{0.49\textwidth}\centering
    \includegraphics[width=\linewidth,height=0.42\textheight,keepaspectratio]{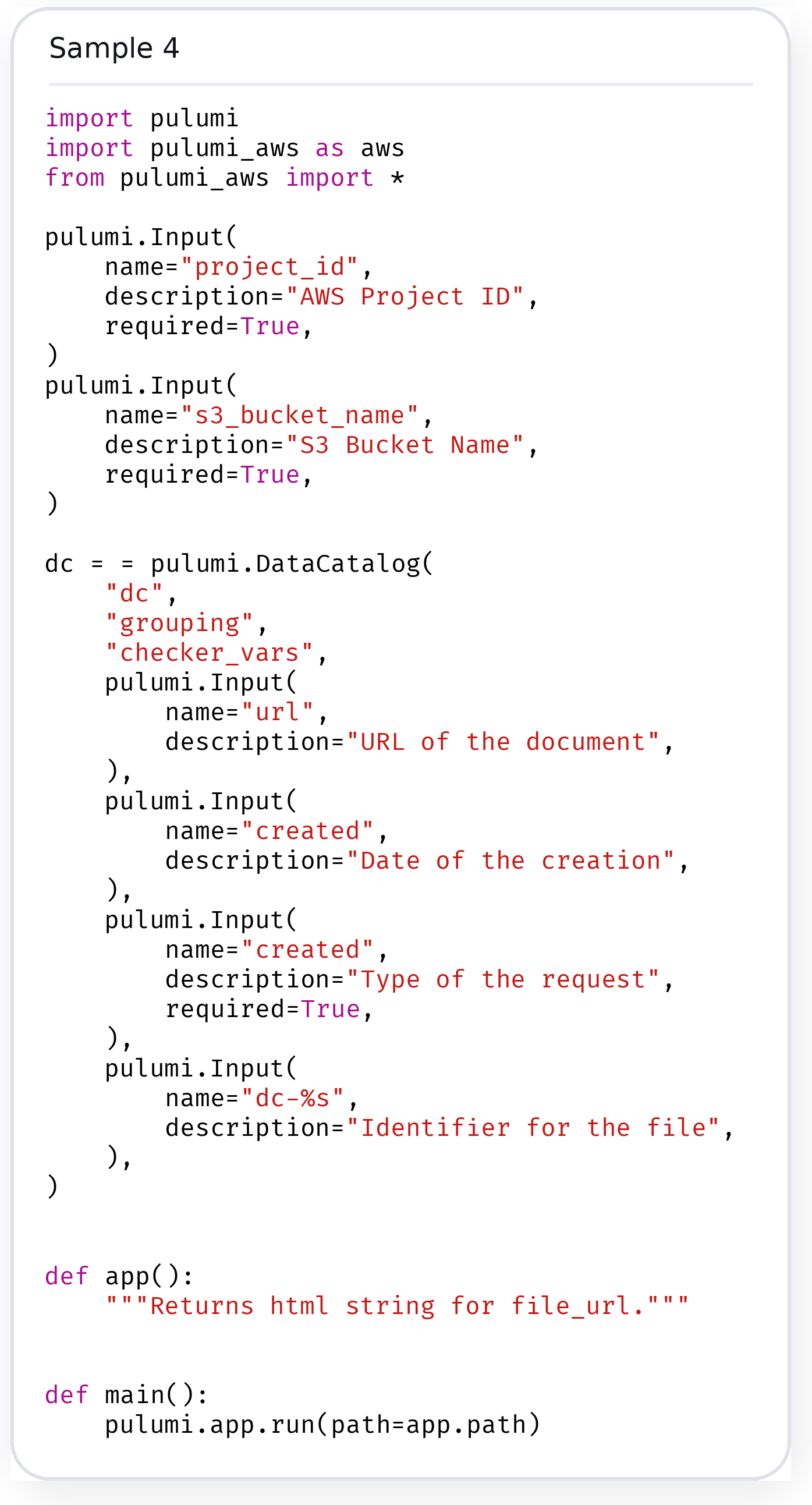}%
    \caption{Sample 1}\label{fig:len256_disc128_s1}
  \end{subfigure}\hfill
  \begin{subfigure}[t]{0.49\textwidth}\centering
    \includegraphics[width=\linewidth,height=0.42\textheight,keepaspectratio]{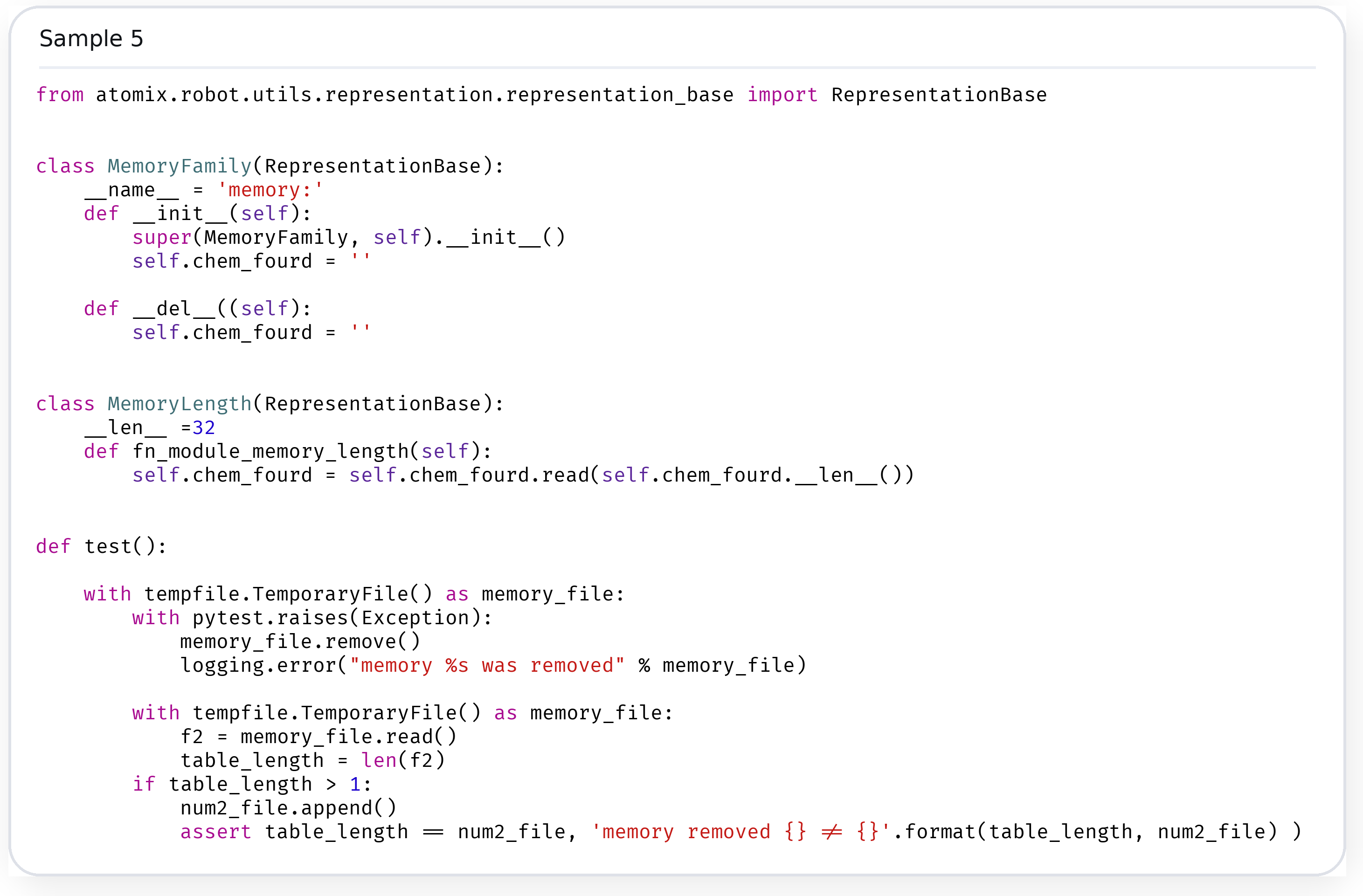}%
    \caption{Sample 2}\label{fig:len256_disc128_s2}
  \end{subfigure}
  \caption{\textbf{Qualitative code generations (length 256), samples 1--2.}
  Continuous registers are sampled with \emph{ConThenDisc} (128 DDIM steps) and decoded with LLaDA (128 discrete denoising steps, single block).}
  \label{fig:len256_disc128_12}
\end{figure}

\begin{figure}[t]\centering
  \begin{subfigure}[t]{0.49\textwidth}\centering
    \includegraphics[width=\linewidth,height=0.42\textheight,keepaspectratio]{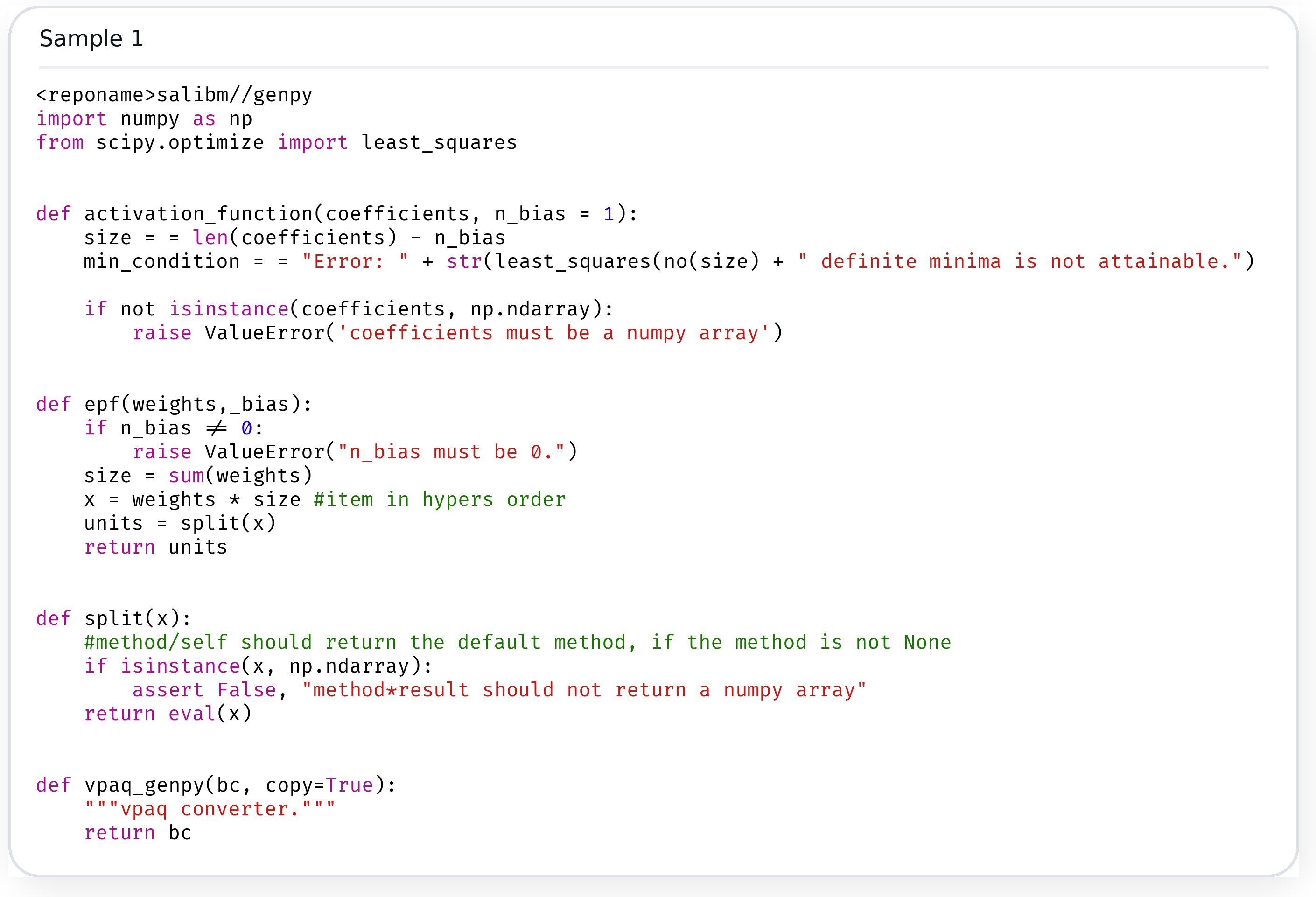}%
    \caption{Sample 1}\label{fig:len256_disc64_s1}
  \end{subfigure}\hfill
  \begin{subfigure}[t]{0.49\textwidth}\centering
    \includegraphics[width=\linewidth,height=0.42\textheight,keepaspectratio]{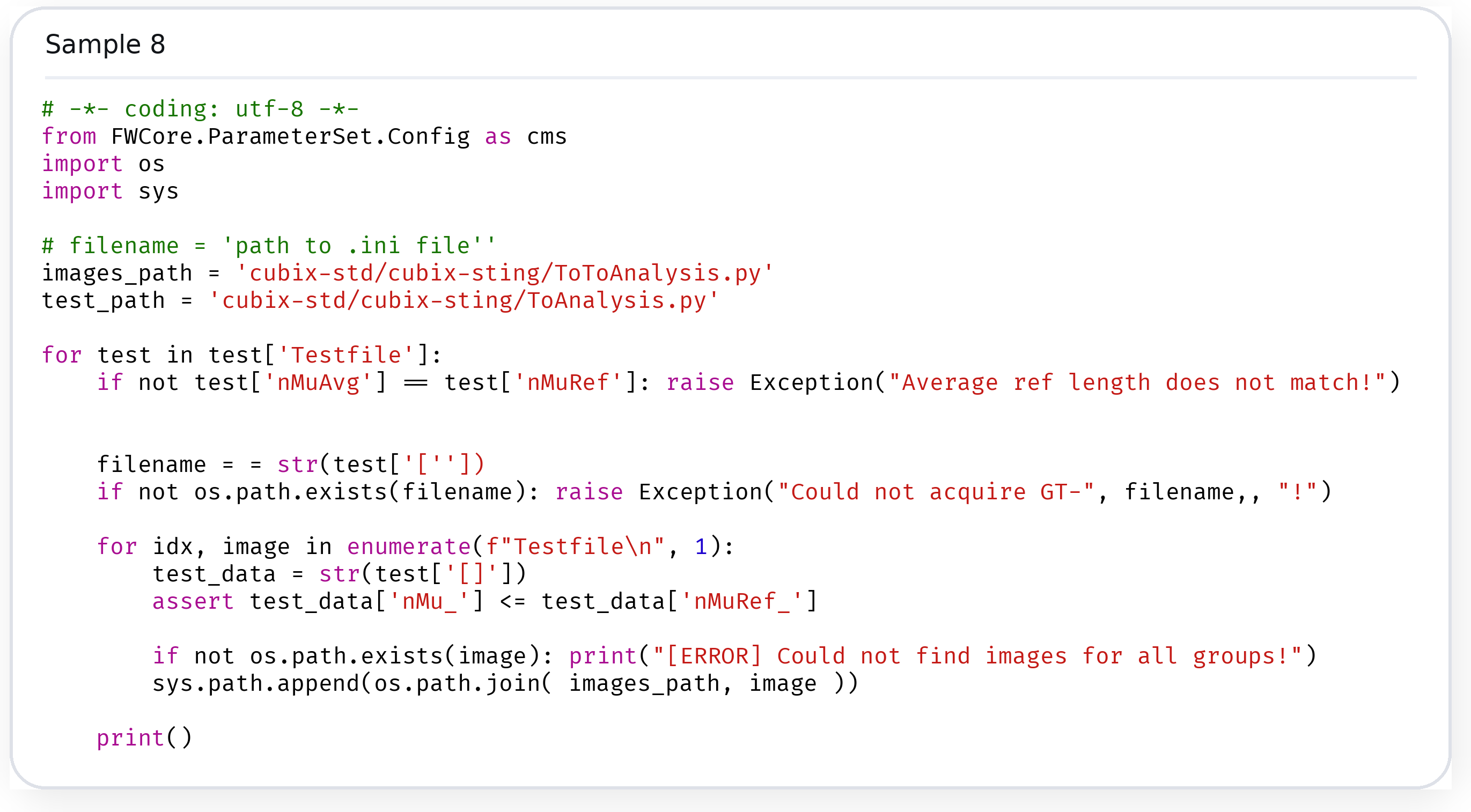}%
    \caption{Sample 2}\label{fig:len256_disc64_s2}
  \end{subfigure}
  \caption{\textbf{Qualitative code generations (length 256), samples 1--2.}
  Continuous registers are sampled with \emph{ConThenDisc} (128 DDIM steps) and decoded with LLaDA (64 discrete denoising steps, single block).}
  \label{fig:len256_disc64_12}
\end{figure}

\begin{figure}[t]\centering
  \begin{subfigure}[t]{0.49\textwidth}\centering
    \includegraphics[width=\linewidth,height=0.42\textheight,keepaspectratio]{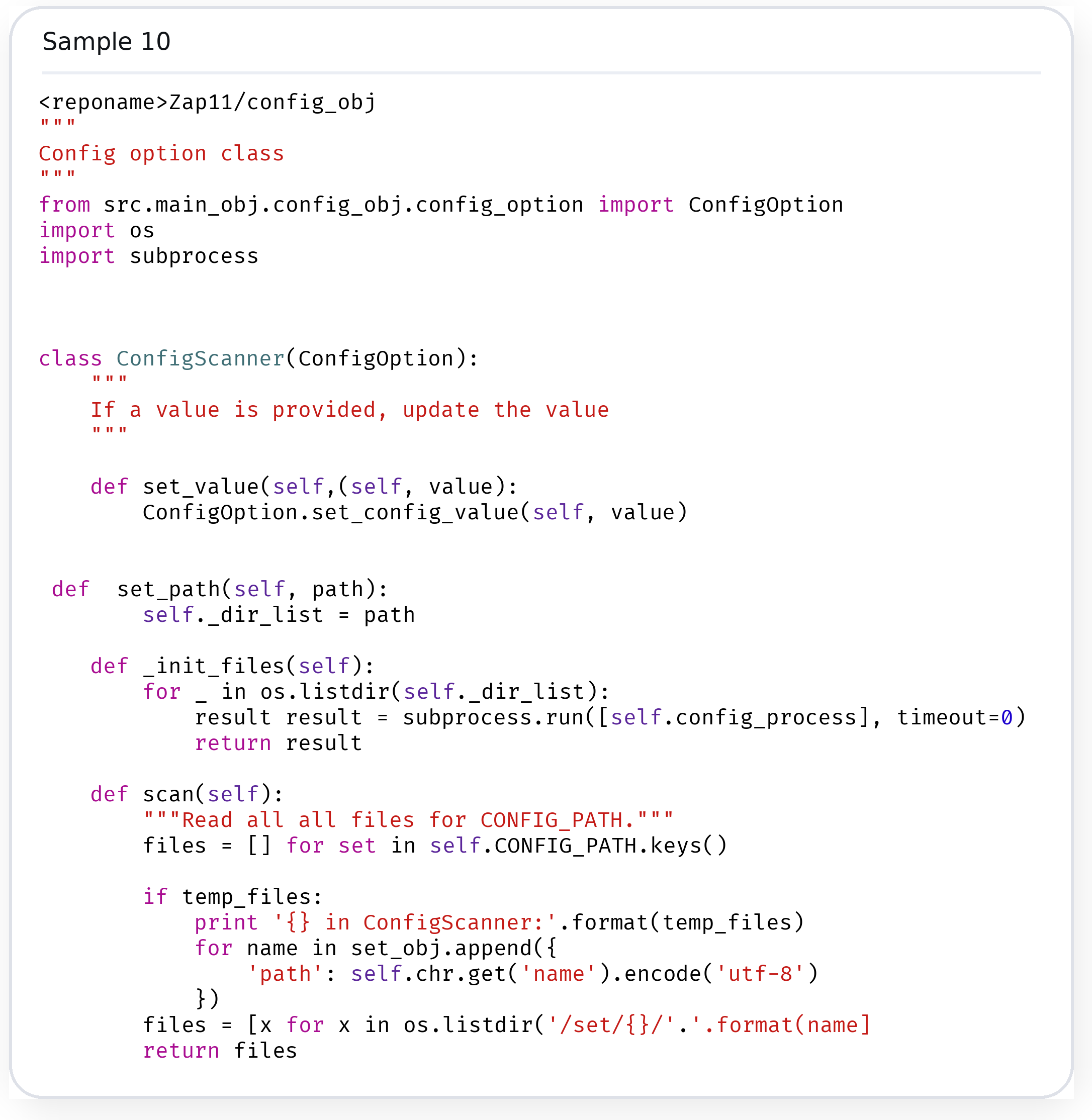}%
    \caption{Sample 1}\label{fig:len256_disc32_s1}
  \end{subfigure}\hfill
  \begin{subfigure}[t]{0.49\textwidth}\centering
    \includegraphics[width=\linewidth,height=0.42\textheight,keepaspectratio]{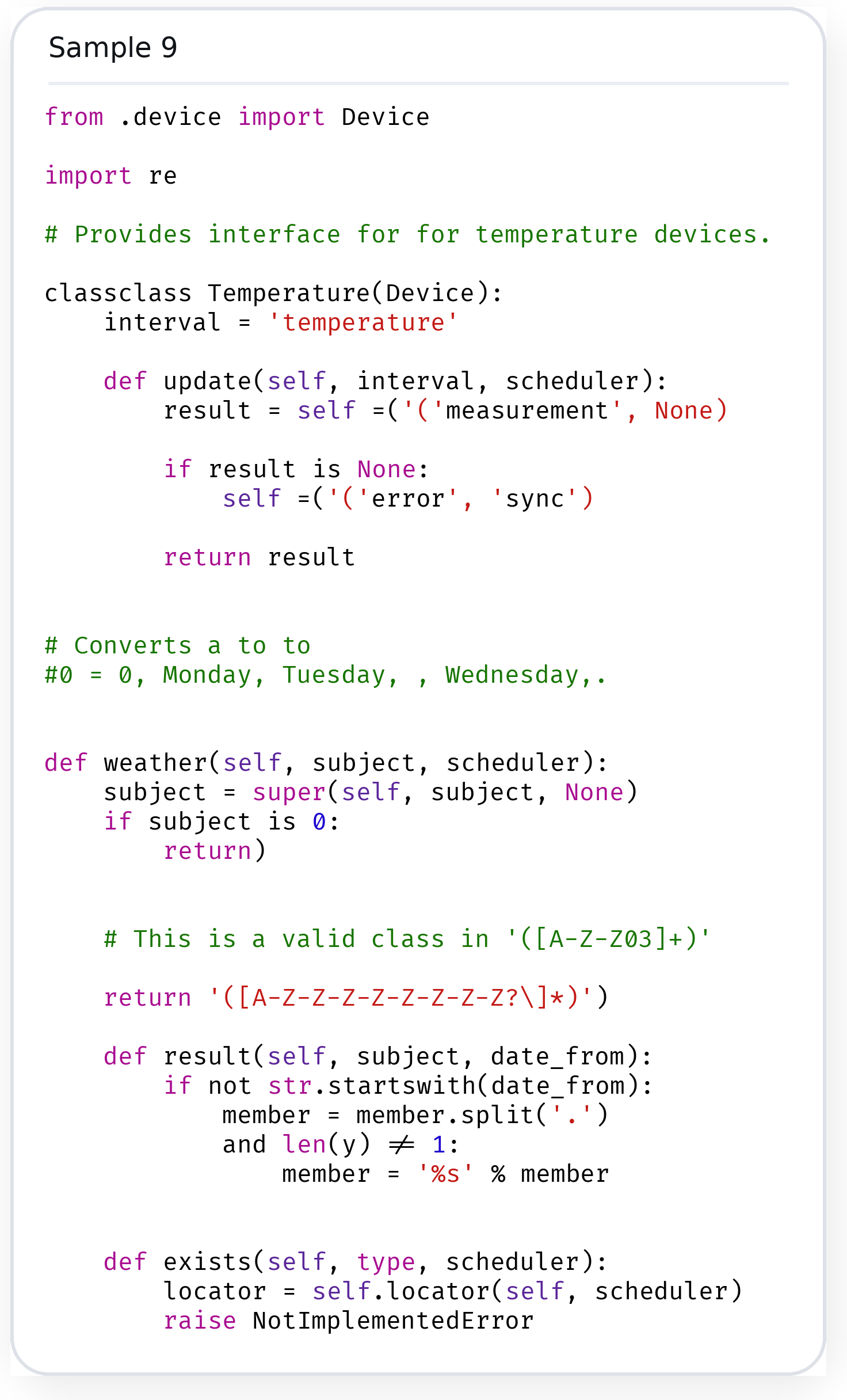}%
    \caption{Sample 2}\label{fig:len256_disc32_s2}
  \end{subfigure}
  \caption{\textbf{Qualitative code generations (length 256), samples 1--2.}
  Continuous registers are sampled with \emph{ConThenDisc} (128 DDIM steps) and decoded with LLaDA (32 discrete denoising steps, single block).}
  \label{fig:len256_disc32_12}
\end{figure}

\begin{figure}[t]\centering
    \begin{subfigure}[t]{0.49\textwidth}\centering
        \includegraphics[width=\linewidth,height=0.42\textheight,keepaspectratio]{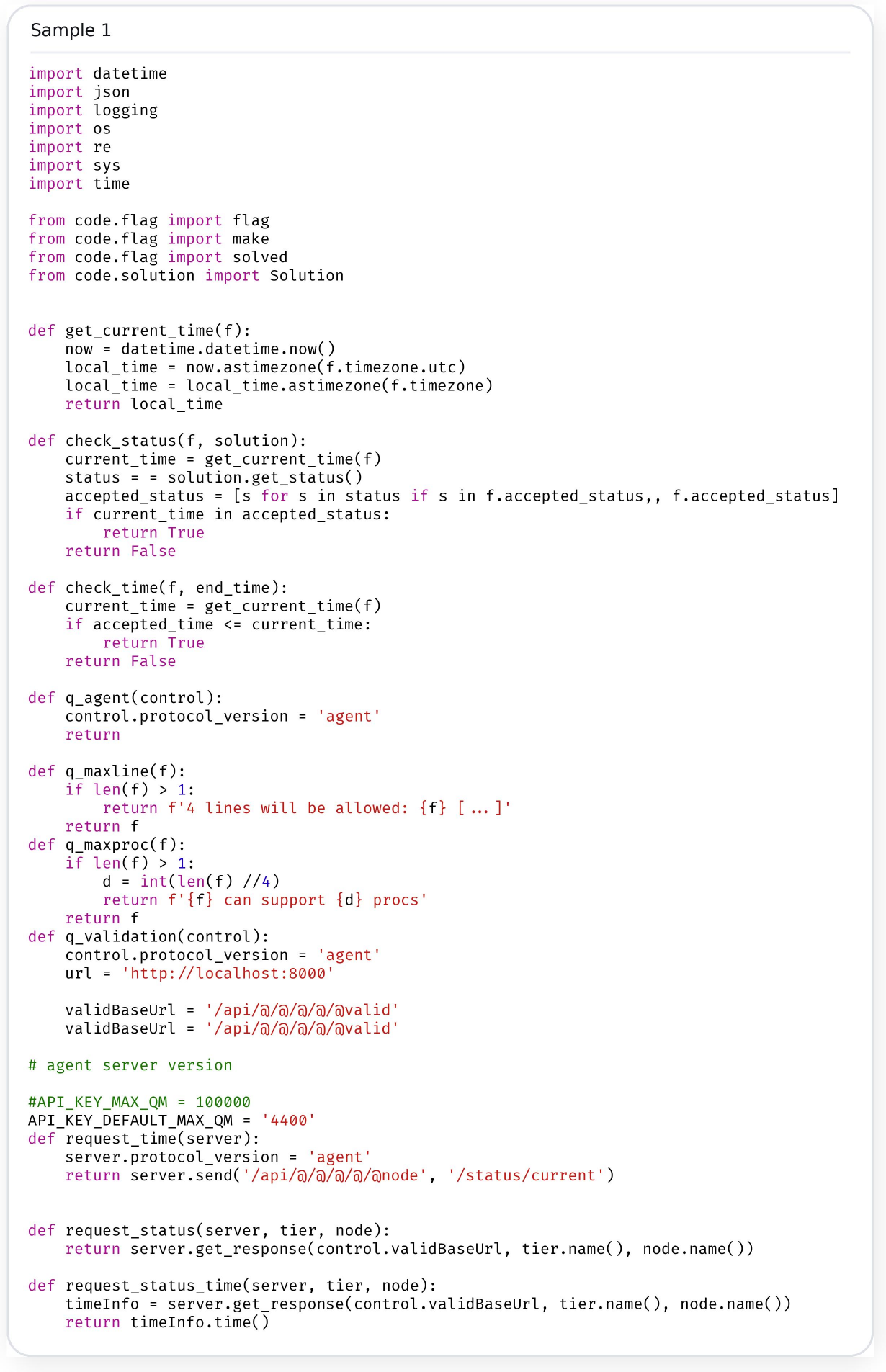}%
        \caption{Sample 1}\label{fig:len512_disc256_s1}
    \end{subfigure}\hfill
    \begin{subfigure}[t]{0.49\textwidth}\centering
        \includegraphics[width=\linewidth,height=0.42\textheight,keepaspectratio]{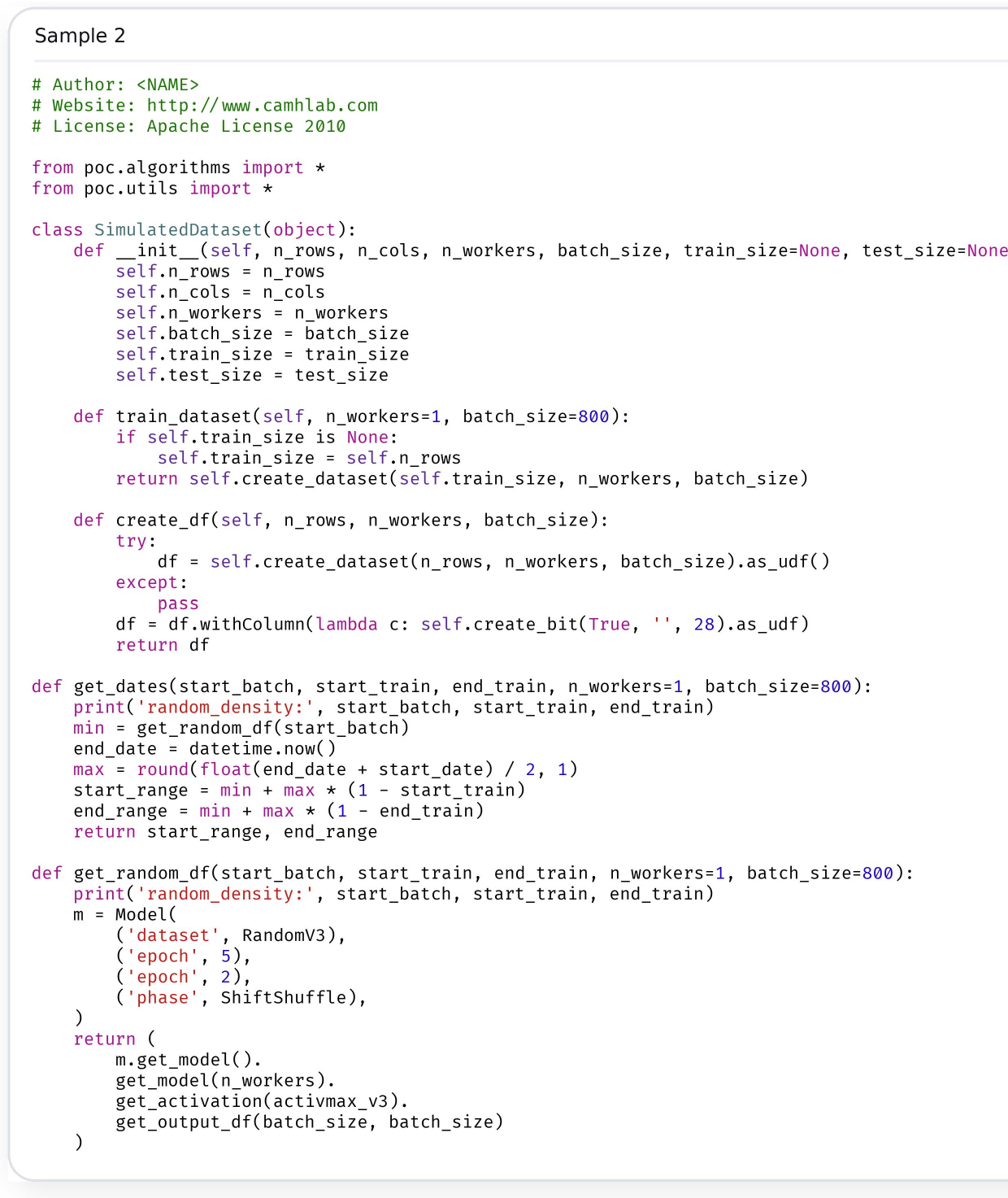}%
        \caption{Sample 2}\label{fig:len512_disc256_s2}
    \end{subfigure}
    \caption{\textbf{Qualitative code generations (length 512), samples 1--2.}
    Continuous registers are sampled with \emph{ConThenDisc} (128 DDIM steps) and decoded with LLaDA (256 discrete denoising steps, single block).}
    \label{fig:len512_disc256_12}
\end{figure}

\begin{figure}[t]\centering
    \begin{subfigure}[t]{0.49\textwidth}\centering
        \includegraphics[width=\linewidth,height=0.42\textheight,keepaspectratio]{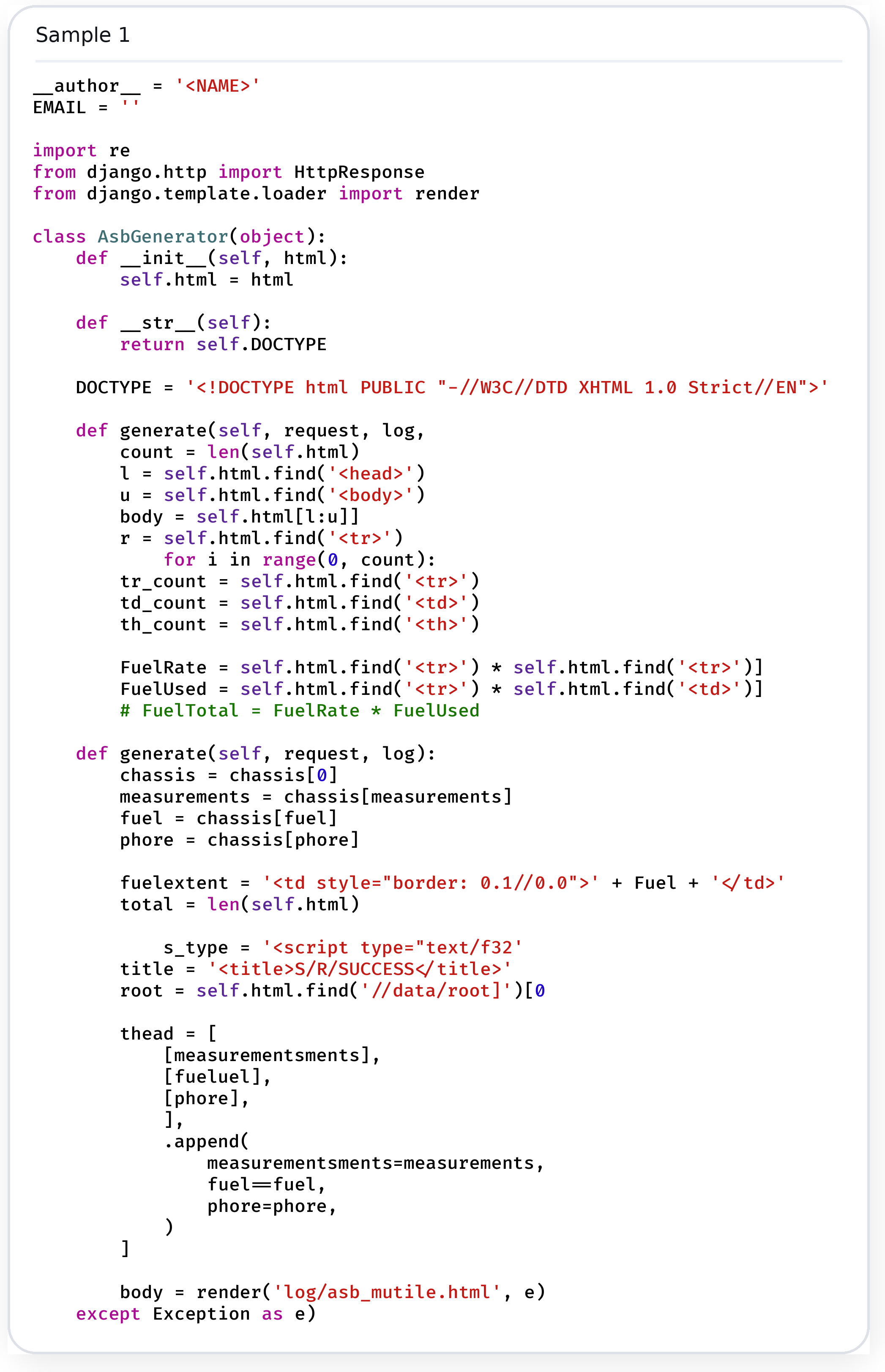}%
        \caption{Sample 1}\label{fig:len512_disc128_s1}
    \end{subfigure}\hfill
    \begin{subfigure}[t]{0.49\textwidth}\centering
        \includegraphics[width=\linewidth,height=0.42\textheight,keepaspectratio]{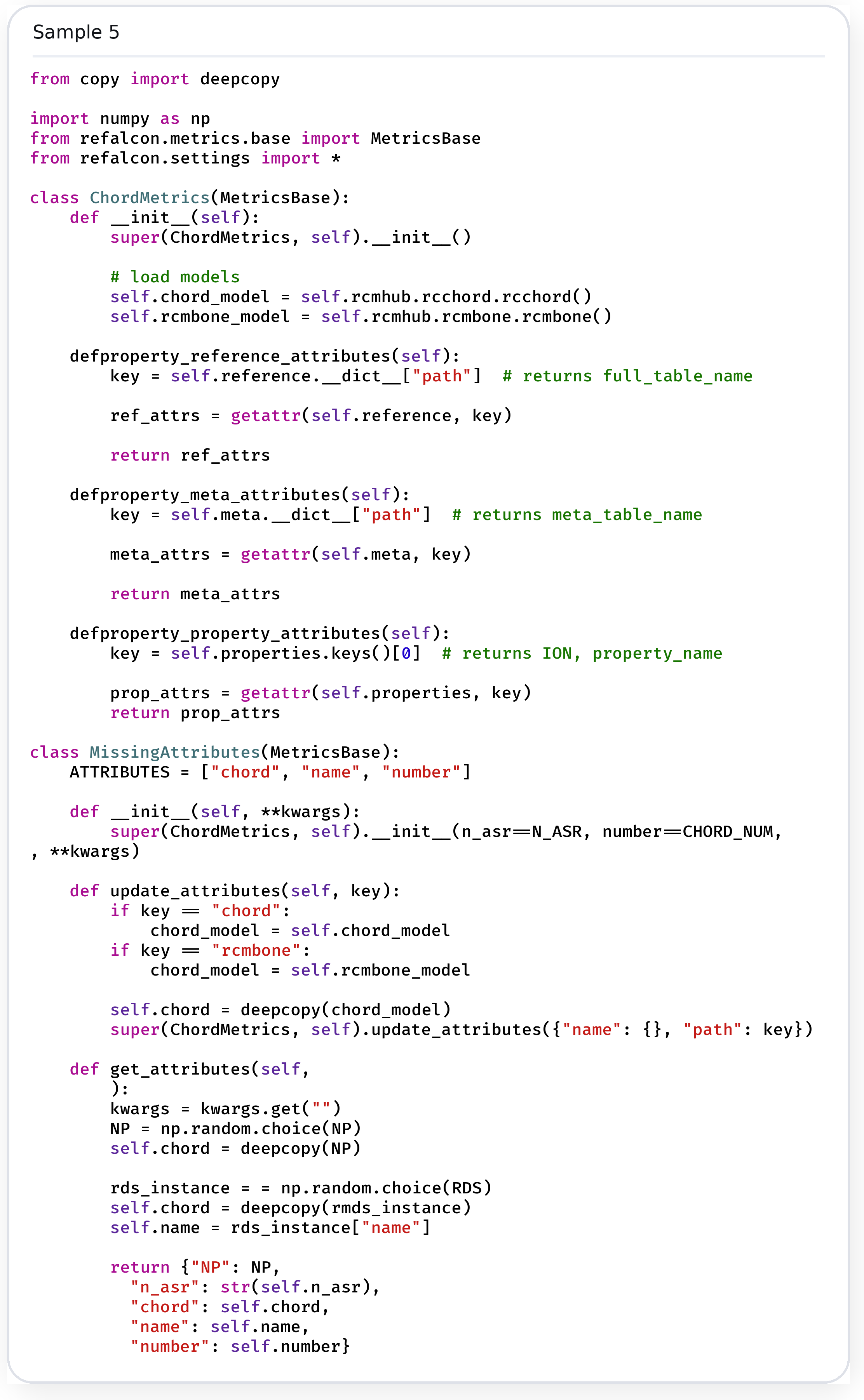}%
        \caption{Sample 2}\label{fig:len512_disc128_s2}
    \end{subfigure}
    \caption{\textbf{Qualitative code generations (length 512), samples 1--2.}
    Continuous registers are sampled with \emph{ConThenDisc} (128 DDIM steps) and decoded with LLaDA (128 discrete denoising steps, single block).}
    \label{fig:len512_disc128_12}
\end{figure}

\begin{figure}[t]\centering
    \begin{subfigure}[t]{0.49\textwidth}\centering
        \includegraphics[width=\linewidth,height=0.42\textheight,keepaspectratio]{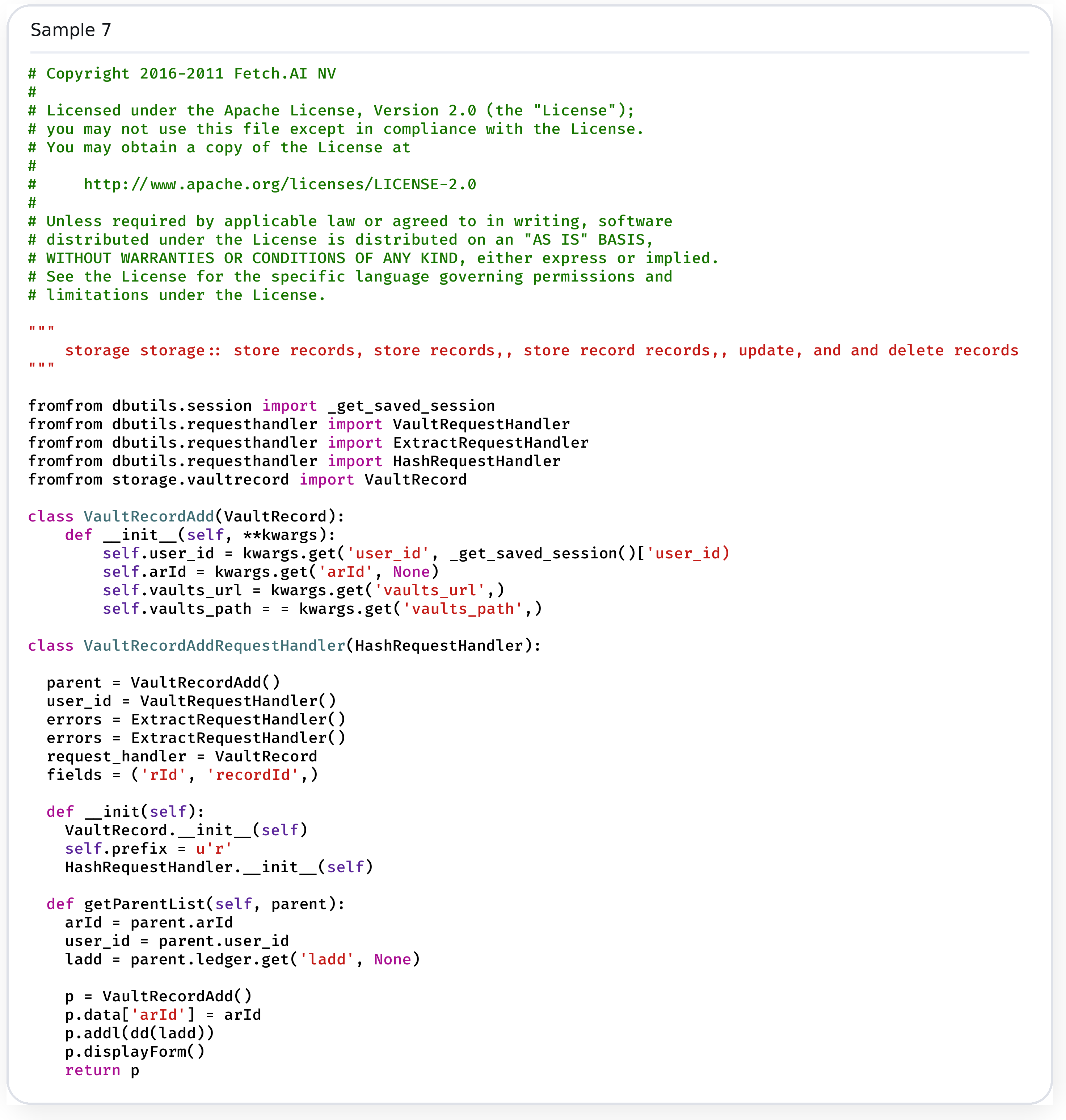}%
        \caption{Sample 1}\label{fig:len512_disc64_s1}
    \end{subfigure}\hfill
    \begin{subfigure}[t]{0.49\textwidth}\centering
        \includegraphics[width=\linewidth,height=0.42\textheight,keepaspectratio]{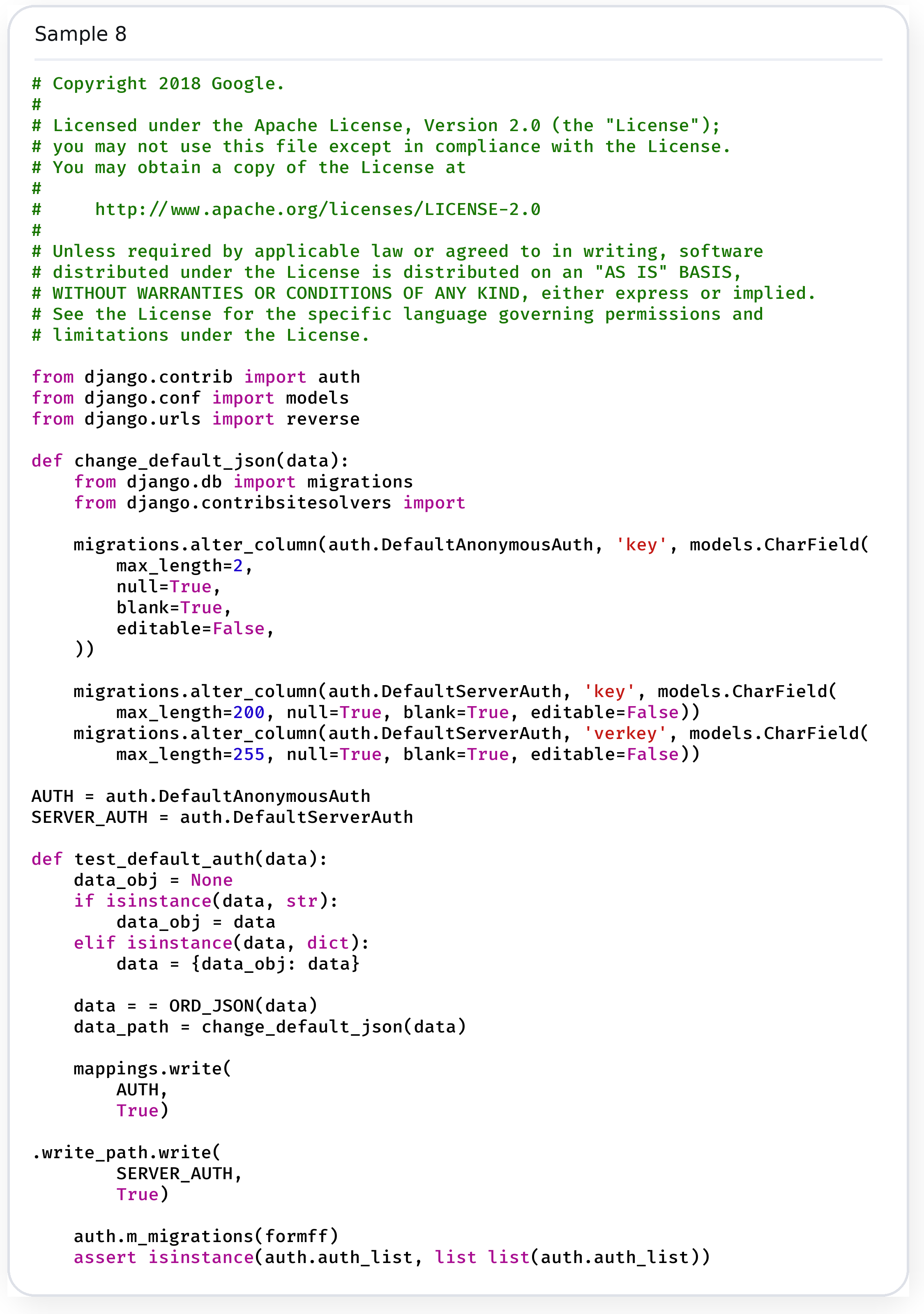}%
        \caption{Sample 2}\label{fig:len512_disc64_s2}
    \end{subfigure}
    \caption{\textbf{Qualitative code generations (length 512), samples 1--2.}
    Continuous registers are sampled with \emph{ConThenDisc} (128 DDIM steps) and decoded with LLaDA (64 discrete denoising steps, single block).}
    \label{fig:len512_disc64_12}
\end{figure}

\begin{figure}[t]\centering
    \begin{subfigure}[t]{0.49\textwidth}\centering
        \includegraphics[width=\linewidth,height=0.42\textheight,keepaspectratio]{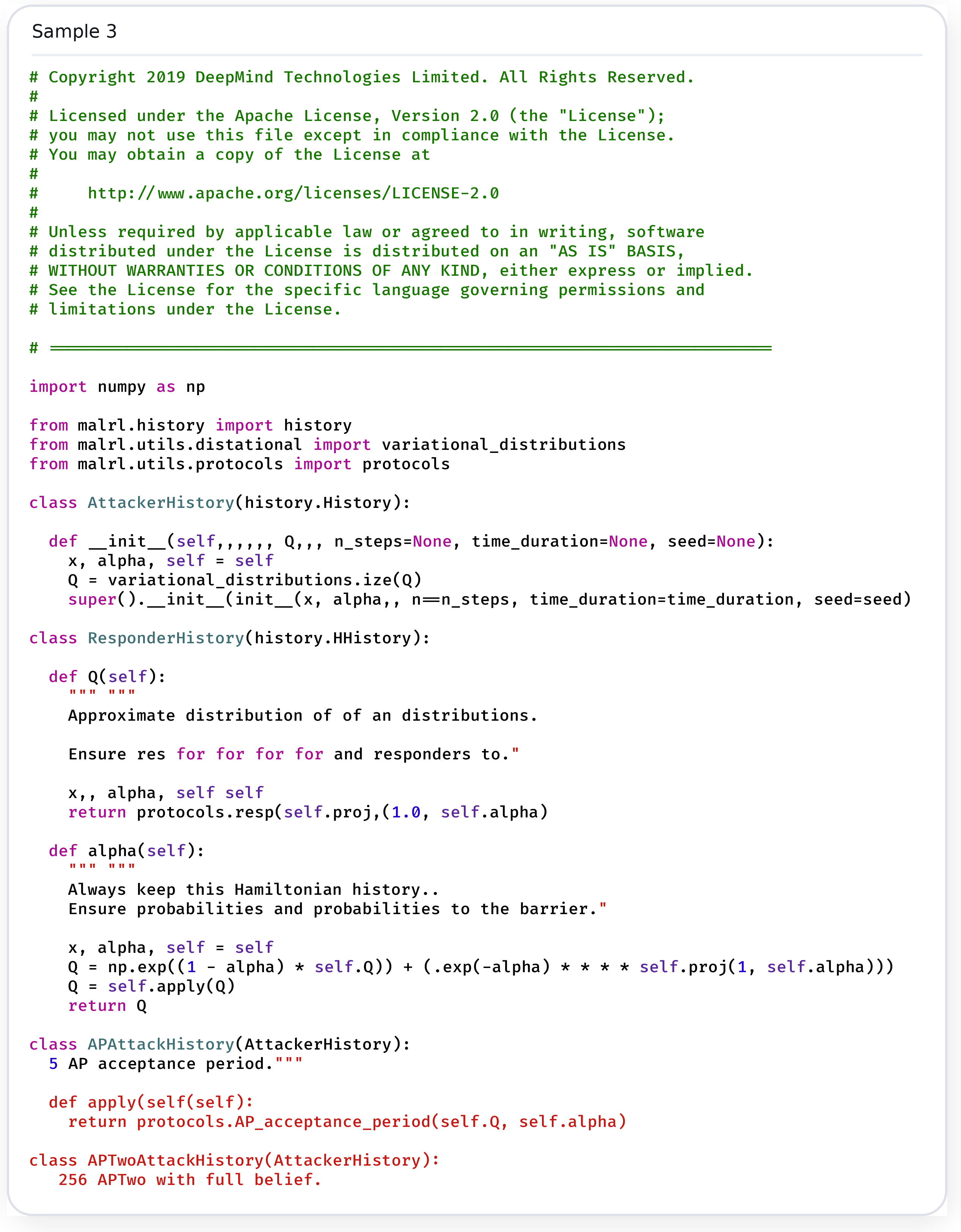}%
        \caption{Sample 1}\label{fig:len512_disc32_s1}
    \end{subfigure}\hfill
    \begin{subfigure}[t]{0.49\textwidth}\centering
        \includegraphics[width=\linewidth,height=0.42\textheight,keepaspectratio]{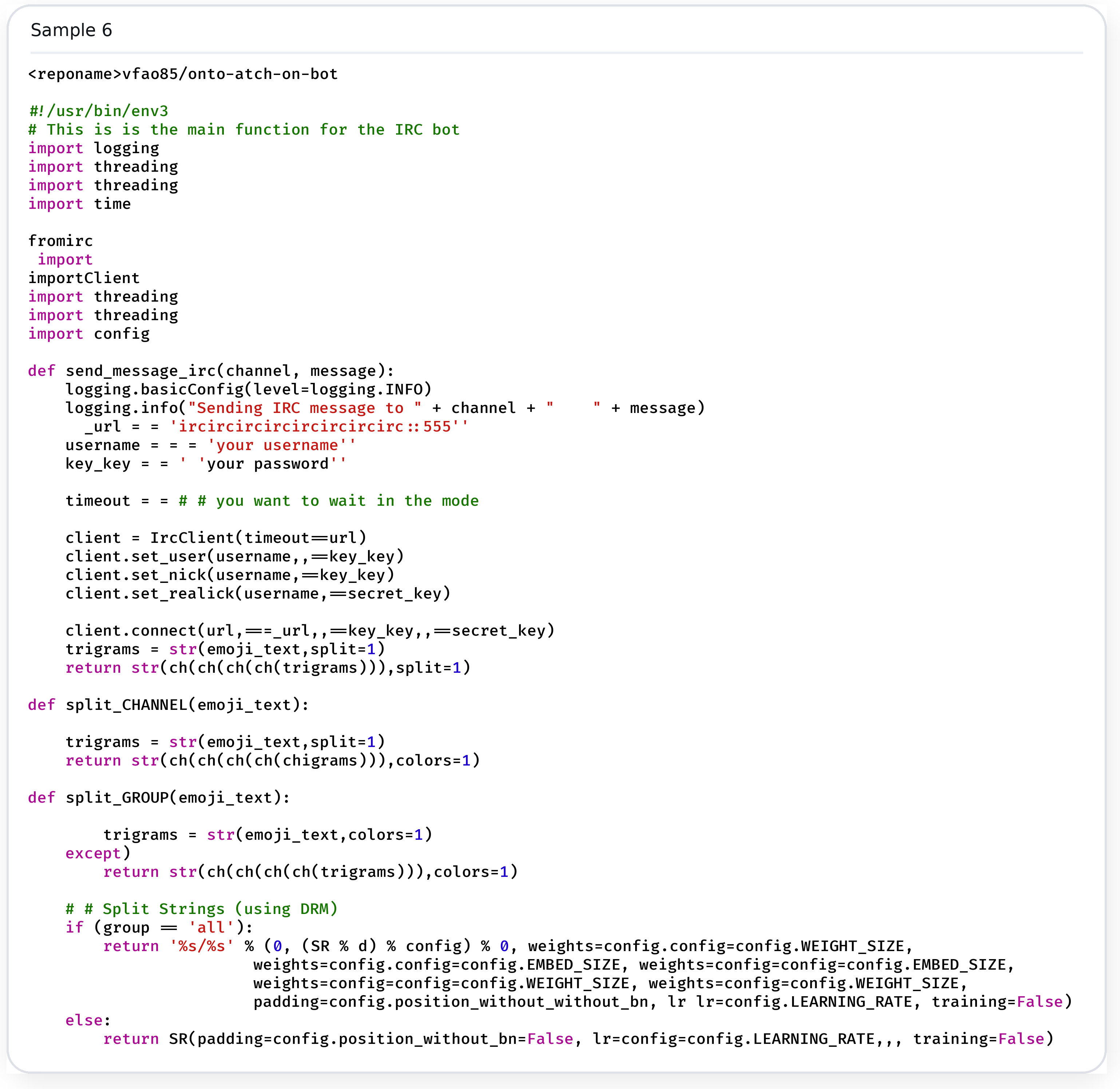}%
        \caption{Sample 2}\label{fig:len512_disc32_s2}
    \end{subfigure}
    \caption{\textbf{Qualitative code generations (length 512), samples 1--2.}
    Continuous registers are sampled with \emph{ConThenDisc} (128 DDIM steps) and decoded with LLaDA (32 discrete denoising steps, single block).}
    \label{fig:len512_disc32_12}
\end{figure}

\end{document}